\documentclass[twoside]{article}

\usepackage[preprint]{aistats2026}

\usepackage{amsmath}
\usepackage{amssymb}
\usepackage{url} 
\usepackage{graphicx}
\usepackage{booktabs}
\usepackage{multirow}
\usepackage{amsmath}
\usepackage{mathtools}
\usepackage{threeparttable}
\usepackage{amssymb}
\usepackage{amsthm}
\usepackage[ruled,vlined]{algorithm2e} \usepackage{subcaption} \usepackage{thmtools}  
\usepackage{graphicx}
\usepackage[hidelinks]{hyperref}
\usepackage{cleveref}

\newtheorem{theorem}{Theorem}[section]

\newtheorem{lemma}[theorem]{Lemma}

\newtheorem{definition}[theorem]{Definition}

\newcommand\nnz{\mathrm{nnz}}
\newcommand\poly{\mathrm{poly}}
\newcommand\polylog{\mathrm{polylog}}

\newcommand{\eps}{\varepsilon}
\renewcommand{\epsilon}{\varepsilon}

\usepackage[round]{natbib}

\renewcommand{\cite}{\citep}

\allowdisplaybreaks
\begin{document}

\twocolumn[

\aistatstitle{Scalable Learning of Multivariate Distributions via Coresets}

\aistatsauthor{Zeyu Ding \And Katja Ickstadt \And Nadja Klein \And Alexander Munteanu \And Simon Omlor}

\aistatsaddress{Lamarr Institute for ML \& AI, \\ TU Dortmund \And Scientific Computing Center,\\Karlsruhe Institute of Technology \And Department of Statistics,\\ TU Dortmund} ]

\begin{abstract}
Efficient and scalable non-parametric or semi-parametric regression analysis and density estimation are of crucial importance to the fields of statistics and machine learning. However, available methods are limited in their ability to handle large-scale data. We address this issue by developing a novel coreset construction for multivariate conditional transformation models (MCTMs) to enhance their scalability and training efficiency. To the best of our knowledge, these are the first coresets for semi-parametric distributional models. Our approach yields substantial data reduction via importance sampling. It ensures with high probability that the log-likelihood remains within multiplicative error bounds of $(1\pm\varepsilon)$ and thereby maintains statistical model accuracy. Compared to conventional full-parametric models, where coresets have been incorporated before, our semi-parametric approach exhibits enhanced adaptability, particularly in scenarios where complex distributions and non-linear relationships are present, but not fully understood.
To address numerical problems associated with normalizing logarithmic terms, we follow a geometric approximation based on the convex hull of input data. This ensures feasible, stable, and accurate inference in scenarios involving large amounts of data. Numerical experiments demonstrate substantially improved computational efficiency when handling large and complex datasets, thus laying the foundation for a broad range of applications within the statistics and machine learning communities.
\end{abstract}

\section{INTRODUCTION}

In today's era of Big Data, the vast volume and increasing complexity of data in many real-world applications pose new challenges to statistics and machine learning. Traditional methods require intense computing times on large data and straightforward solutions such as uniform subsampling may lead to infeasible solutions. They often rely on strong modeling assumptions that are too restrictive to capture complicated phenomena accurately. These limitations demand computationally efficient and scalable techniques for more flexible multivariate models.

To address the need of scalable learning, the method of \emph{coresets} received a lot of attention in recent years. It aims to sample or select a relatively small but representative subset of the original data that can be used to accurately approximate the objective function for any solution, thereby significantly reducing computing, memory and storage requirements \cite{Phillips17,MunteanuS18}.

However, most of the existing coreset literature focused on clustering, linear regression, or generalized linear models that can handle only relatively simple and limited distributional assumptions. Little research has been conducted on more flexible non-parametric, semi-parametric or highly non-linear multivariate models. This gap becomes more critical as the machine learning and statistics communities increasingly rely on complex distributional regression and estimation methodologies including not only multivariate features but also multivariate outcomes.

\subsection{Brief Introduction to MCTMs}
\label{sec:MCTMintro}

In this paper, we specifically focus on multivariate conditional transformation models \citep[MCTMs;][]{klein2022multivariate}, which have the unique advantage of modeling both unconditional and conditional distributions, non-linear dependence structures, and flexible effects of features. The fundamental idea of MCTMs is to model marginal distributions non-parametrically via monotonic transformation functions. Multiple univariate distributions are then combined into a multivariate distribution via a \emph{copula}, that contributes all information required to model the multivariate dependence structure. Sklar's Theorem \citep{Sklar1959} is well-known and states that every multivariate distribution can be composed and represented in that way. Of course, representing and calculating such a model explicitly in a lossless way for arbitrary multivariate distributions may require solving infinite-dimensional numerical problems which is arguably not practical.

MCTMs thus formulate a semi-parametric model that rely on the  monotonic univariate transformations being approximated as a linear combination of Bernstein polynomial basis functions whose choice ensures monotonicity of the composed functions. The dependence structure between the univariate output components is imposed by a Gaussian copula \citep{Son2000}. We note that those modeling choices may be replaced by different functional basis families and different copulas, but we stress and demonstrate that it is not as restrictive as it may seem. Notice that, although Gaussian copulas are centrally symmetric, the resulting model need not be symmetric unless all marginal distributions are symmetric as well, see for instance the density contour plots in \Cref{fig:dgp_vis_1,fig:dgp_vis_2,fig:dgp_vis_3,fig:dgp_vis_4,fig:dgp_vis_5} in \Cref{app:sim_visual}.

For the sake of a concise presentation, we restrict ourselves to unconditional density estimation without features. See \citep{klein2022multivariate} for details and an extension to the conditional case of distributional regression.
The goal is to learn the density $ f_Y(y) $ of a $ J $-dimensional random vector $ Y = (Y_1, \ldots, Y_J)^T  \in \mathbb{R}^J $ and its distribution function $ F_Y(y) = P(Y \leq y) =P(h(Y)\leq h(y))$. This involves a bijective, strictly monotonically increasing transformation function $ h : \mathbb{R}^J \rightarrow \mathbb{R}^J $ estimated from the data. It maps $ Y $ to another random vector $ Z \in \mathbb{R}^J $ that follows some convenient reference distribution. In particular, the common marginal distribution $P_Z$ of the components $Z_j\sim P_Z$, $ j \in [J]$ is a modeling choice and is assumed to not depend on the data. It thus holds that 
\[ h(Y) = (h_1(Y), \ldots, h_J(Y))^T  \stackrel{d}{=} (Z_1, \ldots, Z_J)^T  = Z
\,.\]
Furthermore, it is assumed that each component $ h_j $ can be expressed as a linear combination of strictly monotonically increasing marginal transformation functions $ \tilde{h}_j : \mathbb{R} \rightarrow \mathbb{R} $ of the form
\[ h_j(y_1, \ldots, y_j) = \lambda_{j1}\tilde{h}_1(y_1) + \ldots + \lambda_{jj}\tilde{h}_j(y_j)\,,\] 
where $\lambda_{jl}$, $l\leq j$ are the unrestricted entries of a $J \times J$ lower triangular matrix $\Lambda$, such that $\Sigma=\Lambda^{-1}(\Lambda^{-1})^T$ is the covariance matrix of $Z$. Hence, $\Lambda$ describes the conditional dependencies between the $J$ components. The functions $ \tilde{h}_j(y_j), j \in [J] $, are scaled marginal transformations, each of which carry an intercept term and together allow the generation of arbitrary marginal distributions. The marginal transformations are modeled semi-parametrically via some basis functions $ a_j : \mathbb{R} \rightarrow \mathbb{R}^d $ and basis coefficients $ \vartheta_j \in\mathbb{R}^d$. The $ a_j $ are chosen, for example, to be Bernstein polynomials which allow for a convenient way to impose monotonicity. The unknown model parameters are collected in the vector $ \theta = (\vartheta^T , \lambda^T )^T$.
For consistency with other coreset literature, we consider \emph{minimizing} the negative log-likelihood, or specifically
the sum of loss contributions of each data point $ y_i \in \mathbb{R}^J, i \in [n], $ to the negative log-likelihood, which is equivalent to calculating the maximum likelihood estimator $\hat \theta_n = \operatorname{argmin}\limits_{\theta=(\vartheta^T ,\lambda^T )^T } f(\theta)$, where 
\begin{align}
\label{eq:log_llk}
f(\theta) = \sum_{i=1}^{n} \frac{1}{2}\sum^J_{j = 1} 
& \left( \sum^{j-1}_{\jmath = 1} \lambda_{j\jmath} {a}_{\jmath}(y_{i\jmath})^T  {\vartheta_{\jmath}}+ {a}_j(y_{ij})^T {\vartheta}_j \right)^2 \notag \\
&\qquad\qquad - \log ({a}'_j(y_{ij})^T {\vartheta}_j)\,.
\end{align}
See \citep{klein2022multivariate} for details and a derivation of the log-likelihood and its gradients. 
MCTMs can be fitted to data by optimizing the linear basis coefficients $\vartheta$ of univariate transformation models and the covariance structure of the Gaussian copula, which is parametrized through the unrestricted entries of the lower triangular matrix of the modified Cholesky factor $\Lambda$. MCTMs can thus be seen as a multivariate extension of univariate conditional transformation models \citep[CTMs;][]{HothornKB2014} to handle multi-dimensional outputs and their dependence structure.

As can be seen in \Cref{eq:log_llk}, the log-likelihood function of an MCTM contains both quadratic and logarithmic parts in terms of the parameters to optimize. The quadratic part is numerically tractable and can be reformulated in a way that enables handling within the framework of $\ell_2$-subspace embeddings \citep{Woodruff14}, in particular via $\ell_2$ leverage score subsampling \citep{DrineasMM06,RudelsonV07}. The logarithmic part is known to be unstable and computationally burdensome when optimizing models to fit large-scale datasets. Considering sensitivity sampling for similarly structured loss functions has required special efforts and novel techniques in recent work on Poisson regression models \citep{LieM24}. For the MCTMs considered here, the situation is even more aggravated since the terms that show up in the log-likelihood do not directly depend on the plain data, but instead on the polynomial basis transformations in the quadratic part and on their derivatives in the logarithmic part.

\subsection{Our Goals and Contributions}
We develop the first coreset approximation framework for MCTMs, aiming to achieve the following goals.

1. \textbf{Significant computational reduction while guaranteeing an accurate log-likelihood approximation.} By combining $\ell_2$ leverage score sampling with a geometric convex-hull approximation, we provide a $(1\pm\varepsilon)$-factor log-likelihood approximation for MCTMs using only a small fraction of the data.

2. \textbf{Solving the problem of unstable values of logarithmic terms.} We construct coresets separately for $a(y)$ that represent polynomial basis transformations of plain data $y$, and their respective derivatives $a'(y)$. This allows to use standard methods for the former part. For the second part, extending recent prior work of \citet{LieM24}, our method includes points to avoid extreme directions where the logarithm tends to infinity or produces infeasible solutions.

3. \textbf{Compatibility with other popular probabilistic models:} MCTMs build a flexible semi-parametric framework that has connections to contemporary machine learning methods such as deep learning and normalizing flows \citep{NFsurvey2021,papamakarios2021normalizing}. By scaling up MCTMs, our coresets bring significant speedup and scaling possibilities to downstream learning tasks that build upon MCTMs.

\textbf{Theoretically}, we develop a leverage score and sensitivity sampling analysis along with a convex hull approximation scheme to provide a small subset of data. We prove that it approximates the log-likelihood within a multiplicative $(1\pm\eps)$ factor to the full dataset.

\textbf{Empirically}, we demonstrate through simulated and real-world data experiments that the method offers significant benefits for large-scale data: increased computational efficiency and scalability, preserving the original model fit, likelihood ratio, and estimation bias.

We summarize our main contributions as follows:
\begin{enumerate}
\item \textbf{The first coresets for the MCTM framework.} To the best of our knowledge, we are the first to develop a coreset construction method for MCTMs, filling the gap of data reduction techniques for highly flexible semi-parametric model fitting methods for multivariate distributions.
\item \textbf{Convex hull based stabilization of logarithmic normalizing terms.} Numerical issues of logarithmic terms in the log-likelihood are eliminated based on prior work \citep{LieM24} by a convex-hull approximation of the derivative $a'(y)$ of transformed data $y$ . 
The properties of this geometric approximation are investigated theoretically and empirically. 
\item \textbf{Illustrative large-scale data scenarios.} We illustrate through experiments with simulated and real-world data that our new method operates more efficiently than using the original large-scale data. At the same time it retains almost the same model fit as the original full-data estimation, illustrating our theoretical guarantees in practice.
\end{enumerate}

\subsection{Related Work}
\textbf{Probabilistic Transformation Models. } 
Originally restricted to univariate outputs, \citet{HothornKB2014} proposed a boosting approach for estimating a monotonically invertible function that transforms the output variable to a convenient, data independent distribution (e.g., standard normal). This transformation function is a semi-parametric, monotonic spline approximation that can depend on input variables in a flexible manner. This framework indirectly models the full conditional distribution of the output and was later extended towards an entirely likelihood-based approach, known as \emph{most likely transformations} \citep{hothornMB2018}. \citet{klein2022multivariate} built on this approach using likelihood inference. They developed MCTMs that can estimate  the joint conditional distribution of a multivariate response directly based on the features. The original work of \citet{klein2022multivariate} considers datasets of up to ten thousand observations and up to ten-dimensional outputs. In contrast to previous probabilistic methods, MCTMs directly estimate the joint distribution function of a multivariate output vector, avoiding the loss of information that would result from modeling only the mean or quantiles.

Transformation models are closely related to normalizing flows \citep[NFs;][]{NFsurvey2021,papamakarios2021normalizing}, which are popular in machine learning. {We elaborate on this connection in \Cref{app:NFs}.} The main idea of both approaches is to apply a data-driven transformation such that the transformed variable follows a standard normal or some other convenient distribution. 
However, NFs transform their input via some kind of deep neural network as a black box, whereas our model implements a semi-parametric transformation using a Bernstein polynomial basis, which remains analytically tractable. NFs are most commonly employed as flexible variational distributions in machine learning, whereas our focus is on modeling the distribution of multivariate data along with their dependence structure. MCTMs have also been extended towards regression, where density estimation is conditioned on feature variables. The linear underlying structure of transformation functions of MCTMs have the great advantage of enhancing interpretability compared to the neural network approach of NFs. 
For instance, in MCTMs we can directly interpret the dependencies between marginal components of the outcome and clearly separate these effects from those modeling the marginal distributions. Further, MCTMs are likelihood-based and therefore yield access to confidence intervals via bootstrapping.

\textbf{Coresets for Large-Scale Statistical Modeling. } The main idea of the coreset method is to select the \emph{most important} subset of data based on their contribution to the objective function (e.g., log-likelihood or loss) and reweight them, so that training the model on this subset closely approximates the results of the original dataset. Early works on coresets focus on parametric models, such as linear regression \citep{clarkson2005subgradient,DrineasMM06,dasgupta2009sampling} or generalized linear models \citep{munteanu2018coresets,MolinaMK18,MunteanuOP22,LieM24,FrickKM24}. These also make up the bulk of work on coresets for statistical models.

Inferring whole probability distributions received relatively little attention in the coreset literature. Among them are coresets for univariate kernel density estimation \citep{PhillipsT18,PhillipsT20,CharikarKNS20}. Bayesian coreset based models focus on multivariate parameter distributions rather than multivariate outcomes \cite{geppert2017random,huggins2016coresets,campbell2018bayesian,ding2024scalable}. To our knowledge, the only attempt to learn multivariate data distributions using coresets for scalability has been made in the context of dependency networks that infer inter-variable dependencies with graph structure \cite{MolinaMK18}. This approach is again composed of multiple parametric models.

In summary, the coreset method has shown great potential in improving the efficiency of regression models and Bayesian inference. However, we stress that these methods have to date been mainly developed for parametric models, e.g., (generalized) linear models. Very limited work has considered more complex and flexible non- or semi-parametric distribution models for multivariate outputs. Our attempt to combine coresets with MCTMs is, to the best of our knowledge, the first of its kind, leveraging the flexibility of multivariate distributional regression with efficient data reduction. Our work fills the gap in the methodology for inference of large-scale multivariate distributional models.

\section{MCTM CORESET CONSTRUCTION}

Considering the MCTM model whose negative log-likelihood function can be defined as in \Cref{eq:log_llk}, the goal of the coreset approach is to obtain a small subset $C\subset D$ of data, such that the approximation of the original likelihood function is bounded within a factor $(1\pm\epsilon)$. The full details and proofs of our construction are deferred to \Cref{app:theory}.

After applying the basis functions $a$ and their derivatives $a'$ to the raw data, we can assume that for $(i, j) \in [n] \times [J]$ we are given data points $a_{ij}\coloneqq a_j(y_{ij}) \in \mathbb{R}^d$ and similarly for the derivatives we have $a_{ij}'\coloneqq a_j'(y_{ij})\in \mathbb{R}^d$. We use $A, A'\in \mathbb{R}^{nJ\times d}$ to denote the corresponding data matrices. Additionally, we assume that there is an intercept, i.e., that for each $(i, j)$ the first coordinate of both $a_{ij}$ and $a_{ij}' $ is $1$ to make it consistent with the presence of intercepts in the transformation functions $h$ defined before, in \Cref{sec:MCTMintro}.

Now for $i \in [n], j \in [J]$ consider the function $f(a_{ij}, \vartheta, \lambda)=\frac{1}{2}(\sum_{k=1}^j \lambda_{j, k} \vartheta_k a_{ij} )^2 - \log(\vartheta_j a_{ij}') $ which is the negative logarithm of the corresponding likelihood component $g(i, j)={\vartheta_j a_{ij}'} \exp(- \frac{1}{2}(\sum_{k=1}^j \lambda_{j, k} \vartheta_k a_{ij} )^2) $.
We assume that for all $i \in [n], j \in [J]$ and all choices of parameters it holds that $g(i,j) \leq c$ for some constant $c \in \mathbb{R}_{\geq 1}$. This is a natural Lipschitz-type restriction ensuring that the distribution function has a smooth transition from $0$ to $1$ preventing sudden jumps. Note that this is equivalent to $-\ln(g(i, j))\geq -\ln(c)$.

\phantomsection\label{helper}
We further add to the $nJ$ loss contributions a total shift of $\log \mathcal N = nJ(\ln(c)+1)$. This corresponds to a normalization term $\mathcal N$ that ensures non-negativity and thus allows a relative approximation to be meaningful. Crucially, it does not affect the optimization because it is independent of the parameters that we optimize.

In the following, we sample with probabilities that are larger than the $\ell_2$ leverage scores plus a uniform term and add the convex hull of $\{a_{ij}' ~|~ i \in [n], j\in [J]\}$. This yields a coreset for the loss function $f(A,\vartheta,\lambda)$ to be minimized in \Cref{eq:log_llk}. Let $w\in \mathbb{R}^{n \times J}$ be a weight vector. We split the weighted version of $f$ into three parts (weights $w_{ij}$ are omitted in the unweighted case):
\begin{align*}
    &\hspace{-.2in}\text{1) squared part:}\quad f_1(A, \vartheta, \lambda, w) \\
    &= \frac{1}{2}\sum\nolimits_{(i, j) \in [n]\times [J]}w_{i j}\left( \sum_{k=1}^j \lambda_{j, k} \langle \vartheta_k, a_{ij} \rangle \right)^2 \\
    &\hspace{-.2in}\text{2) positive log part:}\quad f_2(A, \vartheta, \lambda, w) \\
    &= \sum\nolimits_{(i, j) \in [n]\times [J]} w_{i j}\max\{\log(\langle \vartheta_j, a'_{ij} \rangle), 0\} \\
    &\hspace{-.2in}\text{3) negative log part:}\quad f_3(A, \vartheta, \lambda, w) \\
    &= \sum\nolimits_{(i, j) \in [n]\times [J]}w_{i j}\max\{-\log(\langle \vartheta_j, a'_{ij} \rangle), 0\}
\end{align*}

\textbf{1) Squared Part.}\label{sec:squared}
We let $u_i = \sup_{\|x\|_2=1}\frac{|M_ix|^2}{\|Mx\|_2^2}$ for the $i$-th row of a matrix $M$ denote their $\ell_2$ leverage score. We show that sampling proportionally to the $\ell_2$ leverage scores preserves the squared part $f_1$. Given rows $a_{ij} \in \mathbb{R}^d$ where $i \in [n]$ and $j \in [J]$, we are looking for an $\varepsilon$-coreset, which is given by a matrix comprising a subset of rows indexed by $S \subseteq [n] \times [J]$ and corresponding weights $w_{i,j}$ for every $(i,j)\in S$, such that for all $\vartheta_1 , \ldots, \vartheta_J \in \mathbb{R}^d $ and $\lambda \in \mathbb{R}^{J \times J}$ it holds that
\begin{align*}
    &\hspace{-.05in}\left|\sum_{i=1}^n \sum_{j=1}^J \left(\sum_{k=1}^j\lambda_{j, k} \langle \vartheta_k, a_{ij} \rangle \right)^2 \right. \\
    &\left. \hspace{1.1in} - \sum_{(i, j) \in S} w_{i, j} \left(\sum_{k=1}^j \lambda_{j, k} \langle \vartheta_k, a_{ij} \rangle \right)^2 \right| \\
    &\hspace{.5in}\leq \varepsilon \left( \sum_{i=1}^n \sum_{j=1}^J \left(\sum_{k=1}^j\lambda_{j, k} \langle \vartheta_k, a_{ij} \rangle\right)^2 \right)\,.
\end{align*}
To this end, we arrange data points in a matrix such that sampling rows of this matrix yields a coreset for the function defined above.
We set $B \in \mathbb{R}^{nJ \times dJ^2}$ to be the matrix whose rows equal
$(b_{iJ+j})_{k}= a_{il}$ if $k=(j-1)J+l$ for some $l \in [J]$ and $(b_{iJ+j})_{k}= 0$ otherwise.

Then $B$ consists of $n$ vertically stacked blocks $B_i$, for $i\in [n]$. The $i$-th block is defined by
\[
    B_i =
  \left[ {\begin{array}{cccccc}
    b_i & 0 & 0 & 0 & \cdots & 0 \\
    0 & b_i & 0 & 0 &\cdots & 0\\
    0 & 0 & b_i & 0 &\cdots & 0 \\
    \vdots & \vdots & \vdots &  \ddots & \cdots &  0\\
    0& 0 & 0 & 0 & \cdots & b_i \\
  \end{array} } \right] \in \mathbb{R}^{J\times dJ^2}\,,
\]
where $b_i=(b_{i1}, b_{i2}, \ldots, b_{iJ})$.
The idea is that for any possible parametrization, the squared part can be represented by a product of the new matrix $B$ with the vector $\theta = (\vartheta^T , \lambda^T )^T $. An $\eps$-subspace embedding for $\ell_2$ is therefore sufficient to approximate the squared part, i.e., it yields that $\forall \theta\colon \|B'\theta\|_2^2 = (1\pm\eps) \|B\theta\|_2^2$, where $B'$ consists of few weighted rows subsampled from $B$ according to their $\ell_2$ leverage scores \cite{Woodruff14}.

\begin{restatable}{lemma}{lemsquarepart}
\label{lem:main:squarepart}
There exists a coreset $S$ for $f_1$ of size $O(J^2 d /\varepsilon^2) $ which can be computed in time $O(\nnz(B)\log(nJ) +\poly(dJ))$ and with high probability by sampling and reweighting according to the $\ell_2$ leverage scores of $B$, where $\nnz(B)$ denotes the number of non-zero entries of $B$. We then have $|f_1(A, \vartheta, \lambda)-f_1(A(S), \vartheta, \lambda, w)|\leq \varepsilon f_1(A,\vartheta, \lambda)$, where $A(S)$ denotes the restriction of $A$ to the indices in $S$.
\end{restatable}

\textbf{2) Positive Logarithmic Part.}
We now turn our attention to the positive logarithmic part $f_2$. This is going to be handled using the sensitivity framework, which generalizes the previous leverage score sampling for the $\ell_2$ norm to more general families of functions. We refer to \Cref{sec:sensitivityframework} for details.
First of all, we bound the VC dimension of the logarithmic function. This is done in a standard way. Using strict monotonicity, the logarithmic function of the inner product can be inverted (respecting their domain and range), relating it to linear classifiers in $d$ dimensions. The latter have a VC dimension of $d+1$, which is known from classic learning theory \cite{kearns94clt}.

The second part is bounding the total sensitivity, which is the sum of all sensitivity scores $s_i=\sup_{\vartheta,\lambda}\frac{f_2(a_{ij},\theta,\lambda)}{\sum f_2(a_{ij},\theta,\lambda)} $ of single data point contributions. To bound this value, we leverage that for all parameters $\lambda$ and $\vartheta$ it holds that
$\ln(\vartheta_j a_{ij}') -\frac{1}{2}( w_{i, j}\sum_{k=1}^j \lambda_{j, k} \vartheta_k a_{ij} )^2\leq \ln(c),$
which allows us to relate the contribution of the logarithmic part to the squared part up to a constant $\gamma>1$ such that $s_i\leq \gamma(u_i+1/n)$. This allows to reuse the $\ell_2$ leverage scores for this part as well, albeit with an additional uniform component, and with an increase of the sample size, which comes from incorporating the VC dimension and the Lipschitz bound $c$. We also note that the $\eps$-error is relative to $f_1$ instead of $f_2$ directly.

\begin{restatable}{lemma}{lemposlogpart}
\label{lem:main:poslogpart}
Assume that $S$ is a sample consisting of $O(J^2 d^2 \ln(c d J)  c^{6} /\varepsilon^2) $ points drawn with probability $p_i\geq \alpha(u_i+1/n)$, where $\alpha \in O(J^2 d \ln(c d J)  c^{6} /\varepsilon^2)$. Then with high probability it holds that $|f_2(A, \vartheta, \lambda)-f_2(A(S), \vartheta, \lambda, w)|\leq \varepsilon f_1(A,\vartheta, \lambda)$, where $A(S)$ denotes the restriction of $A$ to the indices in $S$.
\end{restatable}

\textbf{3) Negative Logarithmic Part.}
Next, we handle the remaining negative logarithmic part given by $f_3$. We note that it has an asymptote at $0,$ precluding finite bounds on the sensitivity. We handle this similarly to \cite{LieM24} by restricting the optimization space to $D(\eta)=\{ (\vartheta, \lambda) ~|~ \forall (i, j) \in [n] \times [J]: \langle \vartheta_j, a_{ij}'\rangle > \eta \}$ comprising only solutions for which the inner product is bounded away by $\eta\geq 0$ from zero. Setting $\eta=0$ corresponds to the original domain.\footnote{The final choice will be $\eta = \Theta(\eps)$ and negative value correction to the positive domain was common practice in \citep{klein2022multivariate} before our theoretical investigations.} By avoiding high sensitivity points in this way, $f_3$ can be bounded in terms of uniform sensitivities together with the VC dimension bound $d+1$ and approximated by invoking the sensitivity framework again.

\begin{restatable}{lemma}{lemneglogpart}
\label{lem:main:neglogpart}
    Let $ (\vartheta^*, \lambda^*)$ be an optimal solution. Then there exist $ (\vartheta, \lambda) \in D(\eta)$ with $|f(A, \vartheta, \lambda)-f(A, \vartheta^*, \lambda^*)|\leq 2 J \eta f_1(A, \vartheta^*, \lambda^*) + J \eta n + J^2 \eta ^2 n$.
    Further, if $S$ is a sample consisting of $\alpha$ points drawn with probability $p_i\geq \alpha/n$ where $\alpha \in O(d \ln(c d J)   /\eta^2)$ combined with all points that are on the convex hull of $\{a_{ij}' ~|~ i \in [n], j\in [J]\}$. Then with high probability, it holds for all $ (\vartheta, \lambda) \in D(\eta)$  that $|f_3(A, \vartheta, \lambda)-f_3(A(S), \vartheta, \lambda, w)|\leq \eta f_1(A, \vartheta, \lambda) + \eta n$, where $A(S)$ denotes the restriction of $A$ to the indices in $S$.
\end{restatable}

\textbf{Main Result.} We get the following theorem by a union bound over \Cref{lem:main:squarepart,lem:main:poslogpart,lem:main:neglogpart} and adding up their error bounds using the triangle inequality. The additive errors of \Cref{lem:main:neglogpart} are further charged against the optimal cost using the normalization given by $\log\mathcal N$, see above at the beginning of \Cref{helper}.

\begin{restatable}{theorem}{thmmain}
   Assume that $S$ is a sample consisting of $O(J^2 d^2 \ln^3(c d J)  c^{6} /\varepsilon^2)$ points drawn with probability $p_i\geq\alpha(u_i+1/n)$, where $\alpha \in O(J^2 d \ln^3(c d J) c^{6}/\varepsilon^2)$ combined with all extreme points $(i, j) \in [n] \times [J]$ that are on the convex hull of $\{a_{ij}' ~|~ i \in [n], j\in [J]\}$. Then with high probability it holds for any $(\vartheta, \lambda) \in D(\eta) $ that $|f(A, \vartheta, \lambda)-f(A(S), \vartheta, \lambda, w)|\leq \varepsilon f(A, \vartheta, \lambda)$ and there exists $(\vartheta, \lambda)\in D(\eta)$ such that $f(A,\vartheta, \lambda)\leq (1+O(\eps)) f(A,\vartheta^*, \lambda^*),$ for an optimal solution $(\vartheta^*, \lambda^*).$
\end{restatable}

The formal proof can be found in \Cref{app:theory}. We remark that the convex hull can comprise $\Omega(nJ)$ points, however, different $\eta$-kernel coresets of size $\Theta(1/\eta^{(d-1)/2})$ exist for the problem \citep{AgarwalHV04,Chan04}, surveyed in \citep{AgarwalHV05}, in the field of computational geometry. These $\eta$-kernels also match the requirements of the shifted domain $D(\eta)$. We discuss one particular choice \citep{blum2019sparse} in the experimental \cref{sec:experiments} below.

\textbf{Lower Bounds. }
We also give two lower bounds under different natural assumptions on the $\lambda$ coefficients that define the dependence structure of the Gaussian copula. These indicate the limitations of coresets for MCTMs and show that our upper bounds leave only small gaps. The details are again in \Cref{app:theory}.

\begin{restatable}{lemma}{lemlowerone}
There exists a dataset $\{ a_{ij} \}_{i \in [n] j \in [J]} $ such that any coreset for MCTMs with $\eps<1$ has size at least $\Omega(d J^2) $ even if $\lambda_{ij}=0$ for all $i,j  \in [n]$ with $j>i $ and $|\lambda_{ij}|\leq 1 $ for all $i, j \in [J]$.
\end{restatable}

\begin{restatable}{lemma}{lemlowertwo}
There exists a dataset $\{ a_{ij} \}_{i \in [n] j \in [J]} $ such that any coreset for MCTMs with $\eps<1$ has size $\Omega(d J) $ even if $\lambda_{ii}=1$ for $i \in [J]$ and $\lambda_{ij}=0$ for $i<j$.
\end{restatable}

\section{EMPIRICAL EVALUATION}
\label{sec:experiments}
In this section, we systematically investigate {coreset} data reduction techniques for MCTMs with large-scale data for a variety of 2-dimensional data generation processes as well as two multi-dimensional, large-scale, real-world applications. Our objectives are to

\emph{(i) validate the effectiveness of the proposed sampling and convex hull algorithm for MCTMs at fitting  different distributions, different correlation structures, and non-linear or heavy-tailed scenarios, and to} 

\emph{(ii) quantify and compare the performance of uniform sampling, pure $\ell_2$ leverage score sampling, and our sensitivity sampling combined with the convex hull approximation on various metrics.}

Our algorithms are stated in \Cref{app:algorithm}. In particular, our sensitivity sampling is detailed in \Cref{algorithm:mctm_coreset_sparse}. As convex hull approximation, we implemented \citep{blum2019sparse} given in \Cref{alg:epskernel}. For \emph{mild} data, i.e., data that admits an $\eta$-kernel below the $\Omega(1/\eta^{(d-1)/2})$ lower bound, it yields an $\eta$-kernel of size $O(k^*/\eta^2)$, where $k^*$ is the smallest possible size.\footnote{Our theoretical results show that $\eta = \Theta(\eps)$ suffices, and in the experiments we simply choose $\eta=2\varepsilon$.} All experiments were carried out on a 2021 MacBook Pro equipped with an Apple M1 Pro chip and 16 GB of RAM. The code to reproduce our experiments is available at \url{https://github.com/zeyudsai/mctmcoreset/}.

\subsection{Simulation Study on 2-Dimensional Data}

We conduct a set of $2$-dimensional data simulation experiments encompassing different dependence structures to validate the advantages of our proposed coreset approach ($\ell_2$-hull) over simple leverage scores ($\ell_2$-only) and uniform sampling as baselines.

All experimental results can be found in \Cref{app:experiments}. In particular, complete descriptions of 14 data generation processes (DGPs) are specified in \Cref{app:sim_DGP}. Additional density contour plots that visualize the DGPs are in \Cref{app:sim_visual}.

\Cref{tab:performance_comparison_30_partial} summarizes the simulation results for five representative scenarios. The experimental results clearly show that even in the extreme case of fitting the model using only a few dozen points, our proposed coreset method achieves a satisfactory approximation. To ensure the reliability of our experimental results, we performed $10$ independent repetitions of each simulation, to obtain the means and standard deviations of results.

\begin{table}[htbp]
\centering
\caption{Performance comparison of coreset methods (coreset size = 30). Results show mean $\pm$ std over 10 runs. Relative improvement: avg \% improvement across metrics. Best values highlighted in \textbf{bold}.}
\label{tab:performance_comparison_30_partial}
\footnotesize
\setlength{\tabcolsep}{0.1pt}
\begin{tabular}{@{}ll|c@{\hspace{1pt}}|c@{\hspace{1pt}}|c@{\hspace{1pt}}|c@{}}
\toprule
\textbf{\quad} & \textbf{Method } & \textbf{$\ell_2$ err.} & \textbf{$\lambda$ err.} & \textbf{LR} & \textbf{ Imp.(\%)} \\
\midrule
\multirow{3}{*}{\rotatebox{90}{\tiny DGP 1}} 
& $\ell_2$-hull & $2.56 \pm 0.7$ & $0.44 \pm 0.1$ & $\mathbf{\,1.54 \pm 0.2}$ & $\mathbf{12.8}$ \\
& $\ell_2$-only & $\mathbf{2.54 \pm 0.6}$ & $0.51 \pm 0.1$ & $1.65 \pm 0.3$ & $1.6$ \\
& uniform & $4.91 \pm 4.7$ & $\mathbf{\,0.29 \pm 0.2}$ & $1.94 \pm 1.1$ & baseline \\
\midrule
\multirow{3}{*}{\rotatebox{90}{\tiny DGP 2}} 
& $\ell_2$-hull & $\mathbf{1.76 \pm 0.8}$ & $\mathbf{0.09 \pm 0.1}$ & $\mathbf{\,1.03 \pm 0.0}$ & $\mathbf{49.8}$ \\
& $\ell_2$-only & $2.18 \pm 0.8$ & $0.10 \pm 0.1$ & $1.05 \pm 0.0$ & $38.8$ \\
& uniform & $3.08 \pm 0.9$ & $0.11 \pm 0.1$ & $1.21 \pm 0.4$ & baseline \\
\midrule
\multirow{3}{*}{\rotatebox{90}{\tiny DGP 3}} 
& $\ell_2$-hull & $\mathbf{\,2.60 \pm 1.0}$ & $\mathbf{0.14 \pm 0.1}$ & $\mathbf{1.07 \pm 0.0}$ & $\mathbf{49.6}$ \\
& $\ell_2$-only & $2.76 \pm 0.9$ & $0.16 \pm 0.1$ & $1.08 \pm 0.0$ & $43.7$ \\
& uniform & $4.14 \pm 1.8$ & $0.2 \pm 0.1$ & $1.42 \pm 0.3$ & baseline \\
\midrule
\multirow{3}{*}{\rotatebox{90}{\tiny DGP 4}} 
& $\ell_2$-hull & $\mathbf{2.51 \pm 2.6}$ & $\mathbf{0.06 \pm 0.0}$ & $\mathbf{1.33 \pm 0.5}$ & $\mathbf{41.4}$ \\
& $\ell_2$-only & $4.09 \pm 4.3$ & $0.07 \pm 0.0$ & $1.58 \pm 0.5$ & $9.0$ \\
& uniform & $4.11 \pm 2.4$ & $0.14 \pm 0.1$ & $1.47 \pm 0.5$ & baseline \\
\midrule
\multirow{3}{*}{\rotatebox{90}{\tiny DGP 5}} 
& $\ell_2$-hull & $\mathbf{2.07 \pm 0.6}$ & $\mathbf{0.08 \pm 0.0}$ & $\mathbf{1.07 \pm 0.1}$ & $\mathbf{64.8}$ \\
& $\ell_2$-only & $2.65 \pm 0.8$ & $\mathbf{0.08 \pm 0.0}$ & $1.09 \pm 0.1$ & $58.4$ \\
& uniform & $4.47 \pm 3.3$ & $0.16 \pm 0.1$ & $1.63 \pm 1.1$ & baseline \\
\bottomrule
\end{tabular}
\end{table}

Based on the full experimental results in \Cref{tab:performance_comparison_30,tab:performance_comparison_100} in \Cref{app:experiments}, we conclude that out of all $14$ DGPs, our proposed coreset method ($\ell_2$-hull) significantly outperforms uniform subsampling in $12$ scenarios, and fails to show an advantage only in two scenarios, but never falls behind. Similarly, it also outperforms the plain $\ell_2$ leverage scores ($\ell_2$-only).
This suggests that our coreset method has a robust performance across a wide range of data dependence structures, and is particularly suitable for dealing with non-linearities, heteroscedasticity, complex dependence structures, and multimodality. Our results thus highlight the stability and universality of our method.

Shortcomings in the two remaining scenarios, namely for \texttt{t-copula} and \texttt{skew-t}, show limitations for heavy-tailed distributions, when the coreset size is fixed. These have a dense convex hull and thus require the size of the convex hull approximation to be increased in order to compensate and reduce the error.

\subsection{Real-World Data Experiments}

We evaluate our method on two large-scale multivariate datasets from different real-world applications.

$\bullet$ \textbf{Forest Cover Data (Covertype).} The UCI Covertype dataset \cite{covertype_31} contains $n=581\,012$ samples and 54 variables describing terrain attributes. We focus on 10 continuous variables (e.g., elevation, slope, hillshade, distances) and a subsample of size $n=300\,000$. We report comparisons of $\ell_2$-hull versus $\ell_2$-only, uniform sampling on parameter $\vartheta$ error, $\lambda$ error, log-likelihood ratio in \Cref{tab:covertype_unconditional_performance_filled}. Additional baselines include ridge leverage scores (ridge-lss) and root leverage scores (root-l2). Further results and running times are presented in \Cref{fig:covertype_results} and in \Cref{app:experiments}.

$\bullet$ \textbf{Equity Returns.} We consider two stock‐return datasets: one comprising 10 major stocks’ forty-year daily returns ($n\approx10\,000$), and another with 20 stocks. Performances at varying coreset sizes $k\in\{50,100,200,300\}$ are summarized in \Cref{tab:stock_return_performance,tab:stock_return_performance_2} in \Cref{app:experiments}, and \Cref{fig:stock_comparison} visualizes the log-likelihood ratio, parameter $\ell_2$ distance and $\lambda$ distance.

We refer to \Cref{app:real_world} for further details on real-world data experiments. Across both real-world applications, our proposed $\ell_2$-hull strategy outperforms uniform subsampling, and other baselines, achieving tighter log-likelihood approximations and smaller parameter errors at comparable running times. These results confirm that coresets enable scalable, accurate MCTM fitting on large, multivariate datasets.
\begin{figure*}[htbp]
  \centering
  \begin{subfigure}[t]{0.33\textwidth}
    \centering
    \includegraphics[width=\textwidth]{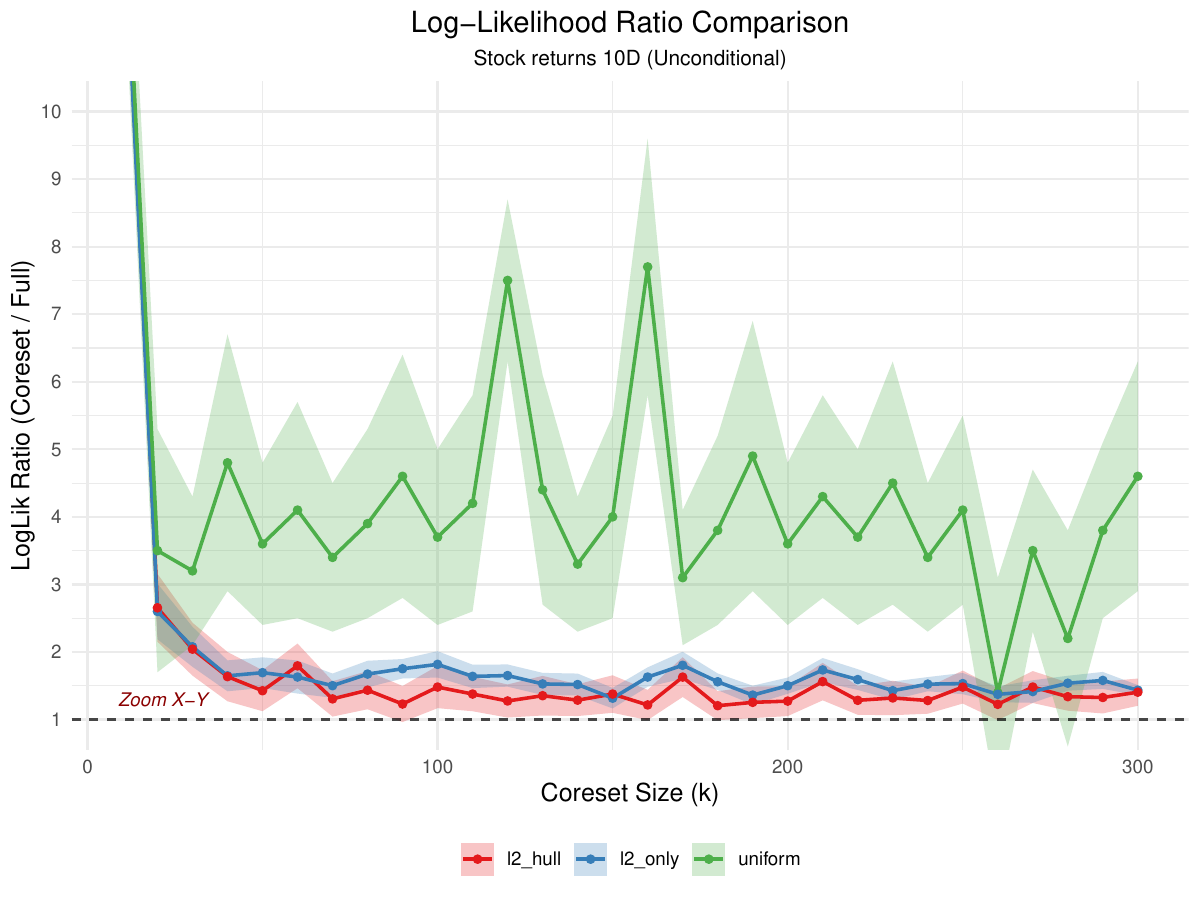}
    \caption{10-stock log-likelihood ratio}
    \label{fig:8_loglik_ratio}
  \end{subfigure}\hfill
  \begin{subfigure}[t]{0.33\textwidth}
    \centering
    \includegraphics[width=\textwidth]{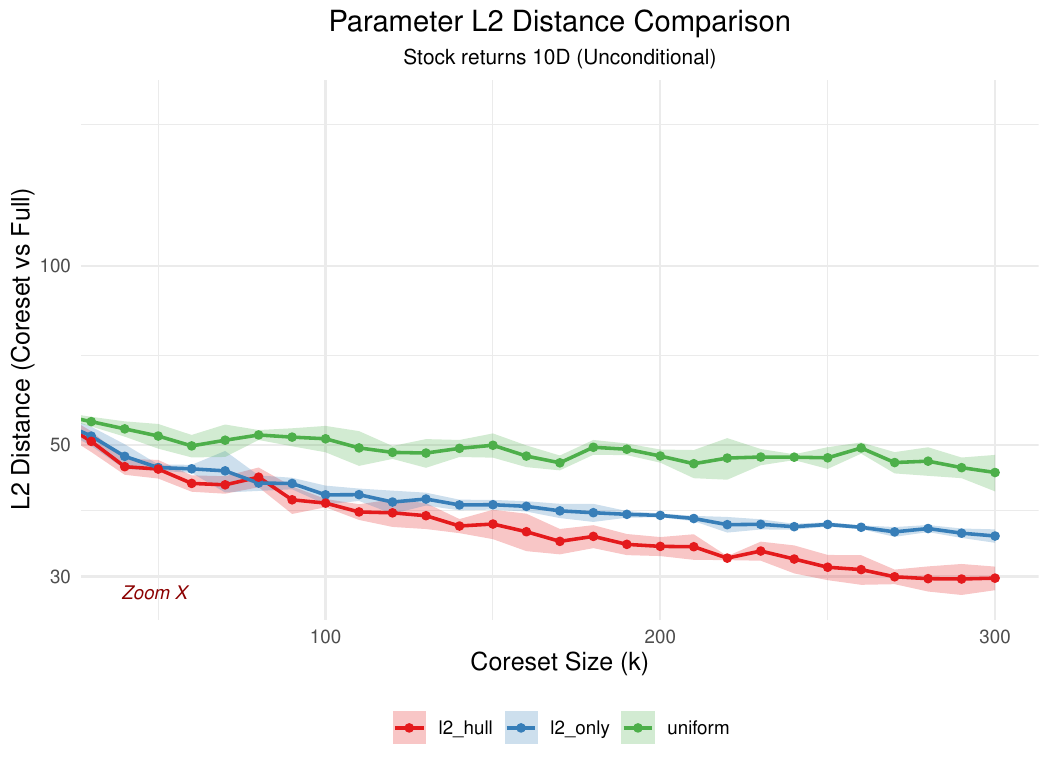}
    \caption{10-stock parameter $\vartheta$ distance}
    \label{fig:8_param_l2}
  \end{subfigure}\hfill
  \begin{subfigure}[t]{0.33\textwidth}
    \centering
    \includegraphics[width=\textwidth]{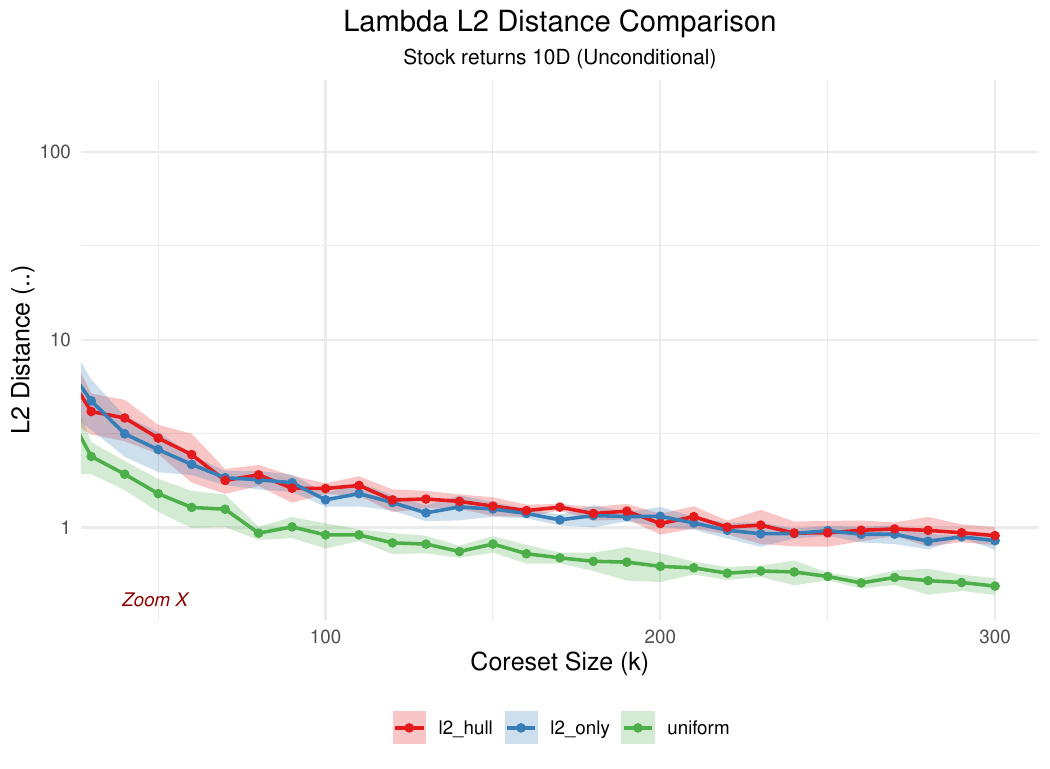}
    \caption{10-stock parameter $\lambda$ distance}
    \label{fig:8_lambda_l2}
  \end{subfigure}
  \begin{subfigure}[t]{0.33\textwidth}
    \centering
    \includegraphics[width=\textwidth]{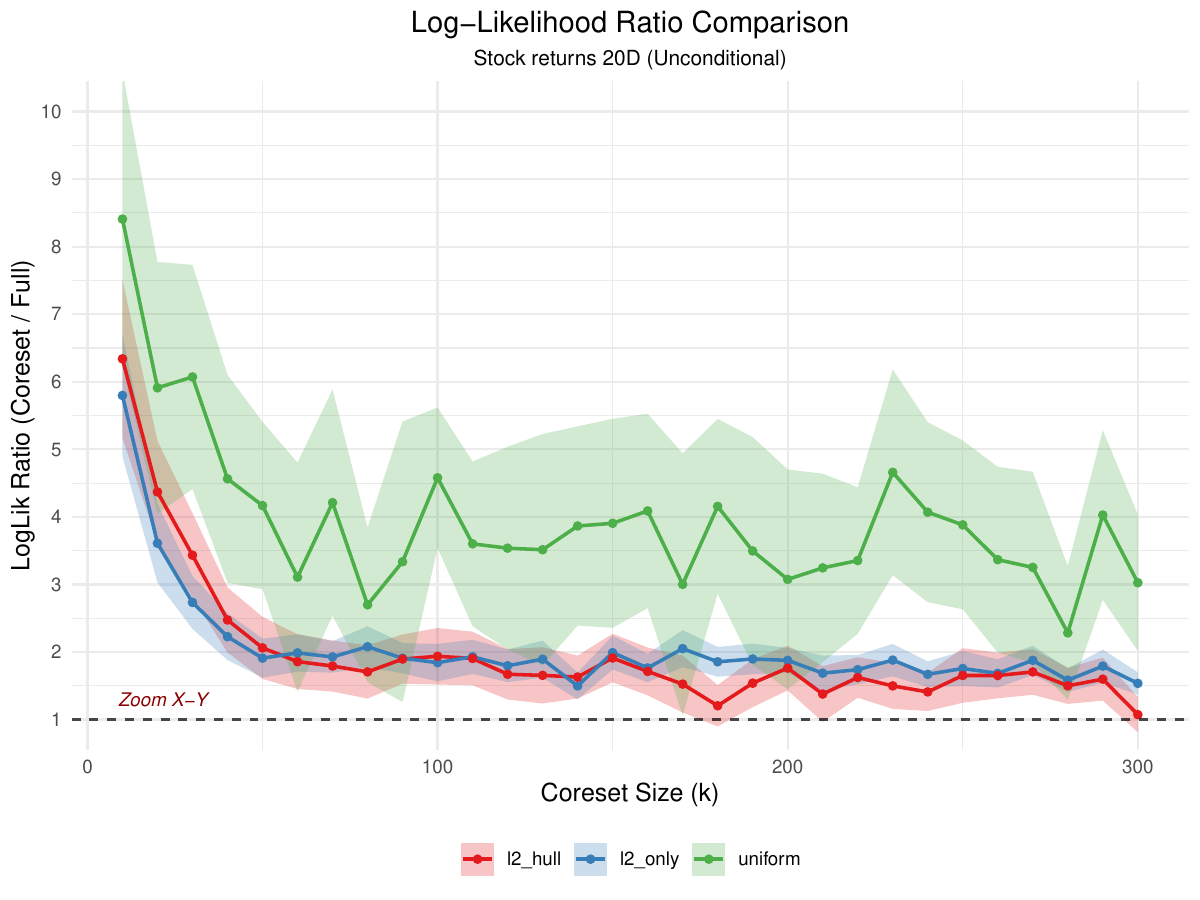}
    \caption{20-stock log-likelihood ratio }
    \label{fig:20_loglik_ratio}
  \end{subfigure}\hfill
  \begin{subfigure}[t]{0.33\textwidth}
    \centering
    \includegraphics[width=\textwidth]{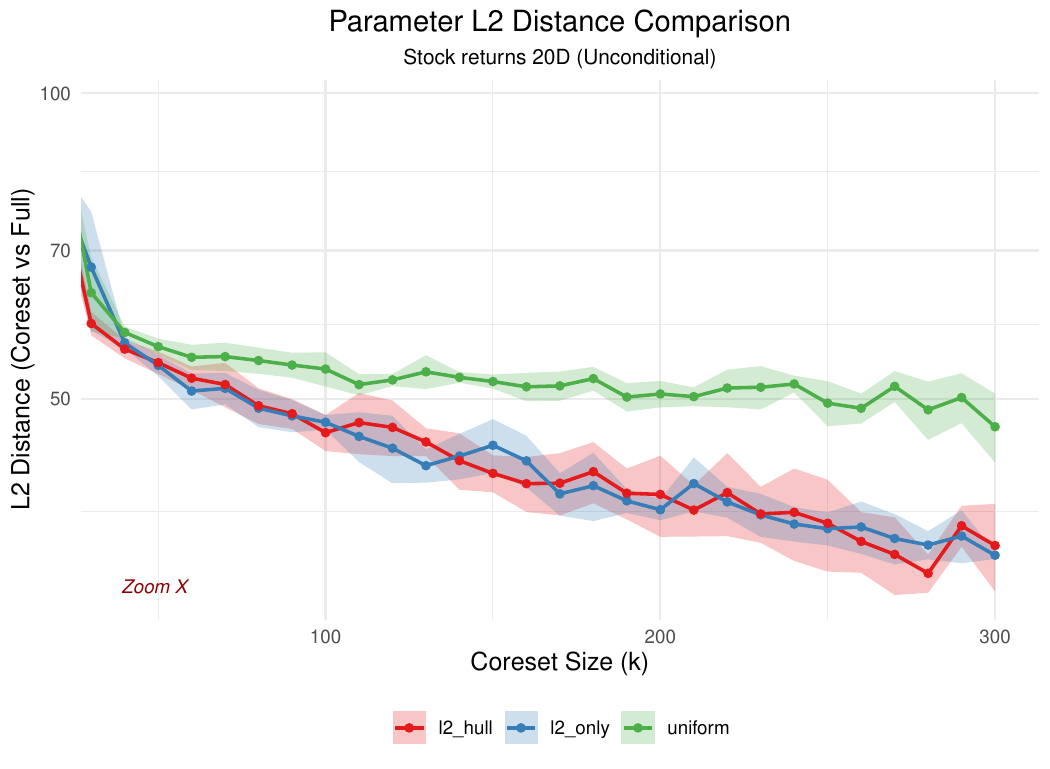}
    \caption{20-stock parameter $\vartheta$ distance }
    \label{fig:20_param_l2}
  \end{subfigure}\hfill
  \begin{subfigure}[t]{0.33\textwidth}
    \centering
    \includegraphics[width=\textwidth]{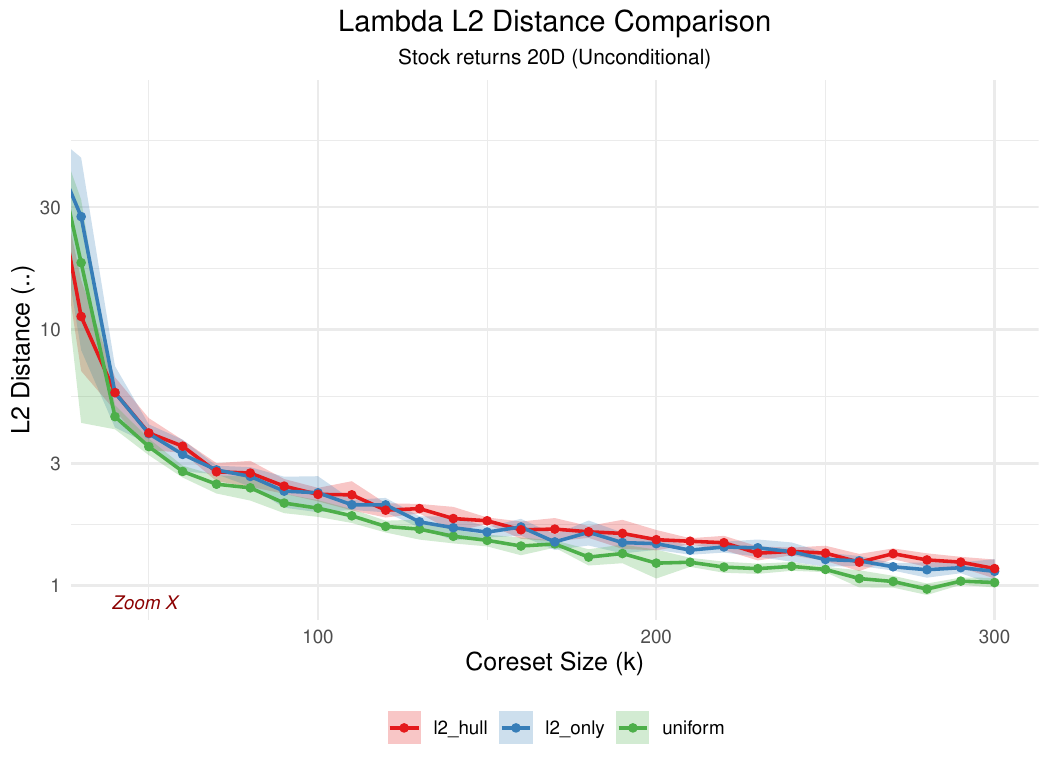}
    \caption{20-stock $\lambda$ distance }
    \label{fig:20_lambda_l2}
  \end{subfigure}
  \caption{Coreset performance comparison on stock-return data. Top row: 10 stocks; bottom row: 20 stocks. 
  Shaded bands indicate $\pm 1$ standard deviation, and solid lines show the averages over multiple repetitions.
  }
  \label{fig:stock_comparison}
\end{figure*}

Through experiments on the real-valued features of the Covertype dataset (as shown in \Cref{tab:covertype_unconditional_performance_filled}), we find that the proposed coreset method (especially the hybrid sampling strategy combining $\ell_2$ sensitivity with convex hull approximation) significantly outperforms uniform sampling and other baselines at different coreset sizes $k\in\{50, 100, 200,500\}$. In particular, when the coreset size is small (e.g., $k=50$), where the performance of uniform sampling deteriorates, our method can effectively preserve the approximation accuracy of the model, demonstrating the ability to capture important key data points. In addition, our coreset method remains stable and maintains its advantage over uniform sampling as the subset size increases. This indicates that our proposed method not only effectively reduces the computational cost on real-world datasets from several hours to a few seconds, but also accurately approximates the benchmark MCTM model fitted to the complete data, thus achieving efficient and accurate large-scale multivariate probabilistic modeling.

\begin{table}[htbp]
\centering
\caption{Covertype data performance for different coreset sizes. Results: mean $\pm$ std over 5 trials. For $\ell_2$ distances, lower is better. For log-likelihood Ratio, closer to 1 is better. Best values highlighted in \textbf{bold}.}
\label{tab:covertype_unconditional_performance_filled}
\footnotesize
\setlength{\tabcolsep}{1pt}
\begin{tabular}{@{}l|c@{\hspace{1pt}}|c@{\hspace{1pt}}|c@{\hspace{1pt}}|c@{\hspace{1pt}}}
\toprule
\textbf{Size} & \textbf{ Method } & \textbf{ Param L2 } & \textbf{ Lambda L2 } & \textbf{LR} \\
\midrule
\multirow{5}{*}{\hspace{.1in}\rotatebox{90}{\tiny $k=50$}} 
& l2-hull & $\mathbf{23.4 \pm 9.8}$ & $\mathbf{18.2 \pm 12.2}$ & $\mathbf{11.4 \pm 6.2}$  \\
& l2-only & $40.7 \pm 24.2$ & $37.0 \pm 26.6$ & $71.8 \pm 8.7$ \\
& ridge-lss & $30.8 \pm 21.5$ & $25.6 \pm 24.3$ & $51.7 \pm 45.7$ \\
& root-l2 & $64.6 \pm 10.4$ & $63.2 \pm 10.7$ & $122.0 \pm 96.1$  \\
& uniform & $52.6 \pm 10.6$ & $50.6 \pm 10.9$ & $84.8 \pm 90.5$ \\
\midrule
\multirow{5}{*}{\hspace{.1in}\rotatebox{90}{\tiny $k=200$}} 
& l2-hull & $\mathbf{14.0 \pm 3.1}$ & $\mathbf{7.4 \pm 5.4}$ & $\mathbf{1.42 \pm 0.13}$  \\
& l2-only & $15.7 \pm 2.3$ & $10.1 \pm 4.0$ & $1.55 \pm 0.33$ \\
& ridge-lss & $16.3 \pm 2.2$ & $11.5 \pm 3.5$ & $3.28 \pm 5.0$ \\
& root-l2 & $28.2 \pm 12.9$ & $24.2 \pm 15.7$ & $14.5 \pm 7.7$ \\
& uniform & $37.2 \pm 4.6$ & $35.2 \pm 4.8$ & $45.1 \pm 16.2$  \\
\midrule
\multirow{5}{*}{\hspace{.1in}\rotatebox{90}{\tiny $k=500$}} 
& l2-hull & $\mathbf{11.1 \pm 2.1}$ & $\mathbf{6.5 \pm 3.5}$ & $\mathbf{1.07 \pm 0.02}$  \\
& l2-only & $15.9 \pm 1.2$ & $11.7 \pm 1.5$ & $1.12 \pm 0.01$ \\
& ridge-lss & $16.3 \pm 1.2$ & $12.2 \pm 1.5$ & $1.27 \pm 0.19$  \\
& root-l2 & $17.4 \pm 9.0$ & $12.7 \pm 11.1$ & $1.27 \pm 0.11$  \\
& uniform & $25.3 \pm 10.5$ & $21.6 \pm 12.8$ & $20.0 \pm 14.0$ \\
\midrule
\multirow{2}{*}{\hspace{.1in}\rotatebox{90}{\tiny Bench.}} 
& $n=100k$ & 0 & 0 & 1  \\
& $n=300k$ & 0 & 0 & 1  \\
\bottomrule
\end{tabular}
\end{table}

Our comparisons of size vs. empirical relative errors consider different methods at fixed coreset sizes. Instead, one could evaluate the reduction in model size that our method achieves compared to its competitors at a fixed error level. However, it is impractical to fix the exact error in advance due to variability of random sampling. The reduction can be seen in \Cref{fig:stock_comparison} (and in similar figures in \Cref{app:experiments}) by drawing a horizontal line at any desired error level to compare the required sample sizes. The plots thus demonstrate that to achieve the same log-likelihood ratio or parameter estimate accuracy, uniform sampling typically requires a lot larger sample size than our proposed method.

\section{DISCUSSION AND EXTENSIONS}
We elaborate on limitations and possible extensions that go beyond the scope of our paper.

\textbf{Choice of copula and basis functions. }
The Gaussian copula has been chosen as an example for \emph{introducing} small coresets to copula models by using their close connection to $\ell_2$ leverage scores. We demonstrate in our $2$-dimensional examples that this already allows very flexible modeling despite a specific formulation. Our coreset method will work as well for other log-concave copulas (with convex level sets). The idea is that such distributions are \emph{similar} to a Gaussian because we can find a so-called John ellipsoid $E$ that is enclosed in a level set and its expansion $\sqrt{d} E$ encloses the same level set. Then, we can derive leverage scores from the quadratic form that describes the ellipsoid as in \citep{TukanMF20}, instead of the leverage scores obtained from the data.
Copulas with non-convex level sets could be handled similarly, depending on the level sets being in some sense \emph{close} to their convex hull.

Our coreset construction is agnostic to the choice of basis functions. The basis function family is thus simply exchangeable. 
We relied on Bernstein polynomials as in prior work by \citet{klein2022multivariate} for convenience, as they allow to easily account for the required monotonicity constraint while being flexible enough to model \emph{arbitrary} marginal distributions. We would thus not expect any significant difference or gains when modeling with alternative choices. However, there may be computational issues to obtain a valid CDF. For instance B-splines have been suggested by \citet{carlan2024bayesian}. Their approach is Bayesian allowing to incorporate monotonicity via appropriate prior choices, which is not directly applicable to our framework, but probably interesting towards a Bayesian extension.

Extending our methods to \emph{conditional} transformation models would be straightforward for a linear conditional structure; it only increases the dimension dependence by the number of features conditioned on. For non-linear structures, the situation is more complicated and one would have to consider different available coreset techniques, e.g., for a specific generalized linear model such as logit, probit, or Poisson.

\textbf{Dimension dependence. }
There are several important aspects regarding the dimension dependence.

First of all, we note that density estimation (with provable guarantees) has \emph{principled} limitations for high-dimensional data. For instance, it is known that the empirical density approaches the original density with an error of $O(1/n^{1/d})$ in terms of Wasserstein distance \citep{NguyenH22a}, thus requiring a sample size of $n=O(1/\eps^d)$ to bring the error term in the order of $\eps$. Existing work is thus limited to very few dimensions.

Our $2$-dimensional simulated examples are mainly meant to visually demonstrate the model's flexibility beyond Gaussian structures, despite the limitation to Gaussian copulas. Our real-world experiments are conducted on $10$ variables (Covertype) and $10$ resp. $20$ variables (Equity returns), exceeding the usual range. We also note that for $10$ dimensions and $7$ basis functions, the dimension of the resulting optimization problem is significantly more than one hundred. As described in \Cref{app:real_world}, for such dimensions and $n\approx 580\,000$ our standard laptop crashes when fitting the full data. With our coreset reduction, we did not experience any such scalability limitations.

The coreset size specified in our theorem is $O(J^2d^2 \polylog(Jd))$, and we give a close lower bound of $\Omega(J^2d)$ corroborating that our methods reducing the double-sum of squares to standard methods do \emph{not} lift to dimensions higher than necessary and do \emph{not} impose restrictions to the spectrum as in prior work \citep{HuangSV20} on comparable loss (only w.r.t. the squared part).
The additional factor $d$ in our upper bound is an artifact of the sensitivity framework used as a black-box in our analysis; recent optimal analyses \citep{MunteanuO24optimallpsampling} suggest that this factor can be eliminated via tighter chaining analyses.

The convex hull component, however, has exponential $\Theta(1/\eta^{(d-1)/2})$ worst-case complexity. Unfortunately, for high-dimensional data, one must rely on their 'mildness' and to heuristics, or preselect a subset of the dimensions using leverage scores or classic PCA \citep{teschke2024detecting,Pearson1901}. In our experiments, heavy-tailed data turned out to be 'hard', for which the approximation deteriorates at fixed coreset size, and it requires an exponential increase of the convex hull component to compensate for this effect.

\textbf{Data streams and distributed data. } 
These computational settings can be realized via black-box reductions using coresets. If data are merely inserted (at possibly different sites) they can be aggregated using Merge \& Reduce \citep[cf.][]{GeppertIMS20,Munteanu23}. To handle deletions, dynamic updates, and sparsity, one usually requires oblivious sketches \citep[e.g.,][]{MunteanuOW21,MunteanuOW23,MaiMM0SW23}. Leverage score sampling, convex hulls and John ellipsoids can also be approximated in data streams \citep{WoodruffY22geometricstreaming,MunteanuO24turnstile}.

\section{SUMMARY AND CONCLUSION}
Our paper focuses on the \emph{coreset} approach to enhance efficiency and scalability of fitting MCTM models to large datasets. Despite MCTM's advantages in flexibly modeling complex multivariate joint distributions, its computational burden increases significantly with increasing sample size, even with moderate output dimension. This hinders its direct application to real-world \emph{Big Data} scenarios. To this end, we developed a novel coreset construction algorithm, which integrates subsampling based on $\ell_2$ leverage scores with convex hull selection. On the one hand, the $\ell_2$ leverage scores ensure that the samples focus on particularly informative observations. On the other hand, the convex hull captures the extreme value patterns of the distribution. In experiments, it has been demonstrated that our \texttt{$\ell_2$-hull} algorithm reduces the data to a negligible proportion of their original size. Moreover, the fitting accuracy (log-likelihood ratio, parameter $\vartheta,\lambda$ distances) are virtually unchanged in contrast to pure $\ell_2$ leverage score sampling and the simplest uniform sampling. Our work comprehensively evaluates fourteen 2-dimensional simulation scenarios and two multivariate real-world datasets with large sample sizes. Our findings demonstrate that \texttt{$\ell_2$-hull} outperforms uniform sampling in $12$ out of $14$ simulated scenarios. Furthermore, it achieves significantly superior statistical performance in the context of multivariate real-world data. The running time of \texttt{$\ell_2$-hull} is comparable to that of uniform sampling, yet it is neither inferior nor merely comparable to pure $\ell_2$ leverage score sampling.

The main contribution of this paper is the first introduction of the coreset method to the complex semi-parametric MCTM framework. Unlike previous coresets, which were limited to parametric models such as (generalized) linear regression, we propose a new sampling scheme that naturally combines $\ell_2$ leverage score sampling with iterative convex hull selection. This allows us to give a rigorous theoretical proof of the error bounds, which ensure that downstream statistical performance measures are retained when applied to large-scale data. Furthermore, we discuss the intrinsic connection between MCTMs and normalizing flows (NFs), which have become popular in machine learning recently: both of them map complex distributions to simple reference distributions through invertible transformations, thus achieving highly flexible distribution modeling. This not only provides a new perspective on the connection of semi-parametric transformation models with deep generative models \citep{herp2025graphicaltransformationmodels}, but also opens new research directions for inserting coreset techniques into NF models or combining MCTMs with NFs more closely in the future. Future research directions of our coreset approach include extensions to semi-parametric additive or mixed models or to Bayesian settings. Online streaming updating, and distributed calculation of our coreset construction can be explored to adapt to time- and location-varying data scenarios. Finally, it is an intriguing question how far our methods can be generalized to build coresets for NFs directly.

\subsubsection*{Acknowledgements}
We thank the anonymous reviewers for their valuable comments. We thank Pia Schreiber for her preliminary work on convex hull approximations.
Alexander Munteanu and Simon Omlor were supported by the German Research Foundation (DFG) -- grant MU 4662/2-1 (535889065) and by the TU Dortmund -- Center for Data Science and Simulation (DoDaS). Zeyu Ding, Katja Ickstadt, and Simon Omlor acknowledge the support of BMBF and MKW.NRW within the Lamarr-Institute for Machine Learning and Artificial Intelligence. Nadja Klein acknowledges support by the German Research Foundation (DFG) through the Emmy Noether grant KL3037/1-1.

{\small

\bibliographystyle{apalike}
\bibliography{Reference}
}

\appendix
\thispagestyle{empty}

\onecolumn
\aistatstitle{Scalable Learning of Multivariate Distributions via Coresets\\Supplementary Materials}

\allowdisplaybreaks
\section{THEORETICAL RESULTS}
\label{app:theory}
\allowdisplaybreaks

\subsection{Preliminaries}
Considering the MCTM model whose negative log-likelihood function can be defined as in \Cref{eq:log_llk}, the goal of the coreset approach is to obtain a subset $C\subset D$ of data, such that the approximation of the original likelihood function is bounded within a factor $(1\pm\epsilon)$.

After applying the basis functions $a$ and their derivatives $a'$ to the raw data, we can assume that for $(i, j) \in [n] \times [J]$ we are given data points $a_{ij} \in \mathbb{R}^d$ and $a_{ij}'\in \mathbb{R}^d$. We use $A, A'\in \mathbb{R}^{nJ\times d}$ to denote the corresponding data matrices.
Now for $i \in [n], j \in [J]$ consider $f(a_{ij}, \vartheta, \lambda)=\frac{1}{2}(\sum_{k=1}^j \lambda_{j, k} \vartheta_k a_{ij} )^2 - \log(\vartheta_j a_{ij}') $ which is the negative logarithm of $g(i, j)={\vartheta_j a_{ij}'} \exp(- \frac{1}{2}(\sum_{k=1}^j \lambda_{j, k} \vartheta_k a_{ij} )^2) $.
We assume that for all $i \in [n], j \in [J]$ and all choices of parameters it holds that $g(i,j) \leq c$ for some constant $c \in \mathbb{R}_{\geq 1}$. This is a natural Lipschitz-type restriction ensuring that the distribution function has a smooth transition from $0$ to $1$ preventing sudden jumps. Note, that this is equivalent to $-\ln(g(i, j))\geq -\ln(c)$.
We further add to the $nJ$ loss contributions a total shift of $\log \mathcal N = nJ(\ln(c)+1)$. This corresponds to a normalization term $\mathcal N$ that ensures non-negativity and thus allows a relative approximation to be meaningful, but does not affect the optimization because it is independent of the parameters that we optimize. Additionally, we assume that there is an intercept, i.e., that for each $(i, j)$ the first coordinate of both $a_{ij}$ and $a_{ij}' $ is $1$ to make it consistent with the presence of intercepts in the transformation functions $h$ defined above.
In the following, we show that if we sample with probabilities that are larger than the $\ell_2$ leverage scores plus a uniform term and add the convex hull of $\{a_{ij}' ~|~ i \in [n], j\in [J]\}$, then we get a coreset for $f(A,\vartheta,\lambda)$ to be minimized as in \Cref{eq:log_llk}. Let $w\in \mathbb{R}^{n\times J}$ be a weight vector. We split the weighted version of $f$ into three parts ($w_{ij}$ are omitted in the unweighted case):
\begin{align*}
    \text{squared part:}\quad&f_1(A, \vartheta, \lambda, w)= \frac{1}{2}\sum\nolimits_{(i, j) \in [n]\times [J]}w_{i j}\left( \sum_{k=1}^j \lambda_{j, k} \langle \vartheta_k, a_{ij} \rangle \right)^2 \\
    \text{positive log part:}\quad&f_2(A, \vartheta, \lambda, w)=\sum\nolimits_{(i, j) \in [n]\times [J]} w_{i j}\max\{\log(\langle \vartheta_j, a'_{ij} \rangle), 0\} \\\quad
    \text{negative log part:}\quad&f_3(A, \vartheta, \lambda, w)=\sum\nolimits_{(i, j) \in [n]\times [J]}w_{i j}\max\{-\log(\langle \vartheta_j, a'_{ij} \rangle), 0\}.
\end{align*}

\subsection{Squared Part}
\label{app:squared}
We let $u_i = \sup_{\|x\|_2=1}\frac{|M_ix|^2}{\|Mx\|_2^2}$ for the $i$-th row of a matrix $M$ denote their $\ell_2$ leverage score. We show that sampling proportionally to the $\ell_2$ leverage scores preserves the squared part $f_1$: given rows $a_{ij} \in \mathbb{R}^d$ where $i \in [n]$ and $j \in [J]$, we are looking for an $\varepsilon$-coreset, which is given by a matrix comprising a subset of rows indexed by $S \subseteq [n] \times [J]$ and corresponding weights $w_{i,j}$ for every $(i,j)\in S$, such that for all $\vartheta_1 , \ldots, \vartheta_J \in \mathbb{R}^d $ and $\lambda \in \mathbb{R}^{J \times J}$ it holds that
\begin{align*}
    \left|\sum_{i=1}^n \sum_{j=1}^J \left(\sum_{k=1}^j\lambda_{j, k} \langle \vartheta_k, a_{ij} \rangle \right)^2  - 
    \sum_{(i, j) \in S} w_{i, j}\right.&\left. \left(\sum_{k=1}^j \lambda_{j, k} \langle \vartheta_k, a_{ij} \rangle \right)^2 \right| \nonumber\\
    &\leq \varepsilon \left| \sum_{i=1}^n \sum_{j=1}^J \left(\sum_{k=1}^j\lambda_{j, k} \langle \vartheta_k, a_{ij} \rangle\right)^2 \right|
\end{align*}

In the following, we arrange data points in a matrix such that sampling rows of this matrix yields a coreset for the function defined above.
We set $B \in \mathbb{R}^{nJ \times dJ^2}$ to be the matrix whose rows equal
$(b_{iJ+j})_{k}= a_{il}$ if $k=(j-1)J+l$ for some $l \in [J]$ and $(b_{iJ+j})_{k}= 0$ otherwise.
Then $B$ consists of $n$ vertically stacked blocks. The $i$-th block, for $i\in [n]$, is defined by
\[
    B_i =
  \left[ {\begin{array}{cccccc}
    b_i & 0 & 0 & 0 & \cdots & 0 \\
    0 & b_i & 0 & 0 &\cdots & 0\\
    0 & 0 & b_i & 0 &\cdots & 0 \\
    \vdots & \vdots & \vdots &  \ddots & \cdots &  0\\
    0& 0 & 0 & 0 & \cdots & b_i \\
  \end{array} } \right]  \in \mathbb{R}^{J\times dJ^2},
\]
where $b_i=(b_{i1}, b_{i2}, \ldots, b_{iJ})$.
The idea is that for any possible parametrization, the squared part can be represented by a product of the new matrix $B$ with the vector $\theta = (\vartheta^T , \lambda^T )^T $. An $\eps$-subspace embedding for $\ell_2$ is therefore sufficient to approximate the squared part, i.e., it yields that $\forall \theta\colon \|B'\theta\|_2^2 = (1\pm\eps) \|B\theta\|_2^2$, where $B'$ consists of few weighted rows subsampled from $B$ according to their $\ell_2$ leverage scores \cite{Woodruff14}.

\lemsquarepart*

\begin{proof}
    We apply $\ell_2$ leverage score sampling to $B$ which gives us a set $S \subseteq [n] \times [J]$ and weights $w_{ij} \in \mathbb{R}_{\geq 1}$ for all $(i,j)\in S$ such that $|S|=O(J^2 d /\varepsilon^2)$ and with high probability for all $x \in \mathbb{R}^{d J^2}$ it holds that
    \[ \left|\Vert Bx \Vert_2^2 - \sum_{(i, j) \in S} w_{i,j} \Vert b_{i, j}x \Vert_2^2  \right| \leq  \varepsilon \Vert Bx \Vert_2^2. \]
    Assume that the statement above holds and consider any parameterization $\vartheta_1 , \dots \vartheta_J \in \mathbb{R}^d $ and $\lambda \in \mathbb{R}^{J \times J}$.
    We define $x \in \mathbb{R}^{dJ^2}$ by $x_{jk}= \lambda_{j,k} \vartheta_k$.
    Then we have that
    \begin{align*}
        &\left|\sum_{i=1}^n \sum_{j=1}^J \left(\sum_{k=1}^j\lambda_{j, k} \langle \vartheta_k, a_{ij} \rangle \right)^2 - \sum_{(i, j) \in S} w_{i, j}\left(\sum_{k=1}^j \lambda_{j, k} \langle \vartheta_k, a_{ij} \rangle \right)^2 \right|\\
        &=  \left|\Vert Bx \Vert_2^2 - \sum_{(i, j) \in S} w_{i,j} \Vert b_{i, j}x \Vert_2^2  \right|
        \leq \varepsilon \Vert Bx \Vert_2^2        =\varepsilon \left| \sum_{i=1}^n \sum_{j=1}^J \left(\sum_{k=1}^j\lambda_{j, j} \langle \vartheta_k, a_{ij} \rangle\right)^2 \right|.
     \end{align*}
\end{proof}

\subsection{Logarithmic Parts}

We now turn our attention to the positive logarithmic part $f_2$. This is going to be handled using the sensitivity framework, which generalizes the previous leverage score sampling for the $\ell_2$ norm to more general families of functions, we refer to \Cref{sec:sensitivityframework} for details. First of all, we bound the VC dimension of the logarithmic function. This is done in a standard way. Using strict monotonicity, the logarithmic function of the inner product can be inverted (respecting their domain and range), relating it to linear classifiers in $d$ dimensions. The latter have a known VC dimension of $d+1$ using classic learning theory results \cite{kearns94clt}. The second part is bounding the total sensitivity, which is the sum of all sensitivity scores $s_i=\sup_{\vartheta,\lambda}\frac{f_2(a_{ij},\theta,\lambda)}{\sum f_2(a_{ij},\theta,\lambda)} $ of single data point contributions. To bound this value, we leverage that for all parameters $\lambda$ and $\vartheta$ it holds that
$\ln(\vartheta_j a_{ij}') -\frac{1}{2}( w_{i, j}\sum_{k=1}^j \lambda_{j, k} \vartheta_k a_{ij} )^2\leq \ln(c),$
which allows us to relate the contribution of the logarithmic part to the squared part up to a constant $\gamma>1$ such that $s_i\leq \gamma(u_i+1/n)$. This allows to reuse the $\ell_2$ leverage scores for this part as well, albeit with an additional uniform component, and with an increase of the sample size, which comes from incorporating the VC dimension and the Lipschitz bound $c$. We also note that the $\eps$-error is relative to $f_1$ instead of $f_2$ directly.

\lemposlogpart*
\begin{proof}
    Our goal is to apply the results of sensitivity sampling.
    For $(i, j)\in [n]\times [J]$ we define $h_{i, j}(\vartheta, \lambda)=\max\{ -\ln(\vartheta_j a_{ij}'), 0 \}$ and $h_0(\vartheta, \lambda)=f_1(A, \vartheta, \lambda)$.
    We set $h(A, \vartheta, \lambda, w)=h_0(\vartheta, \lambda)+\sum_{(i, j)\in [n]\times [J]}h_{i, j}(\vartheta, \lambda)$.
    It holds that the VC-dimension of $ \{ h_{i, j}~|~ (i, j)\in [n]\times [J] \} \cup \{h_0\}$ is bounded by $d+2$ since $ \log(\cdot)$ is a monotonic function and the VC dimension of $\{ h_x:\vartheta \mapsto x \vartheta ~|~ x \in \mathbb{R}^d \}$ is bounded by $d+1$ \cite{kearns94clt}
    and adding the function $h_0$ can only increase the VC-dimension by $1$.

    For the sensitivity we observe the following: since for all parameters $\lambda$ and $\vartheta$ it holds that
    \begin{align*}
        \ln(\vartheta_j a_{ij}') -\frac{1}{2}\left( \sum_{k=1}^j \lambda_{j, k} \vartheta_k a_{ij} \right)^2\leq \ln(c)
    \end{align*}
    it in particular holds that
    \begin{align*}
        \ln(\vartheta_j a_{ij}') -\frac{1}{2}\left( \vartheta_j a_{ij} \right)^2 \leq \ln(c)
    \end{align*}
    Note that there is some $b \in \mathbb{R} $ such that $\vartheta_j a_{ij}'= b\vartheta_j a_{ij}$.
    Since we can scale $\vartheta_j$ arbitrarily we get that
    \begin{align*}
        \ln(bt) -\frac{1}{2}  t^2 \leq \ln(c)
    \end{align*}
    holds for any $t \in \mathbb{R}$.
    Note, that the term on the left hand side is maximized if $t=1$ and thus
    \begin{align*}
        \ln(b) -\frac{1}{2} 1^2 \leq \ln(c)
    \end{align*}
    or equivalently $ b \leq e c$.
    We further have that the derivative of
    \[
        \frac{\ln(bt)}{ \frac{1}{2}t^2}
    \]
    is given by $(t/2 - \ln(bt)t)/(\frac{1}{2}t^2)^2$ which is $0$ if $t=0$ or $ \ln(bt)=1/2$ or equivalently $t=\exp(\frac{1}{2})/b$ which implies that 
    \[
        \frac{\ln(bt)}{ \frac{1}{2} t^2} \leq \frac{\ln(\exp(\frac{1}{2}))}{ (e/2)\left(  1/b \right)^2}= b^2 / e.
    \]
    
    We thus have that  $s_{(i, j)}=(e c^2)u_{i, j}$ is an upper bound for the sensitivity of $h_{i, j}$.
    Consequently the total sensitivity is bounded by $ (e c^2) d J^2$.

    Lastly we have that $h(A, \vartheta, \lambda, w) \leq (e c^2+1)f_1(A, \vartheta, \lambda, w) $. Thus applying sensitivity sampling with error parameter $\varepsilon/c^2$ gives us the desired result.
\end{proof}

Next, we handle the remaining negative logarithmic part given by $f_3$. We note that the asymptote at $0$ precludes finite bounds on the sensitivity. We handle this similarly to \cite{LieM24} by restricting the optimization space to $D(\eta)=\{ (\vartheta, \lambda) ~|~ \forall (i, j) \in [n] \times [J]: \langle \vartheta_j, a_{ij}'\rangle > \eta \}$ comprising only solutions for which the inner product is bounded away by $\eta\geq 0$ from zero. Setting $\eta=0$ corresponds to the original domain.\footnote{We also note that the final choice will be $\eta = \Theta(\eps)$ and negative value correction into the positive domain is common practice and was implemented in \citep{klein2022multivariate} before our theoretical investigations.} By avoiding high sensitivity points in this way, $f_3$ can be bounded in terms of uniform sensitivities together with the VC dimension bound $d+1$ and approximated by invoking the sensitivity framework again.

\lemneglogpart*
\begin{proof}
    Let $e_1 \in \mathbb{R}^d$ be the first unit vector and let $\vartheta', \lambda$ be a feasible solution, i.e. $\vartheta_j' a_{ij}'>0$.
    Then we claim that $(\vartheta, \lambda)$ with $\vartheta= \vartheta'+\eta e_1$ fulfills $f(A, \vartheta, \lambda)\leq f(A, \vartheta', \lambda)+ \eta f_1(A, \vartheta', \lambda) + \eta $.
    Applying this to $ (\vartheta^*, \lambda^*)$ proves the first part of the lemma.
    First note that as $\vartheta_j' a_{ij}'>0$ we have that $(\vartheta, \lambda) \in D(\eta)$.
    Further we have that $$f_3(A, \vartheta, \lambda)-f_2(A, \vartheta, \lambda) \leq f_3(A, \vartheta', \lambda)-f_2(A, \vartheta', \lambda)$$
    as $\vartheta_j a_{ij}' \geq \vartheta_j' a_{ij}'$.

    Lastly note that
    \begin{align*}
        2f_1(A, \vartheta, \lambda)&=\sum_{(i, j) \in [n]\times [J]}\left( \sum_{k=1}^j \lambda_{j, k} \vartheta_k a_{ij} \right)^2\\
        &= \sum_{(i, j) \in [n]\times [J]}\left( \sum_{k=1}^j \lambda_{j, k} (\vartheta_k' +e_1)a_{ij} \right)^2 \\
        &\leq \sum_{(i, j) \in [n]\times [J]}\left( \sum_{k=1}^j \lambda_{j, k} \vartheta_k' a_{ij} \right)^2+ 2J \eta \max\left\{\left(\sum_{k=1}^j \lambda_{j, k} \vartheta_k' a_{ij}\right), 1\right\}+J^2\eta^2 \\
        &\leq \sum_{(i, j) \in [n]\times [J]}\left( \sum_{k=1}^j \lambda_{j, k} \vartheta_k' a_{ij} \right)^2+ 2J \eta \max\left\{\left(\sum_{k=1}^j \lambda_{j, k} \vartheta_k' a_{ij}\right)^2, 1\right\}+J^2\eta^2 \\
        &=(2+4J \eta )f_1(A, \vartheta', \lambda)+2J \eta n + J^2\eta^2 n.
    \end{align*}
 
    Next we show by sampling techniques that with high probability it holds that $|f_3(A, \vartheta, \lambda)-f_3(A(S), \vartheta, \lambda, w)|\leq \eta f_1(A, \vartheta, \lambda) + \eta n$.
    Again we use sensitivity sampling.
    The bound of the VC dimension is again $d+2$.
    For the sensitivity we have that $-\ln( \vartheta_j a_{ij}') \leq \eta^{-1}$.
    Since we only need an absolute error of $\eta n$ we can apply the results of sensitivity sampling.
\end{proof}

\subsection{Main Result}
We start with a technical lemma. With a similar proof technique as in \Cref{lem:main:poslogpart}, we get the following lower bound on the loss function that allows us to charge the previous errors depending only on $f_1$ and additive terms against the original loss function:

\begin{lemma}
\label{lem:main:neglogparthelp}
    For any $ (\vartheta, \lambda) \in D(\eta)$ it holds that $f(A, \vartheta, \lambda) \geq \max \{ nJ, f_1(A, \vartheta, \lambda)/(2\ln(c)) \}$.
\end{lemma}

\begin{proof}
    In previous proof we showed that $\vartheta_j a_{ij}'= b\vartheta_j a_{ij}$. holds for all $(i, j \in [n] \times [J])$ for some $b \leq e c$.
    Consider $\frac{1}{2} t^2- \ln (bt)+\ln(c)+1$.
    For any $t\leq \mathbb{R}_{> 0}$ we have that
    \[
        \frac{1}{2}t^2- \ln (bt)+\ln(c)+1  \geq 1,
    \]
    for any $t\leq \sqrt{2\ln(c)}$ we have that
    \[
        \frac{1}{2}t^2 \leq \ln(c).  
    \]
    and thus
    \[
        \frac{1}{2}t^2- \ln (bt)+\ln(c)+1  \geq \max\left\{1, \frac{1}{2}  t^2\Big/(2\ln(c))\right\}
    \]
    For any $t\geq \sqrt{2\ln(c)}$ we have that
    \[
        \ln (bt) \leq \ln(e c t) \leq \frac{1}{3}t^2 \leq \frac{2}{3}\cdot \frac{1}{2}t^2
    \]
    if $c$ is large enough and thus
    \[
         \frac{1}{2}t^2- \ln (bt)+\ln(c)+1  \geq \frac{1}{2}t^2/(2\ln(c))
    \]
    Thus, we have component-wise that
    \begin{align*}
        \frac{1}{2}w_{i j}\left( \sum_{k=1}^j \lambda_{j, k} \langle \vartheta_k, a_{ij} \rangle \right)^2&-w_{i j}\max\{\log(\langle \vartheta_j, a'_{ij} \rangle), 0\} \\
        &\geq \max\left\{1, \frac{1}{2}w_{i j}\left( \sum_{k=1}^j \lambda_{j, k} \langle \vartheta_k, a_{ij} \rangle \right)^2\Big/(2\ln(c)) \right\}
    \end{align*}
    and thus the lemma follows.
\end{proof}

We get the following theorem by a union bound over \Cref{lem:main:squarepart,lem:main:poslogpart,lem:main:neglogpart} and putting their error bounds together using the triangle inequality. The additive errors of \Cref{lem:main:neglogpart} are further charged against the optimal cost using \Cref{lem:main:neglogparthelp}.

\thmmain*
\begin{proof}
     By Lemma \ref{lem:main:squarepart} we have that $|f_1(A, \vartheta, \lambda)-f_1(A(S), \vartheta, \lambda, w)|\leq \varepsilon f_1(A, \vartheta, \lambda)$ with high probability.
     
     By Lemma \ref{lem:main:poslogpart} we have that $|f_2(A, \vartheta, \lambda)-f_2(A(S), \vartheta, \lambda, w)|\leq \varepsilon f_1(A, \vartheta, \lambda)$ with high probability.

     By Lemma \ref{lem:main:neglogpart} we have that $|f_3(A, \vartheta, \lambda)-f_3(A(S), \vartheta, \lambda, w)|\leq \eta f_1(A, \vartheta, \lambda) + \eta n$ with high probability.

     Overall, by the triangle inequality, we have that $|f(A, \vartheta, \lambda)-f(A(S), \vartheta, \lambda, w)|\leq 2\varepsilon f_1(A,\vartheta, \lambda) + \eta f_1(A, \vartheta, \lambda) + \eta n$.

     By \Cref{lem:main:neglogparthelp} we have that $f(A, \vartheta, \lambda) \geq \max \{ nJ, f_1(A, \vartheta, \lambda)/(2\ln(c)) \}$.
     Substituting $\varepsilon/(2J\ln(c))$ for $\eta$ yields the desired bound  $\forall (\vartheta, \lambda) \in D(\eta): |f(A, \vartheta, \lambda)-f(A(S), \vartheta, \lambda, w)|\leq \varepsilon f(A, \vartheta, \lambda)$

     By \Cref{lem:main:neglogpart} and using \Cref{lem:main:neglogparthelp} again,  there exists $(\vartheta, \lambda)\in D(\eta)$ such that $f(A, \vartheta, \lambda) \leq f(A, \vartheta^*, \lambda^*) + 2 \eta J f_1(A, \vartheta^*, \lambda^*) + J \eta n + J^2 \eta ^2 n \leq (1+O(\eps)) f(A,\vartheta^*, \lambda^*),$ where $(\vartheta^*, \lambda^*)$ is an optimal solution.
\end{proof}

We remark that the convex hull can comprise $\Omega(nJ)$ points, however, different $\eta$-kernel coresets of size $\Theta(1/\eta^{(d-1)/2})$ exist for the problem \citep{AgarwalHV04,Chan04}, surveyed in \citep{AgarwalHV05}, in the field of computational geometry. These $\eta$-kernels also match the requirements of the shifted domain $D(\eta)$. We discuss our particular choice \citep{blum2019sparse} in the experimental section.

\subsection{Lower Bounds}

It suffices to prove a lower bound for the squared part. In particular the bound will hold against preserving the subspace (equivalently the rank) spanned by the data. 
In the following, we give two lower bounds under different natural assumptions on the $\lambda$ coefficients that define the dependence structure of the Gaussian copula. 
The first one is a weaker lower bound that holds without assumptions on $\lambda$ other than its lower triangular structure.
The second lower bound is stronger and holds under additional but common assumptions on $\lambda$.

\lemlowerone*
\begin{proof}
    Consider the instance with $n=(J-1)^2 d $ consisting of $J d$ blocks $\{ a_ij \}_{t j \in [J]} $ where $t \in J\times d$.
    More precisely for $t=(j_0, k) \in J \times J \times d$ we have that
    \begin{align*}
        a_{tj}=
        \begin{cases}
            e_k, & \text{if $j \geq j_0\}$}\\
            0 & \text{else}
        \end{cases}
    \end{align*}
    We set $A \in \mathbb{R}^{j^2d \times J^2d}$ to be the matrix with parameters that represents these rows.
    For $t= (3, k)$ the $t$-th block $A_t \in \mathbb{R}^{J \times J d}$ of the matrix $A$ looks as follows:
    \[
    A_t =
  \left[ {\begin{array}{ccccccc}
    0 & 0 & 0 & 0 & 0 &\cdots & 0 \\
    0 & 0 & 0 & 0 & 0 &\cdots & 0\\
    0 & 0 & e_k & 0 &0 &\cdots & 0 \\
    0 & 0 & e_k & e_k &0 &\cdots & 0 \\
    0 & 0 & e_k & e_k &e_k &\vdots & 0 \\
    \vdots & \vdots & \vdots & \vdots & \vdots & \ddots  &  0\\
    0 & 0 & e_k & e_k & e_k & \cdots & e_k \\
  \end{array} } \right]
    \]
 and the matrix itself has the form
    \[
    A =
  \left[ {\begin{array}{c}
    A_{1, 1} \\
    A_{1, 2} \\
    \vdots \\
    A_{1, J}\\
    A_{2, 1}\\
    \vdots \\
    A_{d, J}\\
  \end{array} } \right]
    \]
    Note that the matrix $A \in \mathbb{R}^{J^2\times d}$ with rows $(a_{ij})$ is of rank at least  $d J(J-1)$ even if we restrict the parameter space by requiring $\lambda_{ij}=0$ for all $i,j  \in [n]$ with $j>i $ and $|\lambda_{ij}|\leq 1 $ for all $i, j \in [J]$.. 
    To see this observe that for each $k \in [d]$ and $j_1$ and $j_2$ with $j_1 \leq j_2$ there is a parametrization such that only the contribution from row $j_2 $ from block $(k, j_1)$ is non zero:
    \begin{align*}
        &\lambda_{j_2j_1}=1 && \text{~}\\
        &\lambda_{j_2(j_1-1)}=-1 &&\text{if $j_1 < j_2$ and $j_1 \geq 1$}\\
        &\lambda_{jl}=0 &&\text{else}\\
        &\vartheta_k=e_k &&\text{for $k = j_0$}\\
        &\vartheta_k=0&& \text{else.}
    \end{align*}
    Thus since any coreset coreset preserves the rank of the matrix any coreset must also be of rank at least $d J(J-1)$ and thus of size at least $\Omega(d J^2) $.
\end{proof}

\lemlowertwo*
\begin{proof}
    Consider the instance with $n=d$ and $a_{ij}=e_i $ for all $i \in [n]$ and $j \in [J]$.
    Now consider any instance $\{ a'_{ij} \}_{i \in [n], j \in [J]} $ with at least one zero entry, i.e. $a'_{i_0j_0}=0$ for some $i_0 \in [n], j_0 \in [J]$.
    Then $A'$ cannot be a coreset of $A$ for the following reason:
    let $A' \in \mathbb{R}^{n \times d}$ be the matrix consisting of rows $a'_{1j_0}, a'_{2j_0}, \dots a'_{nj_0}$.
    Since $n=d$ and $a'_{i_0j_0}=0$ there exists $\beta \in ker(A')\setminus\{0\}$.
    Consider the following parameterizations:
    \begin{align*}
        &\lambda_{jl}=0 && \text{for $j > l$ and $l=j_0$}\\
        &\lambda_{jj}=1 &&\text{for all $j \in [J]$}\\
        &\lambda_{jl}=0 &&\text{else}\\
        &\vartheta_k=\beta &&\text{for $k = j_0$}\\
        &\vartheta_k=0&& \text{else}
    \end{align*}
    Now we have that
    \[
        \sum_{i=1}^n \sum_{j=1}^J \left(\sum_{k=1}^j\lambda_{j, j} \vartheta_k a_{ij}\right)^2=\Vert \beta \Vert_2^2 > 0
    \]
    and
    \[
        \sum_{i=1}^n \sum_{j=1}^J \left(\sum_{k=1}^j\lambda_{j, j} \vartheta_k a'_{ij}\right)^2=0
    \]
    thus $\{ a'_{ij} \}_{i \in [n], j \in [J]} $ cannot be a coreset.
\end{proof}

\section{SENSITIVITY SAMPLING FRAMEWORK}
\label{sec:sensitivityframework}

The sensitivity sampling framework, as outlined and most recently updated in \citep{feldman2020turning}, provides a methodology for generating coresets for optimization problems where the objective is to reduce the cost associated with and aggregate over the input data points. In this method, data points are selected at random, yet the selection probability for each point is proportional to its sensitivity with respect to the optimization problem, as an importance measure. This technique is designed to ensure the inclusion of pivotal points that might otherwise be missed under an unbiased uniform random selection due to the equally low probability of selection.
\begin{definition}[Sensitivity \citep{langberg2010universal}]
Let $ F = \{g_1, \ldots, g_n\} $ be a family of functions that map from $ \mathbb{R}^d $ to the non-negative real numbers, weighted with $ w \in \mathbb{R}^n_{>0} $. The sensitivity of $ g_i $ for $ f_w(\theta) = \sum_{j=1}^{n} w_j \cdot g(x_j \theta) $ is
\[ \zeta_i = \sup_{\theta \in \mathbb{R}^d, f_w(\theta)>0} \frac{w_i g_i(\theta)}{f_w(\theta)}. \]
The sum of the sensitivities $ Z = \sum_{i=1}^{n} \zeta_i $ is called the total sensitivity.
\end{definition}
The sensitivity sampling framework has demonstrated considerable advantages in the construction of coresets across various computational problems and statistical models, including linear regression \citep{clarkson2005subgradient,DrineasMM06,RudelsonV07,dasgupta2009sampling}, logistic regression \citep{munteanu2018coresets}, probit regression \citep{ding2024scalable}, and Poisson regression \citep{MolinaMK18,LieM24}.

Sampling with probabilities proportional to sensitivity scores provably leads to a good approximation, although it requires the determination of the exact sensitivities to solve the original problem. However, it has been shown that the sample can also be drawn proportionally to any upper bounds such that $ S=\sum_{i=1}^n s_i \geq \sum^n_{i=1}\zeta_i = Z$. Therefore, in order to be able to build a coreset using sensitivity sampling, it suffices to find the upper bounds to approximate a sensitivity. However, since the total sensitivity determines the sample size, this overestimation must be controlled carefully. 

The following theorem builds the core for sensitivity sampling based coreset construction and is presented in its most recently updated and optimized version due to \citep{feldman2020turning}.

\begin{theorem}[\citealp{langberg2010universal,feldman2011unified,feldman2020turning}]
Let $ F = \{g_1, \ldots, g_n\} $ be a finite set of functions that map from $ \mathbb{R}^d $ to $ \mathbb{R}_{\geq 0} $ and let $ w \in \mathbb{R}^n_{>0} $ be a vector of positive weights. Let $ \varepsilon $, $ \delta $ be in $ (0, 1/2) $. Moreover, let $ s_i \geq \zeta_i $ be upper bounds for the sensitivities and $ S = \sum_{i=1}^{n} s_i \geq Z $. Then for given $ s_i $, a set $ R \subseteq F $ can be found in time $ O(|F|) $ with
\[ |R| = O\left( \frac{S}{\varepsilon^2} \left( \Delta \log S + \log \left( \frac{1}{\delta} \right) \right) \right) \]
With the calculations of the weighted functions, such that with probability \(1 - \delta\) for all \(\theta \in \mathbb{R}^d\) it holds:
\[ (1 - \varepsilon) \sum_{g \in F} w_i g_i(\theta) \leq \sum_{g \in R} u_i g_i(\theta) \leq (1 + \varepsilon) \sum_{g \in F} w_i g_i(\theta) \]
Each element of \(R\) is drawn i.i.d. with probability \(p_j = \frac{s_j}{S}\) from \(F\), \(u_i = \frac{Sw_j}{s_j |R|}\) denotes the weight of a function \(g_i \in R\), which corresponds to \(g_i \in F\), and \(\Delta\) is an upper bound for the VC dimension of the Range Spaces \(\mathcal{R}_{F^*}\), induced by \(F^*\) that we obtain by scaling all functions \(g_i \in F\) with \(\frac{Sw_j}{s_j |R|}\). It is thus

\[ F^* = \left\{ \frac{Sw_j}{s_j |R|} g_j(\theta) ~\Big|~ j \in [n] \right\}.\]

\end{theorem}

The set of functions  $R \subseteq F$ with new weights $u$ is a representation of our sought coreset. The size of $R$ depends on both, the estimate of the total sensitivity, as well as the VC-dimension of the range spaces, which is induced by the reweighted function set  $F^*$. Since the construction of the coreset is a probabilistic process, it cannot be ruled out that this may fail (unless we choose all data points for our coreset). The theorem thus introduces an controllable error probability $\delta$, which also influences the size of the coreset logarithmically.

\section{ALGORITHMS}
\label{app:algorithm}
\begin{algorithm}[!ht]
\caption{Hybrid coreset construction for MCTMs}
\label{algorithm:mctm_coreset_sparse}
\SetAlgoLined
\DontPrintSemicolon
\KwData{Full dataset $\mathcal{D} = \{y_i\}_{i=1}^n \subset \mathbb{R}^J$, target coreset size $k$}
\KwResult{Weighted coreset $\mathcal{C} = \{(y_j, w_j)\}_{j=1}^k$}
\BlankLine

\textbf{Compute transformed statistics}: Apply Bernstein polynomial transformation $a$ of degree $d$ to each $y_i$ to obtain $B \in \mathbb{R}^{nJ\times dJ^2}$ as described in \Cref{sec:squared} resp. \Cref{app:squared}\;

\textbf{Compute $\ell_2$ leverage scores of $B$}: Use fast leverage score computation as in Theorem 2.13 of \cite{Woodruff14} to get a constant factor approximation for $u_i=\sup_{x\neq 0} \frac{|B_ix|^2}{\|Bx\|_2^2}$ for each $i\in[nJ]$.

\textbf{Compute sensitivity proxy}: For each $i$, set $s_i \leftarrow u_i+1/n$ as a sensitivity score\;

\textbf{Normalize to probabilities}: $p_i \leftarrow \frac{s_i}{\sum_{j=1}^n s_j}$\;

\textbf{Sampling phase}:

  \Indp (a) Let $k_1 = \lfloor \alpha k \rfloor$ (e.g. $\alpha = 0.8$) be the size of the sensitivity sample\;
  
  (b) Sample $k_1$ points independently with probability $\{p_i\}$ \;
  
  (c) Assign weights $w_i \leftarrow \frac{1}{k_1 \cdot p_i}$ for each sampled point\;

\Indm

\textbf{Convex hull augmentation}:

  \Indp (a) Let $k_2 = k - k_1$\;
  
  (b) Compute $\varepsilon/J$-kernel convex hull approximation as in \cite{blum2019sparse} over the derivatives $\{a_{j}'(y_{ij})\}_{i\in[n], j\in[J]}$ and select $k_2$ extremal points\;
  
  (c) Assign weight $w_j \leftarrow 1$ to each convex hull point\;

\Indm

\textbf{Form final coreset}: Combine sampled points and hull points into $\mathcal{C}$ with associated weights $w$\;

\textbf{Coreset fitting}: Fit MCTM using weighted log-likelihood:
\[
\hat{\theta}_{\text{coreset}} = \operatorname{argmax}_{\theta} \sum_{(Y_j, w_j) \in \mathcal{C}} w_j \cdot \log f_\theta(Y_j)
\]
\end{algorithm}

\begin{algorithm}[!ht]
\caption{Sparse approximation of the convex hull \citep{blum2019sparse}}
\label{alg:epskernel}
\SetAlgoLined
\DontPrintSemicolon
\KwData{A set of points $P$, a query point $q$, tolerance $\epsilon$}
\KwResult{Approximated convex hull point $t_M$ closest to $q$}
\BlankLine

Initialize first two points: randomly select one point $a_0$ and obtain $a_1$ as the furthest point to $a_0$

Obtain the third point $a_2$, which is the furthest point to the line $\overline{a_0a_1}$. The set $\{a_0, a_1, a_2\}$ compose as the initial convex hull.

\For{$j \leftarrow 1$ \KwTo $n-3$}{

$t_0 \leftarrow$ closest point of $P$ to $q$\;
$M \leftarrow O(1/\epsilon^2)$\;

\For{$i \leftarrow 1$ \KwTo $M$}{
  $v_i \leftarrow q - t_{i-1}$\;
  
  $p_i \leftarrow$ point in $P$ that is extremal in the direction of $v_i$\;
  
  \eIf{$p_i$ exists}{
    Compute the projection of $q$ onto the line through $t_{i-1}$ and $p_i$ to find $t_i$\;
    
    $t_i \leftarrow$ the closest point to $q$ on the line segment $s_i = t_{i-1}p_i$\;
  }{
    $t_i \leftarrow t_{i-1}$\;
    \textbf{break}\;
  }
  \If{$\|q - t_i\| < \epsilon$ or $i = M$}{
    \textbf{return} $t_i$\;
  }
}
}
\end{algorithm}

\section{RELATIONSHIP BETWEEN MCTMS AND NORMALIZING FLOWS}
\label{app:NFs}
Multivariate Conditional Transformation Models (MCTMs) and Normalizing Flows (NFs) are both powerful frameworks for modeling complex conditional probability distributions \( p_Y(y \mid x) \). They are built upon the same mathematical foundation, namely the probability transformation formula, which enables density estimation by transforming a complex target distribution into a simpler base distribution through a bijective and differentiable function.

Let \( Y \in \mathbb{R}^J \) be the response variable whose conditional density \( p_Y(y \mid x) \) is of interest, and let \( Z \in \mathbb{R}^J \) be a latent variable with a known base distribution \( p_Z(z) \), typically a standard multivariate Gaussian. Both frameworks assume the existence of a transformation function \( {h}: \mathbb{R}^J \times \mathcal{X} \to \mathbb{R}^J \) that is bijective in \( y \) for each fixed \( x \), such that
\[
Z = {h}(Y \mid x).
\]
The inverse function \( {g} = {h}^{-1} \) satisfies
\[
Y = {g}(Z \mid x).
\]
By the probability transformation formula, the conditional density of \( Y \) given \( x \) can be expressed as 
\[
p_Y(y \mid x) = p_Z({h}(y \mid x)) \left| \det \left( \frac{\partial {h}(y \mid x)}{\partial y} \right) \right|.
\]
Taking the logarithm yields
\[
\log p_Y(y \mid x) = \log p_Z({h}(y \mid x)) + \log \left| \det \left( \frac{\partial {h}(y \mid x)}{\partial y} \right) \right|,
\]
which is the objective typically maximized in both MCTM and NF during maximum likelihood estimation respectively training.

In MCTM, the transformation \( {h}(y \mid x) \) is defined in a structured, semi-parametric way. It follows a lower triangular structure such that the \( j \)-th component \( h_j(y \mid x) \) depends only on \( y_1, \ldots, y_j \) and \( x \). This triangular form ensures that the Jacobian matrix \( \frac{\partial {h}(y \mid x)}{\partial y} \) is lower triangular. Each \( h_j \) is further decomposed as a linear combination of marginal transformations \( \tilde{h}_\ell(y_\ell \mid x) \) for \( \ell < j \), and its own marginal \( \tilde{h}_j(y_j \mid x) \). Formally,
\[
h_j(y_1, \ldots, y_j \mid x) = \sum_{\ell=1}^{j-1} \lambda_{j\ell}(x) \tilde{h}_\ell(y_\ell \mid x) + \tilde{h}_j(y_j \mid x).
\]
In matrix-vector notation, the full transformation is written as \( {h}(y \mid x) = \Lambda(x) \tilde{{h}}(y \mid x) \), where \( \tilde{{h}}(y \mid x) = (\tilde{h}_1(y_1 \mid x), \ldots, \tilde{h}_J(y_J \mid x))^T  \), and \( \Lambda(x) \) is a lower triangular matrix with ones on the diagonal. Each marginal transformation \( \tilde{h}_j \) is modeled using a basis function expansion of the form \( \tilde{h}_j(y_j \mid x) = {a}_j(y_j)^T  {\vartheta}_j(x) \), where \( {a}_j(y_j) \) is a fixed basis vector and \( {\vartheta}_j(x) \) is a parameter vector that may depend on \( x \). Monotonicity is enforced by constraining the derivative \( \frac{\partial \tilde{h}_j}{\partial y_j} > 0 \). The Jacobian determinant simplifies significantly due to the triangular structure: it is the product of the marginal derivatives,
\[
\det \left( \frac{\partial {h}(y \mid x)}{\partial y} \right) = \prod_{j=1}^J \frac{\partial \tilde{h}_j(y_j \mid x)}{\partial y_j} = \prod_{j=1}^J ({a}'_j(y_j))^T  {\vartheta}_j(x).
\]
Therefore, the conditional log-likelihood becomes
\[
\log p_Y(y \mid x) = \log \phi_J(\Lambda(x) \tilde{{h}}(y \mid x)) + \sum_{j=1}^J \log (({a}'_j(y_j))^T  {\vartheta}_j(x)),
\]
where \( \phi_J \) denotes the density of the standard \( J \)-variate normal distribution.

Normalizing Flows employ the same principle of density transformation, but define \( {h}(y \mid x) \) or its inverse \( {g}(z \mid x) \) through a sequence of invertible mappings composed of simple building blocks. These mappings are typically parameterized using deep neural networks. A transformation in NF is written as \( {h} = h_L \circ \cdots \circ h_1 \), where each \( h_\ell \) is a bijective transformation with tractable Jacobian determinant. The parameters of each transformation may be functions of \( x \), enabling conditional modeling.

Popular designs include autoregressive flows, where each output \( h_j \) depends only on \( y_1, \ldots, y_{j-1} \) and \( x \), resulting in a lower triangular Jacobian similar to MCTMs. Another class of NF architectures is based on coupling layers, where the input vector is split into two parts, one of which remains fixed while the other is transformed using an affine function whose parameters are learned from the fixed part and the context \( x \). These transformations are composed repeatedly, with role reversal between parts in successive layers to ensure expressiveness.

Despite their differences in parameterization and implementation, MCTM and Normalizing Flows share the same mathematical foundation and ultimately rely on the same probability transformation principle. Both frameworks define a transformation from the observed data space to a latent space with a known density, compute the associated Jacobian determinant, and maximize the resulting log-likelihood over the data.

\section{EXPERIMENTAL DESIGN}
\label{app:experiments}

\subsection{Simulation Study}

\subsubsection{Data Generation Processes}
\label{app:sim_DGP}
To thoroughly evaluate our proposed method across diverse dependency structures, we implemented 14 different data generation processes. Each process was designed to test specific aspects of dependency modeling. For all processes, we generated datasets with $n=10\,000$ samples and evaluated coreset performance at sizes 30 and 100. Below, we describe each process mathematically.

\textbf{1. Bivariate Normal Distribution}

The standard bivariate normal distribution with correlation parameter $\rho$:
\begin{equation*}
(Y_1, Y_2) \sim \mathcal{N}\left(\begin{bmatrix} 0 \\ 0 \end{bmatrix}, \begin{bmatrix} 1 & \rho \\ \rho & 1 \end{bmatrix}\right)
\end{equation*}
where $\rho = 0.7$ represents the correlation between variables. This baseline case tests performance on linear dependency structures.

\textbf{2. Non-linear Correlation}

A non-linear correlation structure where the correlation coefficient varies with the predictor:
\begin{align*}
Y_1 &= X^2 + \epsilon_1, \quad \epsilon_1 \sim \mathcal{N}(0, 0.5^2) \\
Y_2 &\sim \mathcal{N}(0, 1), \text{ with correlation } \rho(X) = \sin(X) \text{ to } Y_1
\end{align*}
where $X \in [-3, 3]$. This tests the ability to capture location-dependent correlation.

\textbf{3. Bivariate Normal Mixture}

A mixture of two bivariate normal distributions with different means and correlation structures:
\begin{align*}
(Y_1, Y_2) &\sim 0.5 \cdot \mathcal{N}\left(\begin{bmatrix} 0 \\ 0 \end{bmatrix}, \begin{bmatrix} 1 & 0.8 \\ 0.8 & 1 \end{bmatrix}\right) \\
&+ 0.5 \cdot \mathcal{N}\left(\begin{bmatrix} 3 \\ -2 \end{bmatrix}, \begin{bmatrix} 1.5 & -0.5 \\ -0.5 & 1.5 \end{bmatrix}\right)
\end{align*}
This tests performance on multimodal data with different local dependency structures.

\textbf{4. Geometric Mixed Distribution}

A distribution combining two distinct geometric patterns:
\begin{align*}
(Y_1, Y_2) &\sim 0.5 \cdot \text{Circular} + 0.5 \cdot \text{Cross}
\end{align*}
where the Circular component generates points along a circle with radius $r \sim \mathcal{N}(2, 0.2^2)$ and uniform angle $\theta \sim \text{Uniform}(0, 2\pi)$, while the Cross component generates points along two perpendicular lines with controlled variance. This distribution tests the ability to capture multiple geometric structures simultaneously, with continuous circular features intersecting with linear patterns. The sharp differences in geometric structure and local data density create a particularly challenging scenario for coreset selection.

\textbf{5. Skewed t-Distribution}

A heavy-tailed distribution with skewness:
\begin{equation*}
(Y_1, Y_2) \sim \text{Skew-}t_{\nu}\left(\xi, \Omega, \alpha\right)
\end{equation*}
where $\xi = [0, 0]^T $ is the location parameter, $\Omega = \begin{bmatrix} 1 & 0.5 \\ 0.5 & 1 \end{bmatrix}$ is the scale matrix, $\alpha = [5, -3]^T $ is the skewness parameter, and $\nu = 4$ specifies the degrees of freedom. This tests performance on asymmetric, heavy-tailed distributions.

\textbf{6. Heteroscedastic Distribution}

A distribution where the variance depends on the location:
\begin{align*}
Y_1 &\sim \mathcal{N}(X^2, \sigma_1^2(X)) \text{ where } \sigma_1(X) = e^{0.5X} \\
Y_2 &\sim \mathcal{N}(\sin(X), \sigma_2^2(X)) \text{ where } \sigma_2(X) = \sqrt{|X|}
\end{align*}
with $X \in [-3, 3]$. This tests the ability to capture variance heterogeneity.

\textbf{7. Copula Complex Distribution}

A Clayton copula-based dependency structure with gamma and log-normal marginals:
\begin{align*}
(U_1, U_2) &\sim C_{\text{Clayton}}(\theta = 2) \\
Y_1 &= F_{\text{Gamma}(2,1)}^{-1}(U_1) \\
Y_2 &= F_{\text{LogNormal}(0,1)}^{-1}(U_2)
\end{align*}
This tests performance on complex dependency structures with non-normal marginals.

\textbf{8. Spiral Dependency}

A spiral pattern with added noise:
\begin{align*}
t &\in [0, 3\pi] \\
r &= 0.5t \\
Y_1 &= r\cos(t) + \epsilon_1, \quad \epsilon_1 \sim \mathcal{N}(0, 0.5^2) \\
Y_2 &= r\sin(t) + \epsilon_2, \quad \epsilon_2 \sim \mathcal{N}(0, 0.5^2)
\end{align*}
This tests performance on complex geometric dependencies.

\textbf{9. Circular Dependency}

A circular pattern with radius variation:
\begin{align*}
\theta &\sim \text{Uniform}(0, 2\pi) \\
r &\sim \mathcal{N}(5, 1^2) \\
Y_1 &= r\cos(\theta) \\
Y_2 &= r\sin(\theta)
\end{align*}
This tests the ability to capture circular dependencies where linear correlation measures fail.

\textbf{10. t-Copula Dependency}

A t-copula with t and exponential marginals:
\begin{align*}
(U_1, U_2) &\sim C_{t}(\rho = 0.7, \nu = 3) \\
Y_1 &= F_{t_5}^{-1}(U_1) \\
Y_2 &= F_{\text{Exp}(1)}^{-1}(U_2)
\end{align*}
This tests performance on tail dependencies with asymmetric marginals.

\textbf{11. Piecewise Dependency}

A piecewise function with different correlation regimes:
\begin{align*}
Y_1 &\sim \mathcal{N}(0, 2^2) \\
Y_2 &= 
\begin{cases} 
1.5Y_1 + \epsilon_1, & \text{if } Y_1 < -1 \\
-0.5Y_1 + \epsilon_2, & \text{if } -1 \leq Y_1 < 1 \\
-2Y_1 + \epsilon_3, & \text{if } Y_1 \geq 1
\end{cases}
\end{align*}
where $\epsilon_1, \epsilon_3 \sim \mathcal{N}(0, 0.5^2)$ and $\epsilon_2 \sim \mathcal{N}(0, 0.8^2)$. This tests ability to capture regime-dependent correlations.

\textbf{12. Hourglass Dependency}

A heteroscedastic pattern with variance increasing quadratically from center:
\begin{align*}
Y_1 &\sim \mathcal{N}(0, 2^2) \\
Y_2 &\sim \mathcal{N}(0, \sigma^2(Y_1)) \text{ where } \sigma^2(Y_1) = 0.2 + 0.3Y_1^2
\end{align*}
This tests ability to capture complex variance structures.

\textbf{13. Bimodal Clusters}

Two distinct clusters with opposing correlation structures:
\begin{align*}
(Y_1, Y_2) &\sim 0.5 \cdot \mathcal{N}\left(\begin{bmatrix} -2 \\ 2 \end{bmatrix}, \begin{bmatrix} 1 & 0.8 \\ 0.8 & 1 \end{bmatrix}\right) \\
&+ 0.5 \cdot \mathcal{N}\left(\begin{bmatrix} 2 \\ 2 \end{bmatrix}, \begin{bmatrix} 1 & -0.7 \\ -0.7 & 1 \end{bmatrix}\right)
\end{align*}
This tests the ability to capture cluster-specific dependency structures.

\textbf{14. Sinusoidal Dependency}

A sinusoidal relationship with noise:
\begin{align*}
Y_1 &\in [-3, 3] \\
Y_2 &= 2\sin(\pi Y_1) + \epsilon, \quad \epsilon \sim \mathcal{N}(0, 0.5^2)
\end{align*}
This tests performance on periodic dependencies.

These 14 data generation processes provide a comprehensive evaluation framework for testing our coreset method across a wide range of dependency structures, from simple linear correlations to complex geometric patterns and regime-dependent relationships.

\subsubsection{Simulation Data Visualization}
\label{app:sim_visual}

We visualize each data generation process with coresets created using three different sampling methods: Uniform Sampling, $\ell_2$ Sensitivity Sampling, and our proposed $\ell_2$-Hull Sampling method. Each visualization shows how well the different coreset methods preserve the underlying data distribution structure.

\begin{figure}[htbp]
    \centering
    \begin{subfigure}[b]{0.32\textwidth}
        \includegraphics[width=\textwidth]{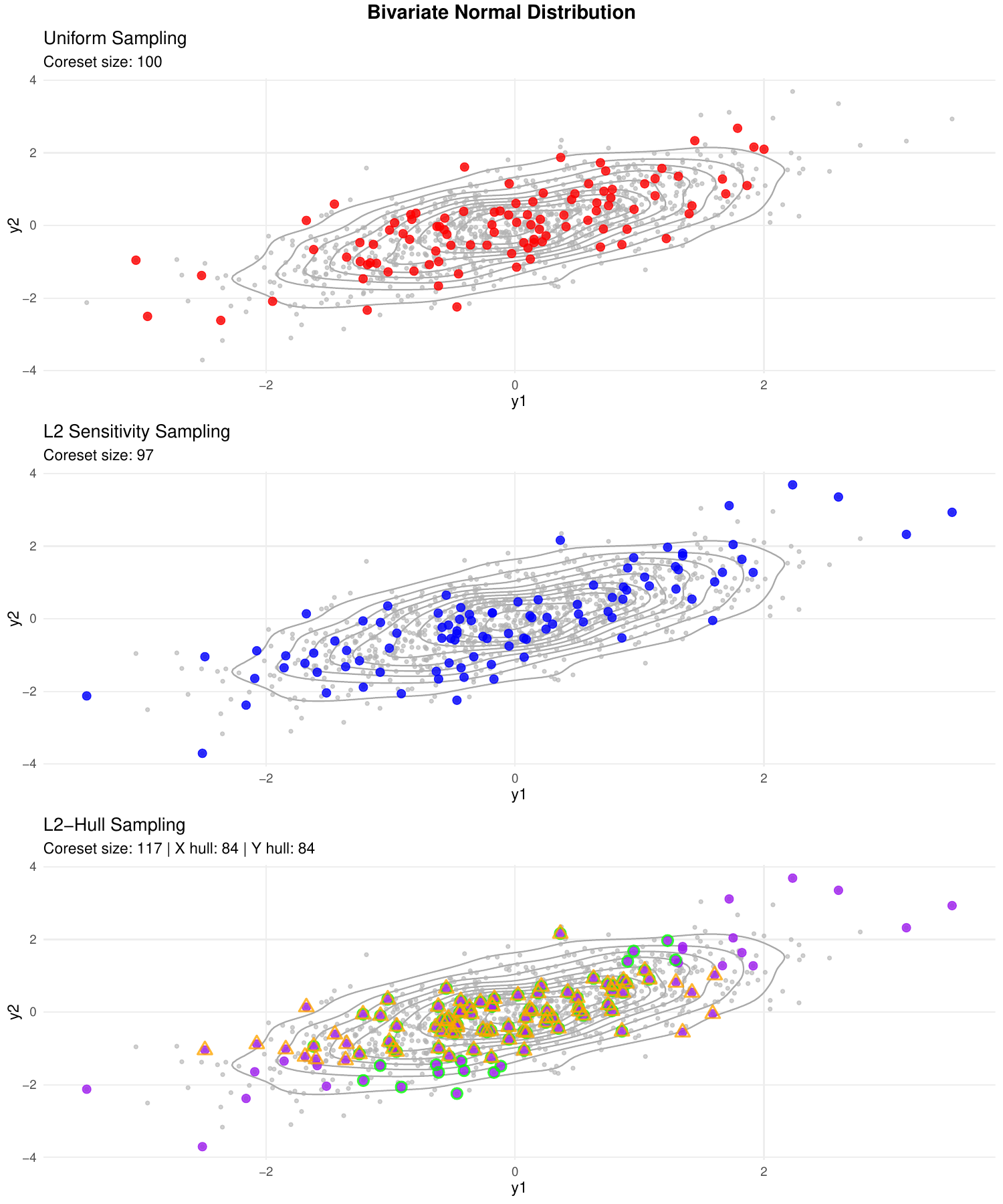}
        \caption{Bivariate Normal}
    \end{subfigure}
    \hfill
    \begin{subfigure}[b]{0.32\textwidth}
        \includegraphics[width=\textwidth]{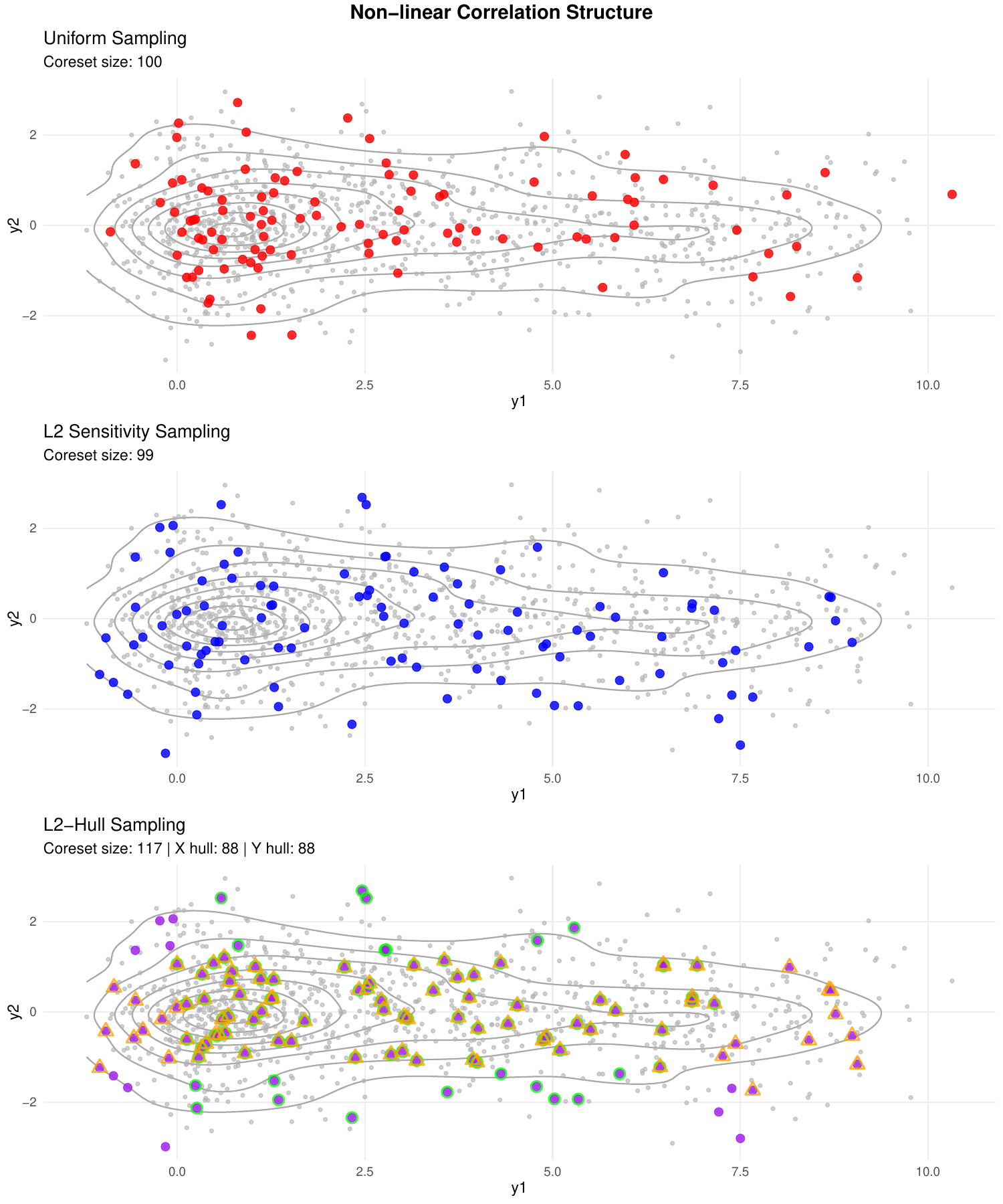}
        \caption{Non-linear Correlation}
    \end{subfigure}
    \hfill
    \begin{subfigure}[b]{0.32\textwidth}
        \includegraphics[width=\textwidth]{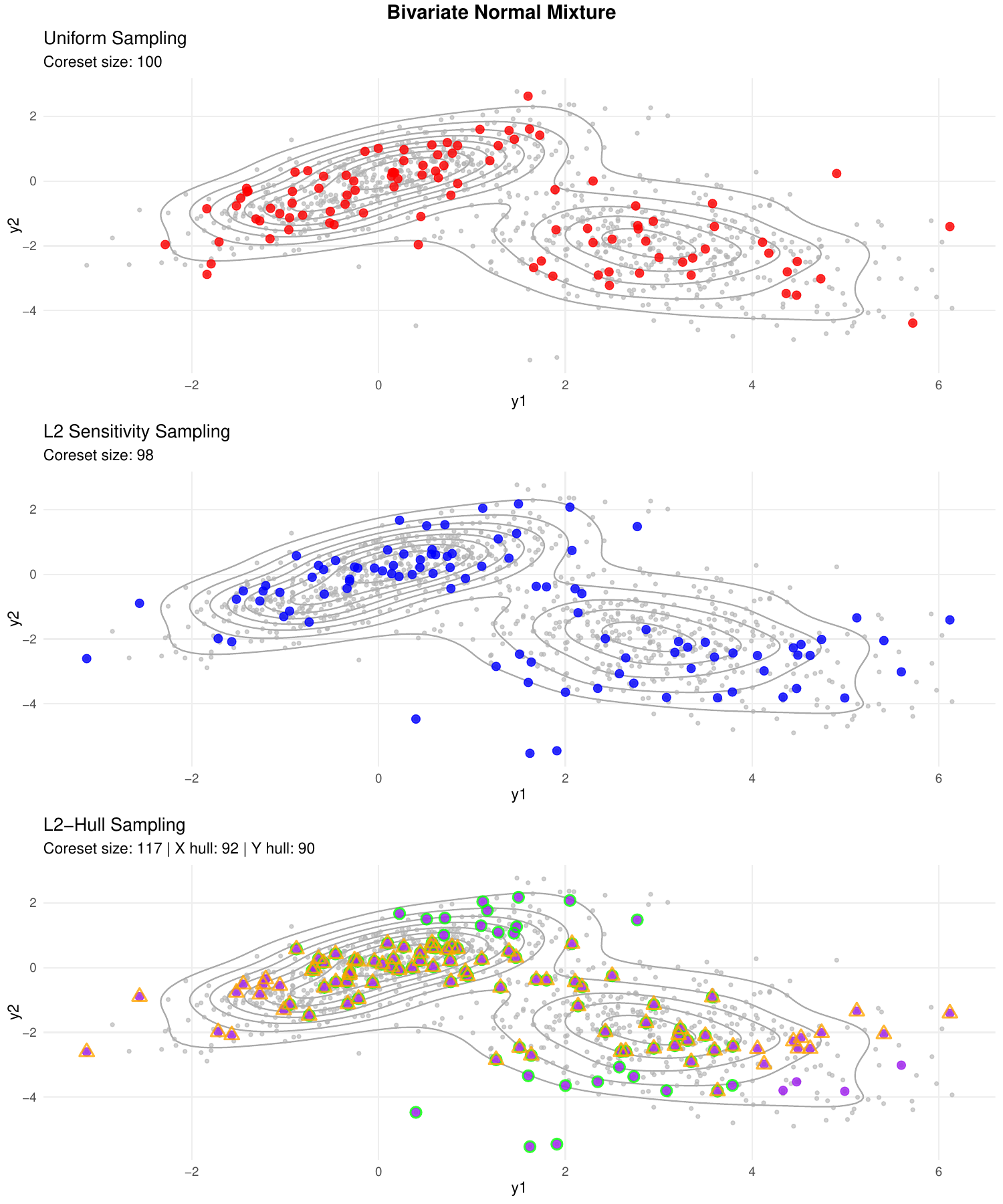}
        \caption{Bivariate Normal Mixture}
    \end{subfigure}
    \caption{Coreset visualization for basic probability distributions. Each row shows a different sampling method: Uniform (top), $\ell_2$ Sensitivity (middle), and $\ell_2$-Hull (bottom).}
    \label{fig:dgp_vis_1}
\end{figure}

\begin{figure}[htbp]
    \centering
    \begin{subfigure}[b]{0.32\textwidth}
        \includegraphics[width=\textwidth]{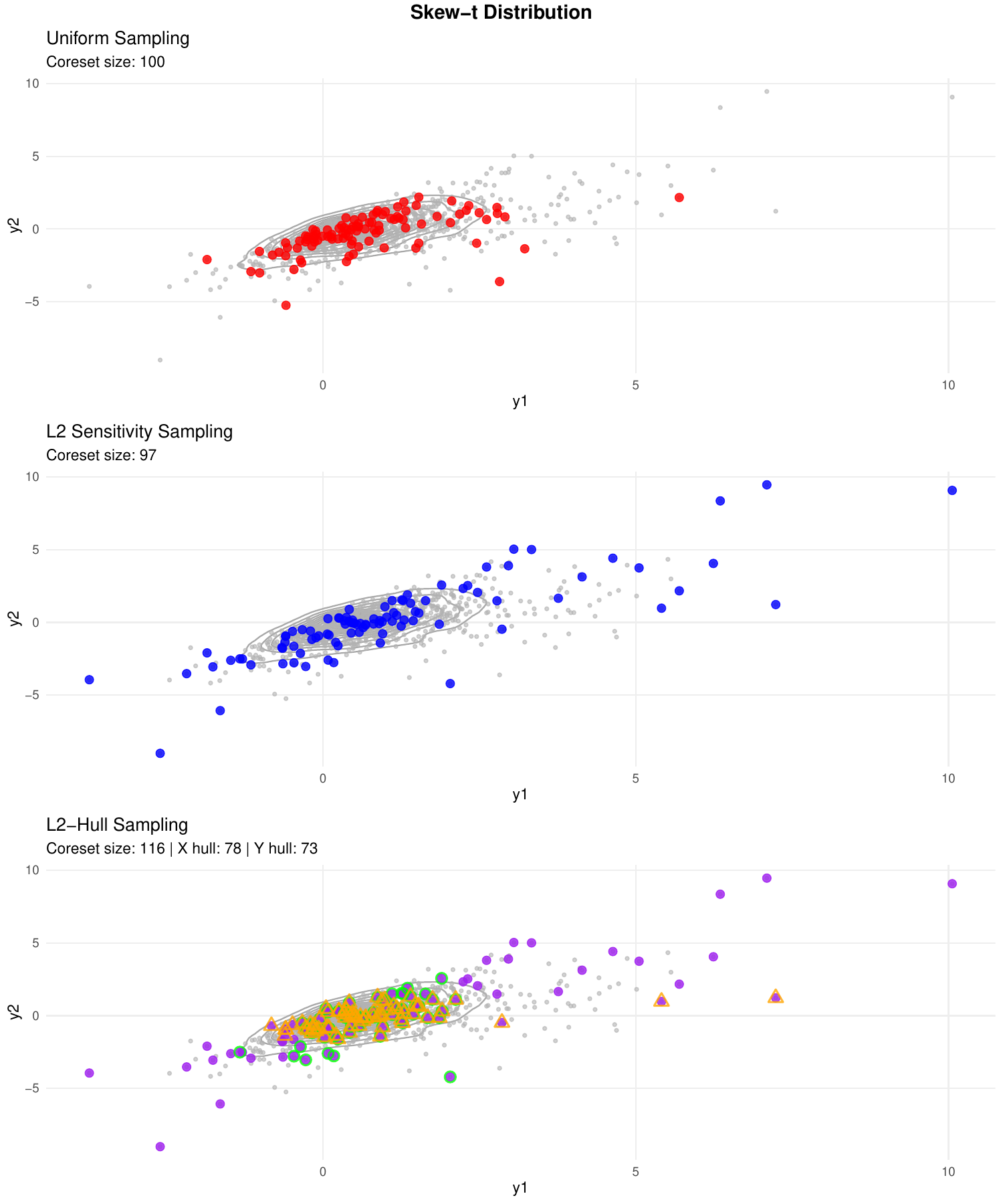}
        \caption{Skew-t Distribution}
    \end{subfigure}
    \hfill
    \begin{subfigure}[b]{0.32\textwidth}
        \includegraphics[width=\textwidth]{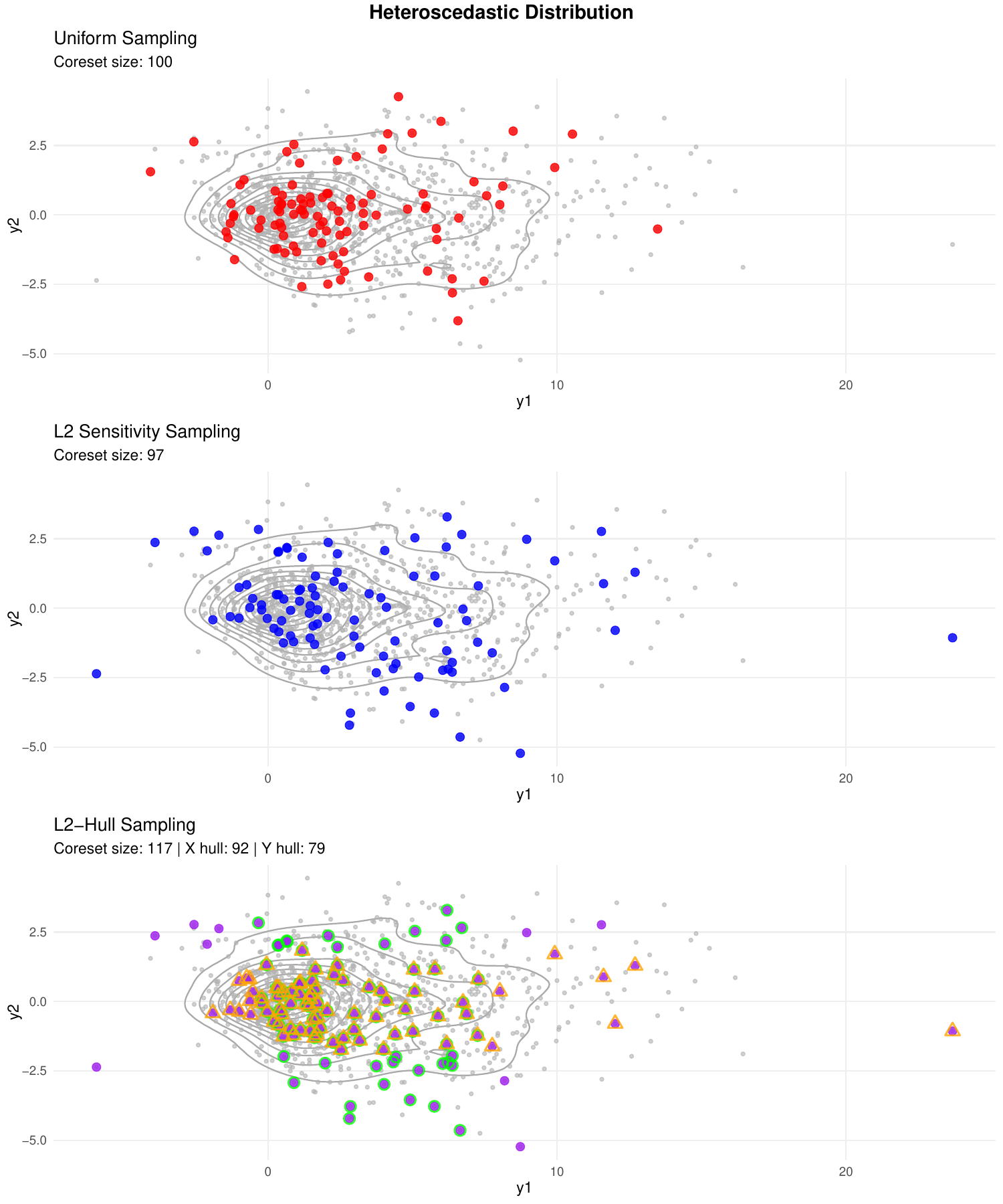}
        \caption{Heteroscedastic Distribution}
    \end{subfigure}
    \hfill
    \begin{subfigure}[b]{0.32\textwidth}
        \includegraphics[width=\textwidth]{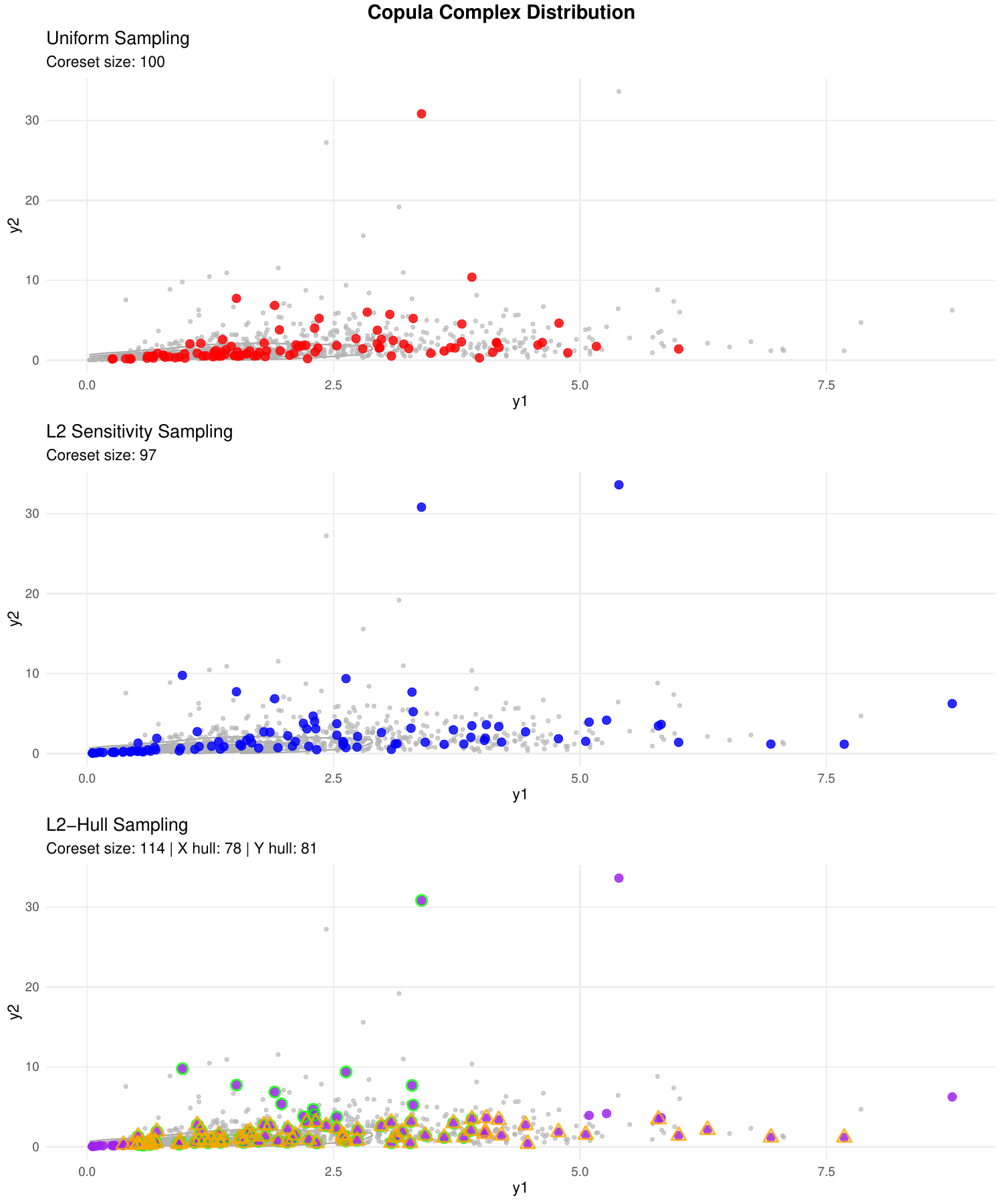}
        \caption{Copula Complex Distribution}
    \end{subfigure}
    \caption{Coreset visualization for complex probability distributions. Each column shows a different sampling method: Uniform (top), $\ell_2$ Sensitivity (middle), and $\ell_2$-Hull (bottom).}
    \label{fig:dgp_vis_2}
\end{figure}

\begin{figure}[htbp]
    \centering
    \begin{subfigure}[b]{0.32\textwidth}
        \includegraphics[width=\textwidth]{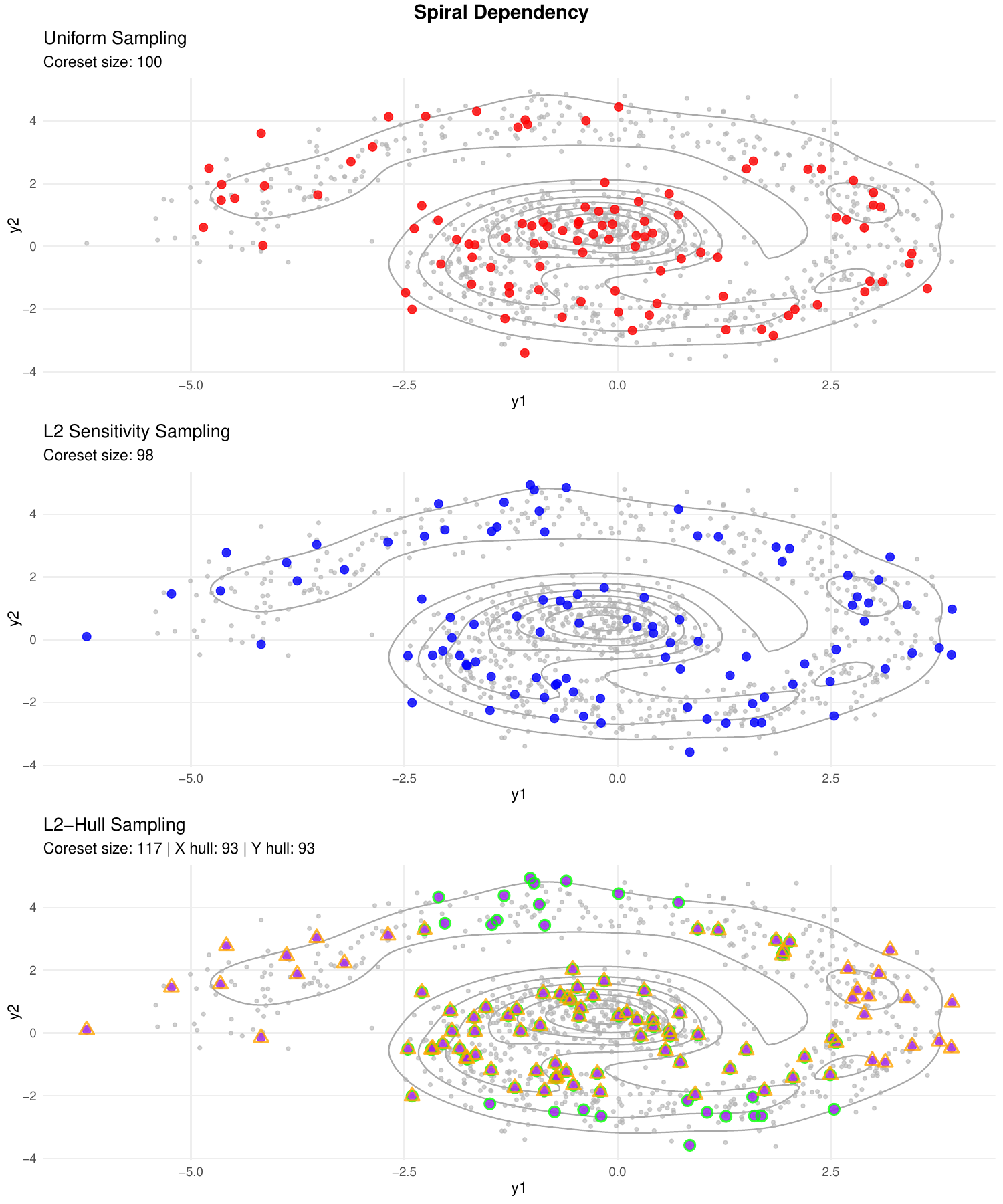}
        \caption{Spiral Dependency}
    \end{subfigure}
    \hfill
    \begin{subfigure}[b]{0.32\textwidth}
        \includegraphics[width=\textwidth]{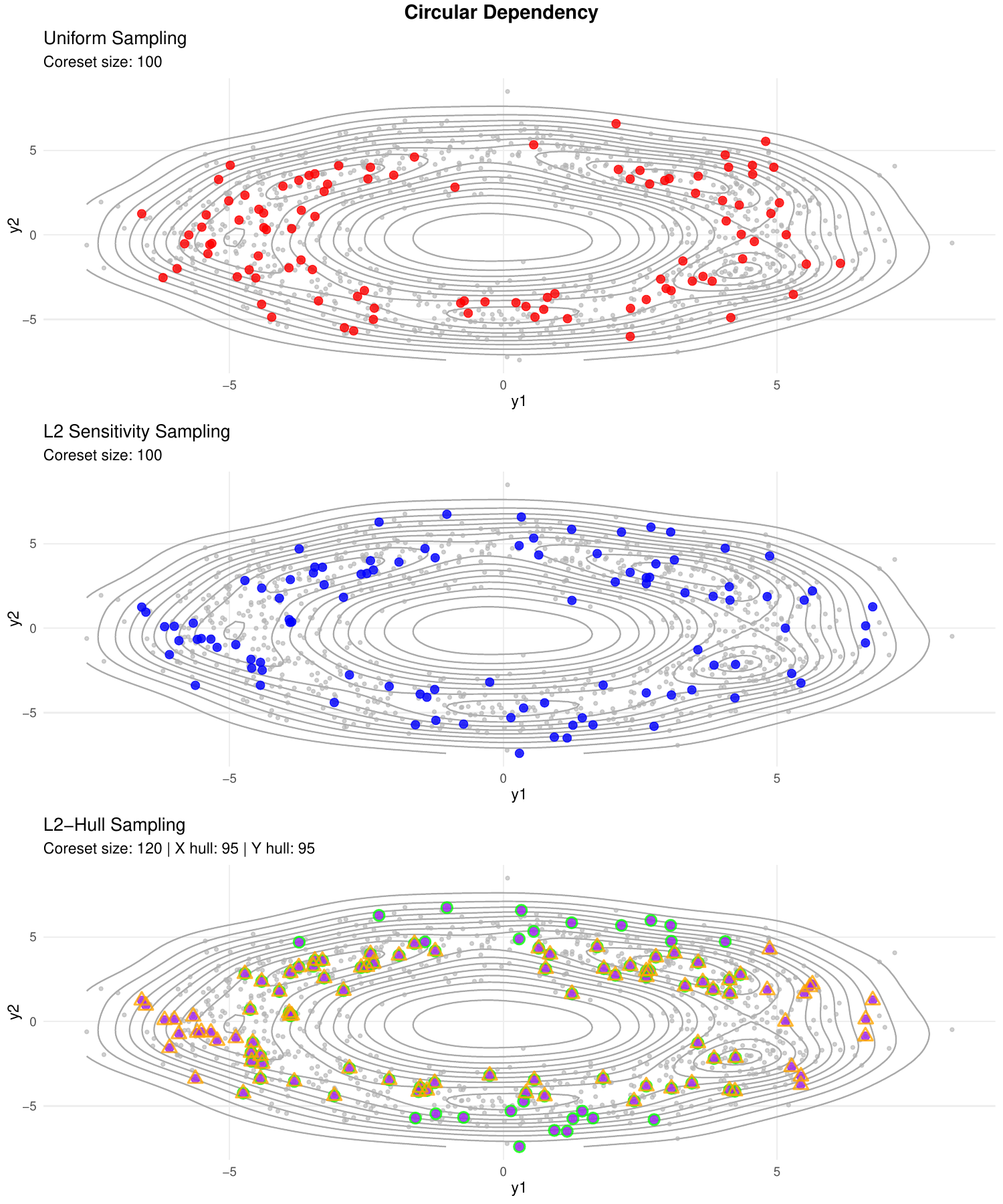}
        \caption{Circular Dependency}
    \end{subfigure}
    \hfill
    \begin{subfigure}[b]{0.32\textwidth}
        \includegraphics[width=\textwidth]{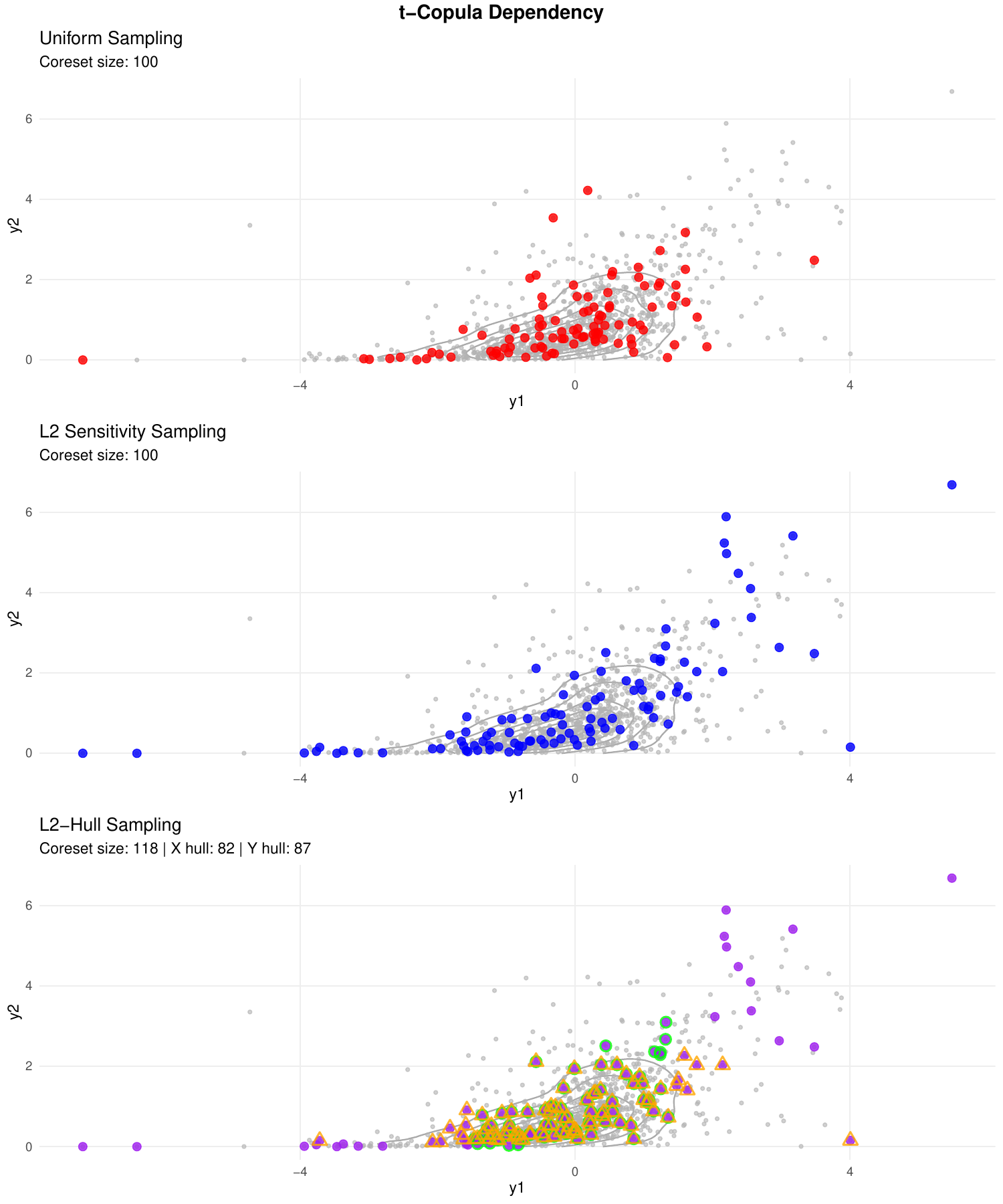}
        \caption{t-Copula Dependency}
    \end{subfigure}
    \caption{Coreset visualization for geometric dependency structures. Each column shows a different sampling method: Uniform (top), $\ell_2$ Sensitivity (middle), and $\ell_2$-Hull (bottom).}
    \label{fig:dgp_vis_3}
\end{figure}

\begin{figure}[htbp]
    \centering
    \begin{subfigure}[b]{0.32\textwidth}
        \includegraphics[width=\textwidth]{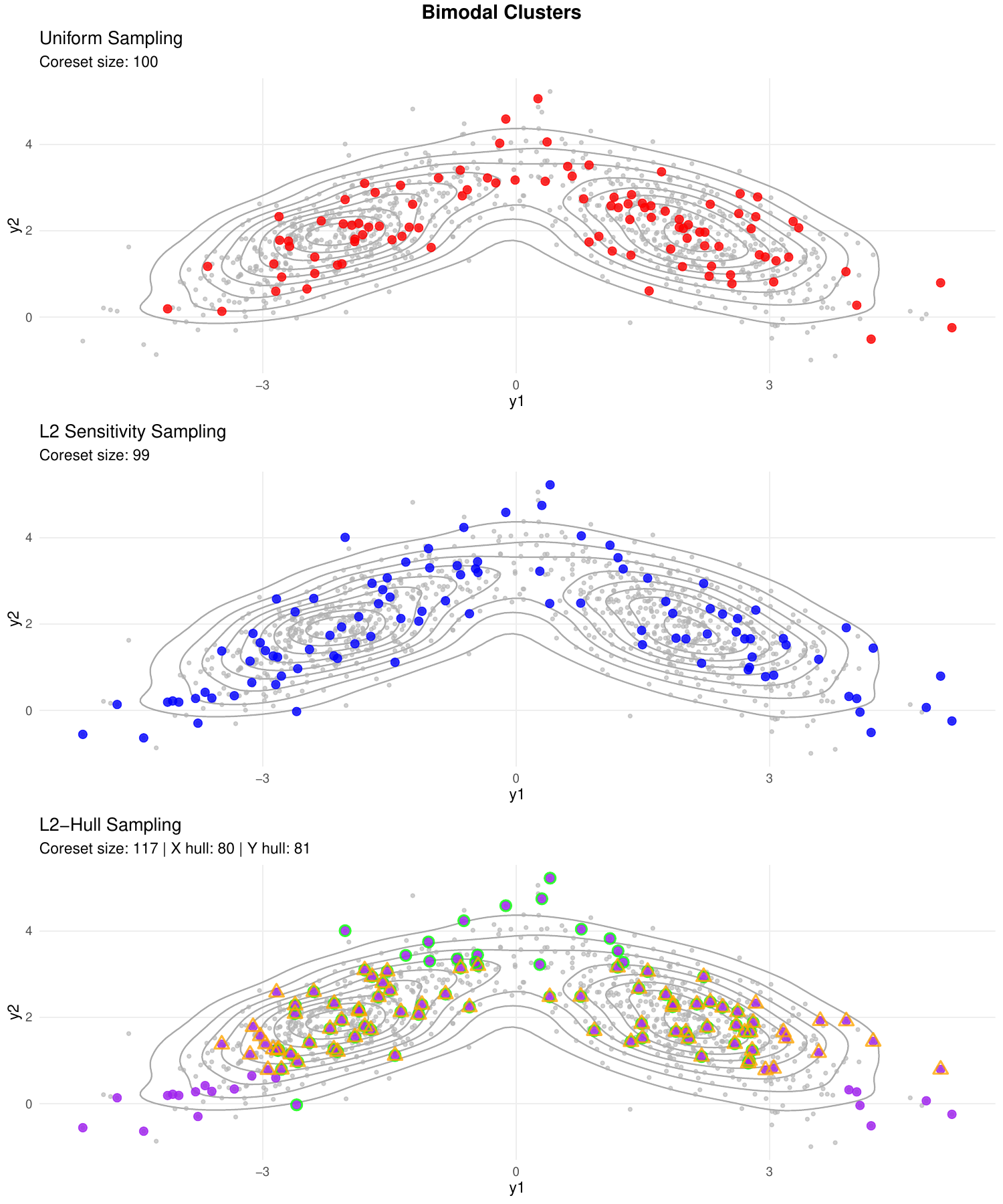}
        \caption{Bimodal Clusters}
    \end{subfigure}
    \hfill
    \begin{subfigure}[b]{0.32\textwidth}
        \includegraphics[width=\textwidth]{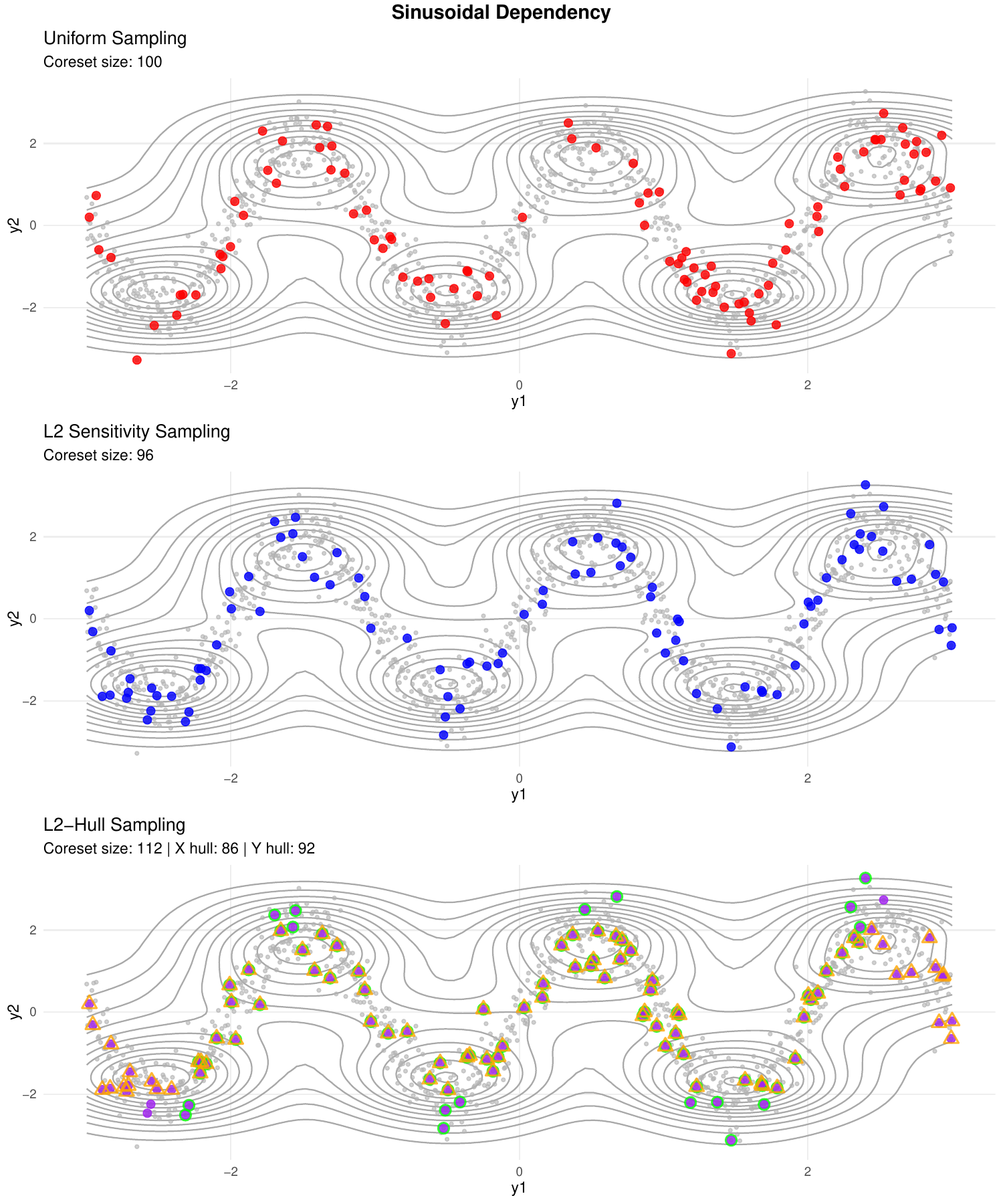}
        \caption{Sinusoidal Dependency}
    \end{subfigure}
    \hfill
    \begin{subfigure}[b]{0.32\textwidth}
        \includegraphics[width=\textwidth]{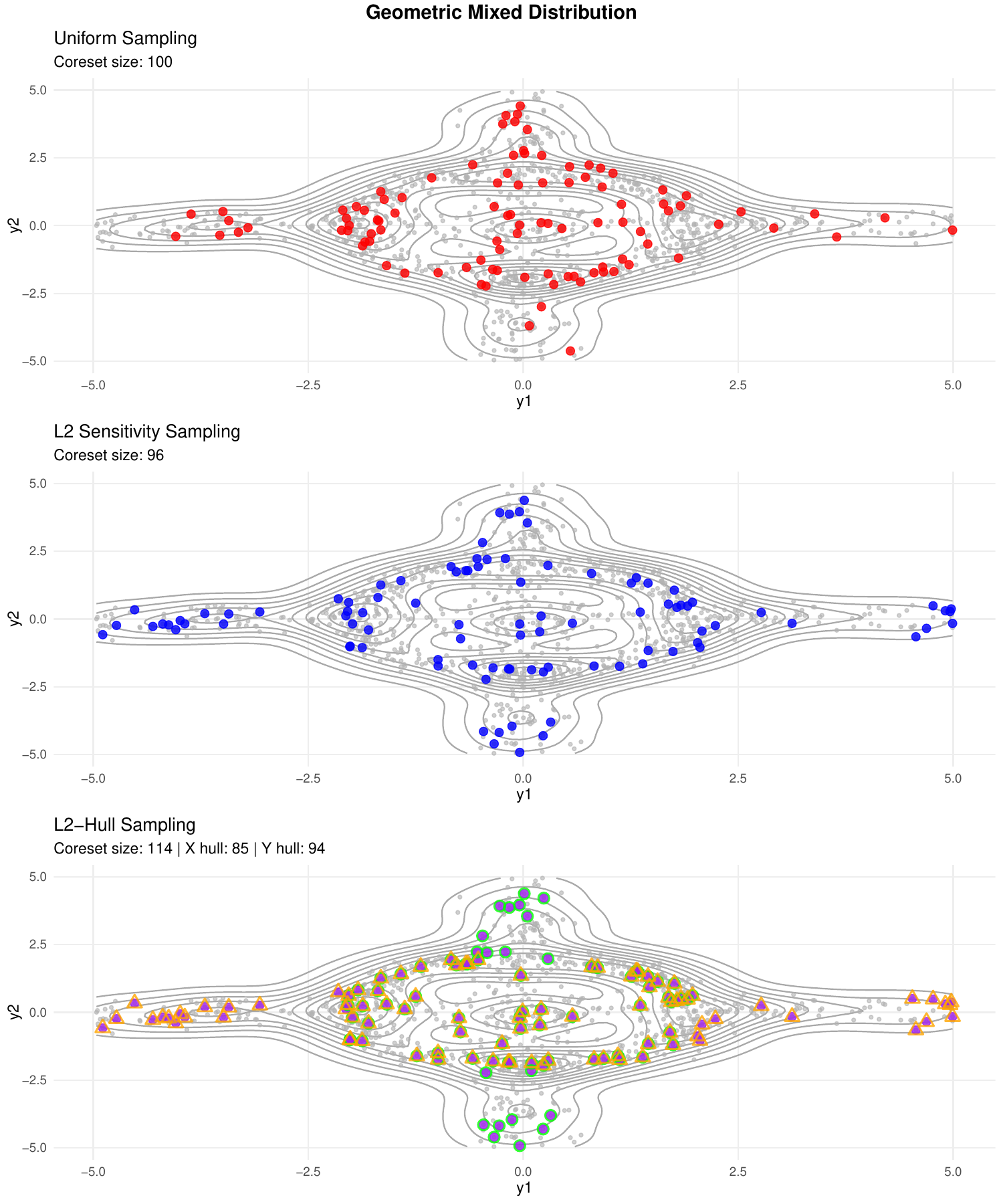}
        \caption{Mixture Distribution}
    \end{subfigure}
    \caption{Coreset visualization for additional dependency structures. Each column shows a different sampling method: Uniform (top), $\ell_2$ Sensitivity (middle), and $\ell_2$-Hull (bottom).}
    \label{fig:dgp_vis_5}
\end{figure}

\begin{figure}[htbp]
    \centering
    \begin{subfigure}[b]{0.45\textwidth}
        \includegraphics[width=\textwidth]{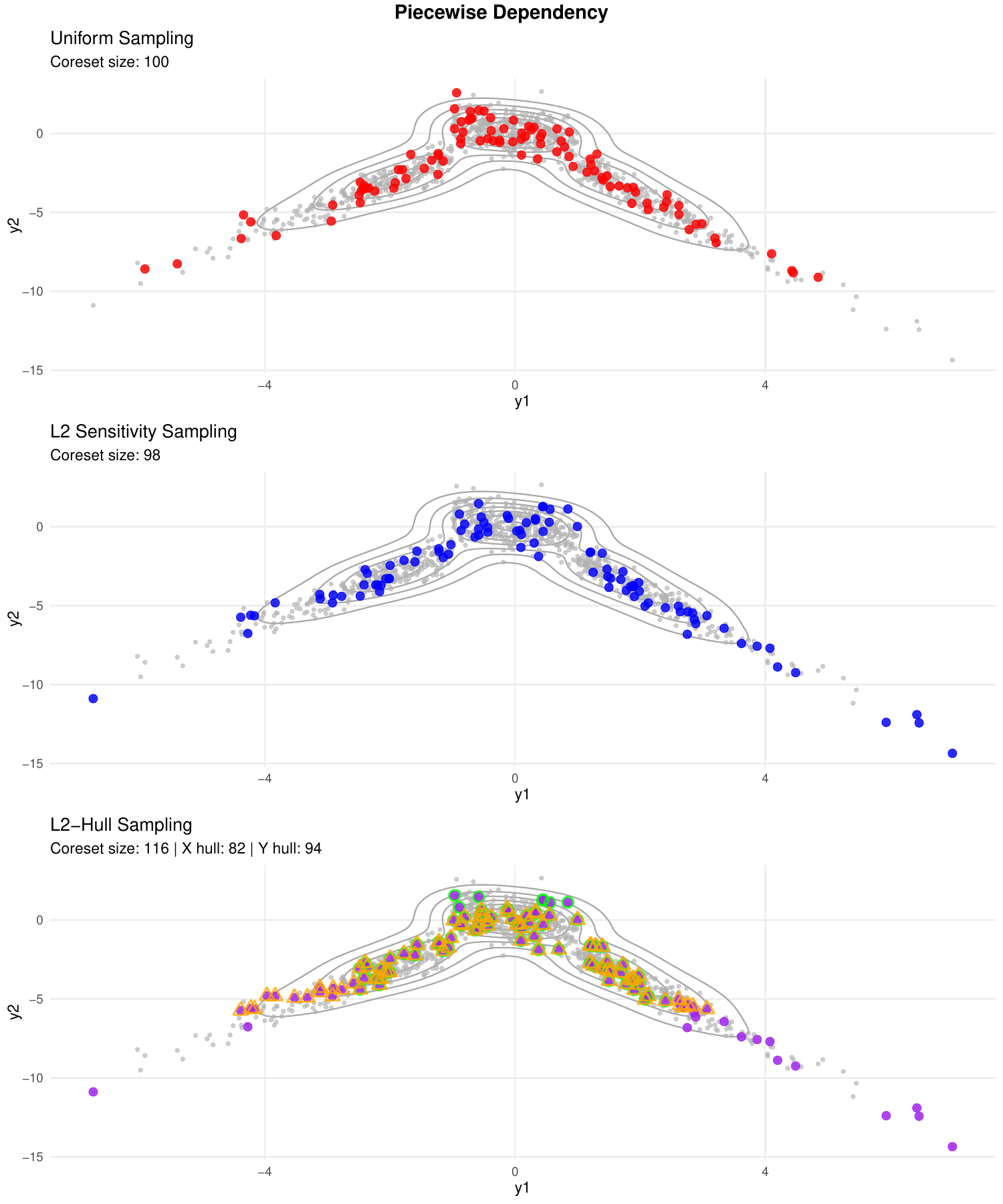}
        \caption{Piecewise Dependency}
    \end{subfigure}
    \hfill
    \begin{subfigure}[b]{0.45\textwidth}
        \includegraphics[width=\textwidth]{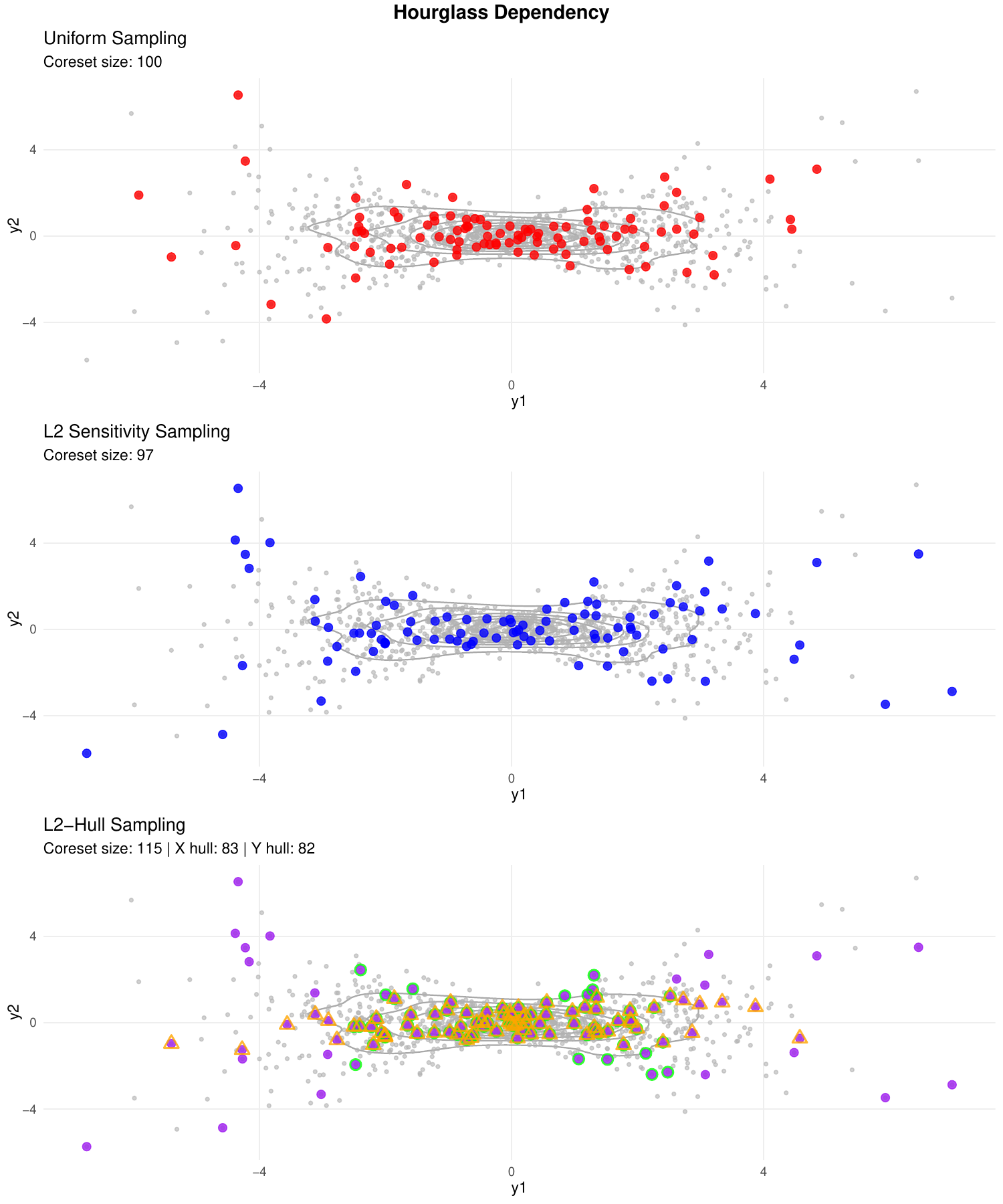}
        \caption{Hourglass Dependency}
    \end{subfigure}
    \caption{Coreset visualization for functional dependency structures. Each column shows a different sampling method: Uniform (top), $\ell_2$ Sensitivity (middle), and $\ell_2$-Hull (bottom).}
    \label{fig:dgp_vis_4}
\end{figure}

To visually demonstrate the advantages of our proposed method, we present coreset visualizations for representative data generation processes. In each figure, we display coresets of approximately $100$ points generated from original samples of $1\,000$ points using three different sampling methods: Uniform Sampling (top), $\ell_2$ Sensitivity Sampling (middle), and our proposed $\ell_2$-Hull Sampling (bottom).

As shown in Figures \ref{fig:dgp_vis_1}--\ref{fig:dgp_vis_5}, our $\ell_2$-Hull method effectively captures the entire shape of the underlying data distribution, particularly in complex scenarios such as the Hourglass Dependency, Piecewise Dependency, Spiral Dependency, Copula Complex Distribution, and Bimodal Clusters. The visualization highlights a key limitation of uniform sampling: it often fails to include points that are distant from the central mass of the distribution but are nevertheless critical for accurately representing the overall data structure.

For instance, in the Bimodal Clusters visualization (Figure \ref{fig:dgp_vis_5}a), our $\ell_2$-Hull method ensures comprehensive coverage of both clusters including their boundaries, while uniform sampling misses several critical points that define the extent of the distribution. Similarly, in the Piecewise Dependency case (Figure \ref{fig:dgp_vis_4}a), the non-linear relationship is better preserved by our method, particularly at the extremes of the distribution.

It is worth noting that in many scenarios, the performances of $\ell_2$ Sensitivity Sampling and our $\ell_2$-Hull method appear similar, especially as the coreset size increases. This is expected, as the primary purpose of adding convex hull points is to safeguard against extreme cases that may arise during optimization. As the sample size grows, the probability of encountering such extreme cases diminishes. Nevertheless, the $\ell_2$-Hull method provides a theoretical guarantee that important boundary points are always included, regardless of the particular dataset instance.

The enhanced coverage provided by our method translates directly to the improved empirical performance metrics observed in \Cref{tab:performance_comparison_30,tab:performance_comparison_100}, particularly for complex dependency structures where capturing the overall shape of the distribution is crucial for accurate statistical inference.

\subsubsection{Simulation Experiments Results}

 In all the above scenarios, we will evaluate under the same sampling/modeling process and can adjust the sample size $n$ as well as the correlation parameters (e.g., $\rho$) as needed.

 \paragraph{Coreset Construction Methods} We compare three sampling strategies: 
 
 \begin{enumerate} 
 
 \item \textbf{Uniform Subsampling}: $k$ points are taken with equal probability from all $n$ samples without replacement, and the weights are set to $\frac{n}{k} $. 
 
 \item \textbf{$\ell_2$ -Only Leverage Score Sampling}: leverage scores (or upper bounds on sensitivity) are calculated for $a(x_i)$ and $a(y_i)$, respectively, and subsamples are combined and reweighted after sampling them with the probability $p_i\propto \ell_2$ leverage score; and the weight is then determined using \emph{merge probability} to compute the final weights and do the normalization. 
 
 \item \textbf{$\ell_2$-Hull Hybrid}: on the basis of the above $\ell_2$ sensitivity sampling, we also perform \emph{convex hull} (or $\varepsilon$-kernel) approximation sampling on its derivative matrix $a'(\cdot)$ by adding an extra batch of points located at the extreme geometric boundaries, thus avoiding the logarithmic term $\log(a'(x_i)^T \vartheta)$ in the likelihood function on which the values are unstable or near zero. The \emph{sensitivity subsample} and \emph{convex hull subsample} are eventually summarized together in a joint index. 
 
 \end{enumerate}

 \paragraph{Main Workflow} For each data generation method, we follow the following flow of experiments: 
 
 \begin{enumerate} 
 
 \item \textbf{Data Generation:} Call the corresponding function (e.g. \texttt{generate\_mixture}) to generate the samples, which can be randomly initialized multiple times. \item \textbf{Full Data Baseline:} Perform MCTM fitting with full data, record its log-likelihood $\ell_{full}$
with the estimated coefficients $\hat{\theta}_{full}$, as a baseline. 

\item \textbf{Coreset Sampling \& Fitting:} Use \emph{Uniform Sampling}, \emph{$\ell_2$-Only}, and \emph{$\ell_2$-Hull}, respectively, for multiple subset sizes from the set $k\in\{\text{min\_size},\dots,\text{max\_size}\}$ to obtain a subsample and compute the weights; subsequently, the MCTM is fitted using only this subsample to obtain the log-likelihood $\ell_{coreset}$ and estimated coefficients $\hat{\theta}_{coreset}$. 

\item \textbf{Evaluation Metrics:} 
\begin{itemize} 
\item \textbf{Likelihood Ratio:} compare $\ell_{coreset}/\ell_{full}$, if it is close to 1; 
\item \textbf{Parameter Error:} calculate the value of 
$\|\hat{\vartheta}_{coreset}-\hat{\vartheta}_{full}\|_2^2$ and convergence trend under different $k$; 
\item \textbf{Lambda Error:} Measure the absolute difference in the dependency parameter 
$|\lambda_{coreset} - \lambda_{full}|$ to evaluate how well the coreset preserves the dependency structure;
\item \textbf{Time Cost:} Record the sampling time consumption and optimization time consumption respectively for evaluating the efficiency of the algorithm. 
\end{itemize} 

\end{enumerate}

\begin{figure}[htbp]
    \centering

    \begin{subfigure}[b]{0.3\textwidth}
        \includegraphics[width=\textwidth]{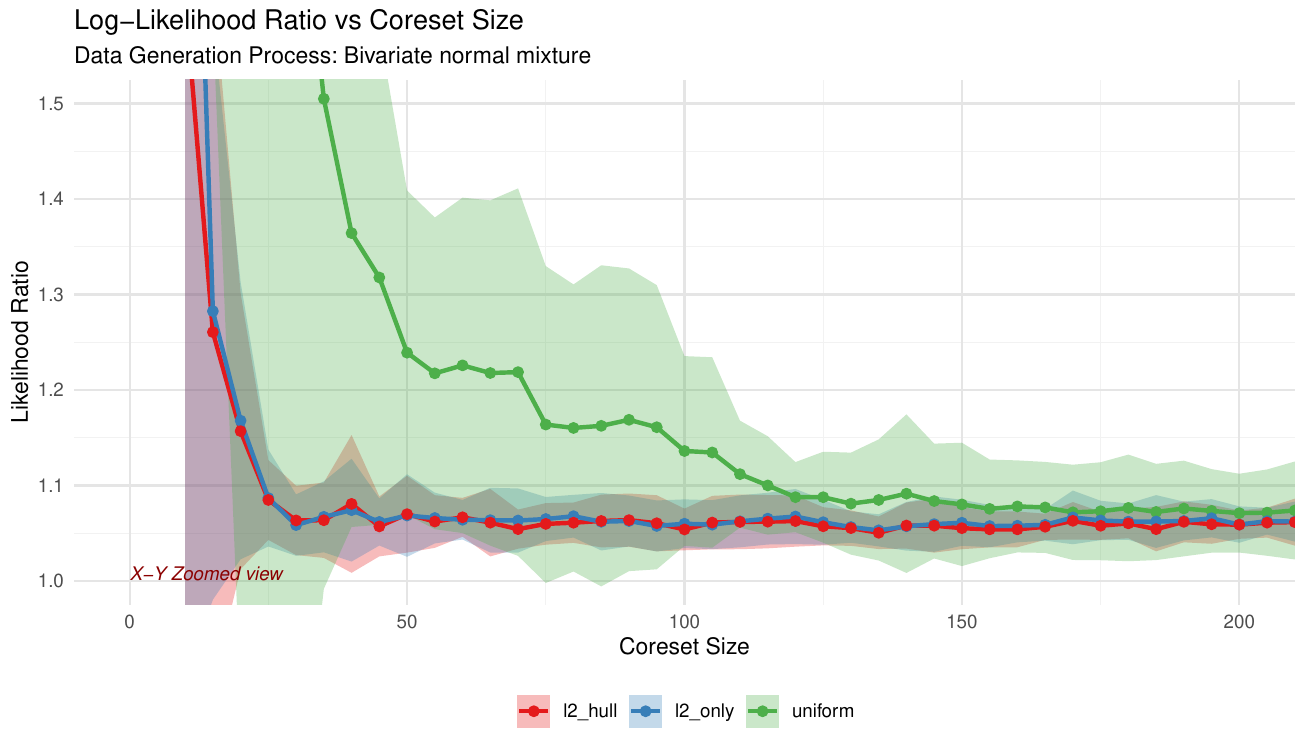}
    
    \end{subfigure}
    \hfill
    \begin{subfigure}[b]{0.3\textwidth}
        \includegraphics[width=\textwidth]{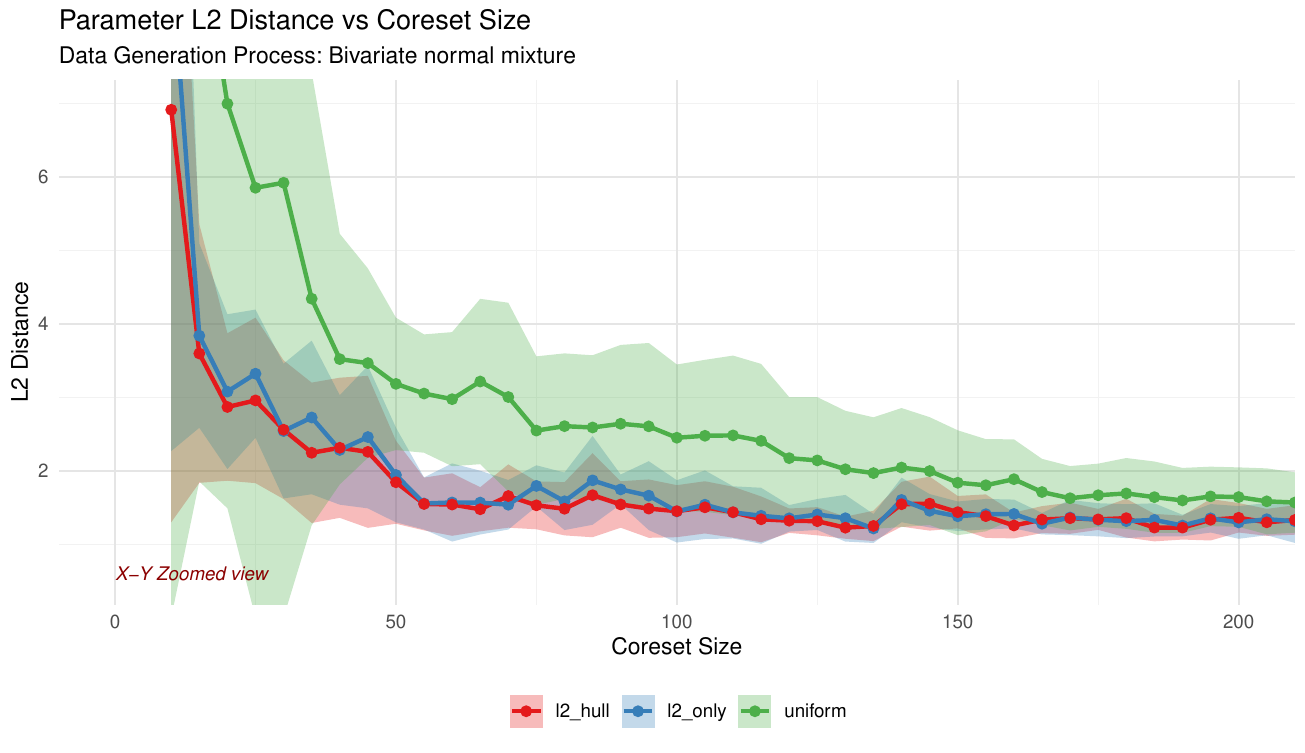}
        
    \end{subfigure}
    \hfill
    \begin{subfigure}[b]{0.3\textwidth}
        \includegraphics[width=\textwidth]{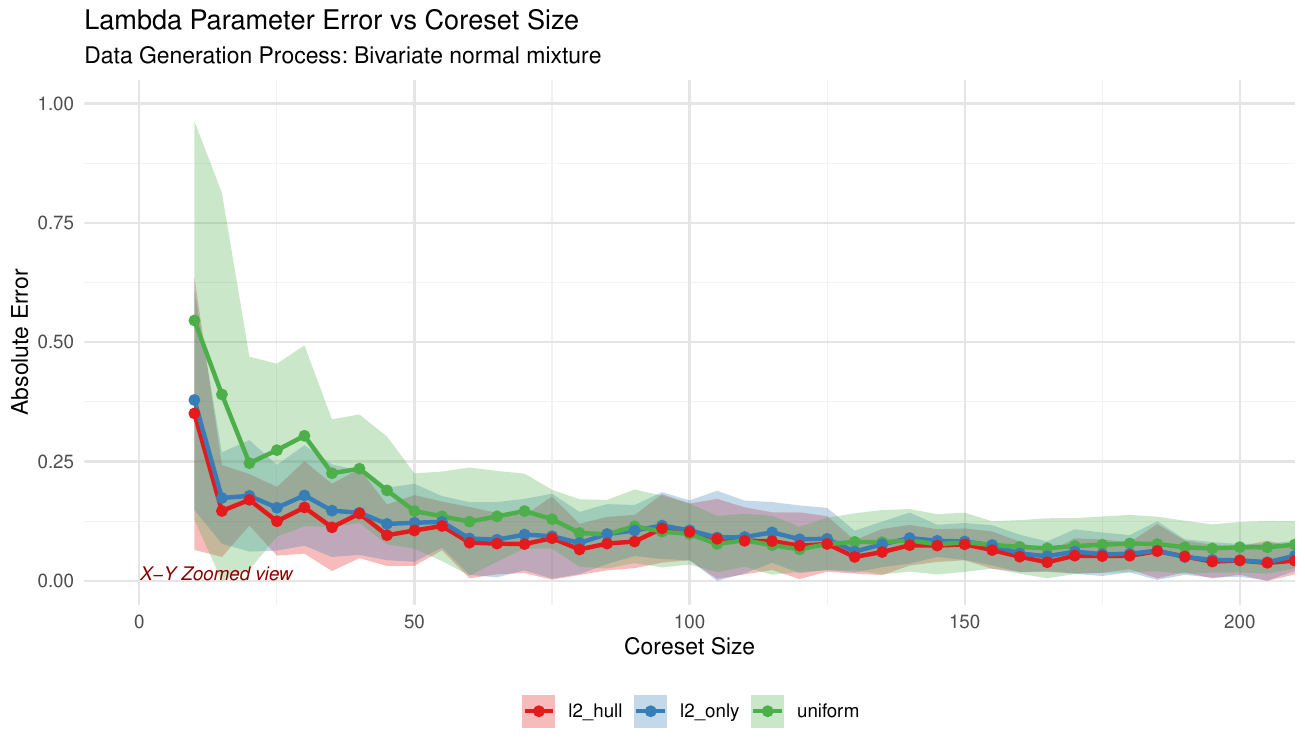}
        
    \end{subfigure}

    \vspace{0.5cm}
    \begin{subfigure}[b]{0.3\textwidth}
        \includegraphics[width=\textwidth]{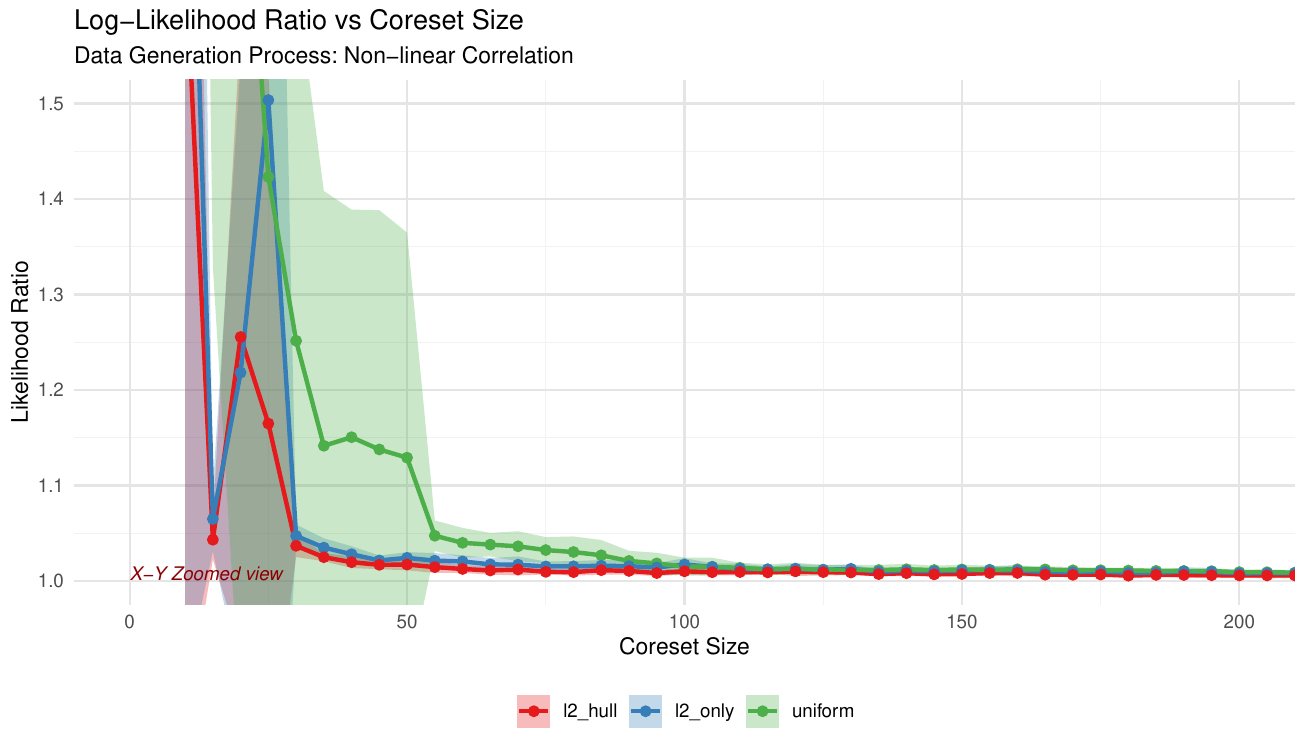}
    
    \end{subfigure}
    \hfill
    \begin{subfigure}[b]{0.3\textwidth}
        \includegraphics[width=\textwidth]{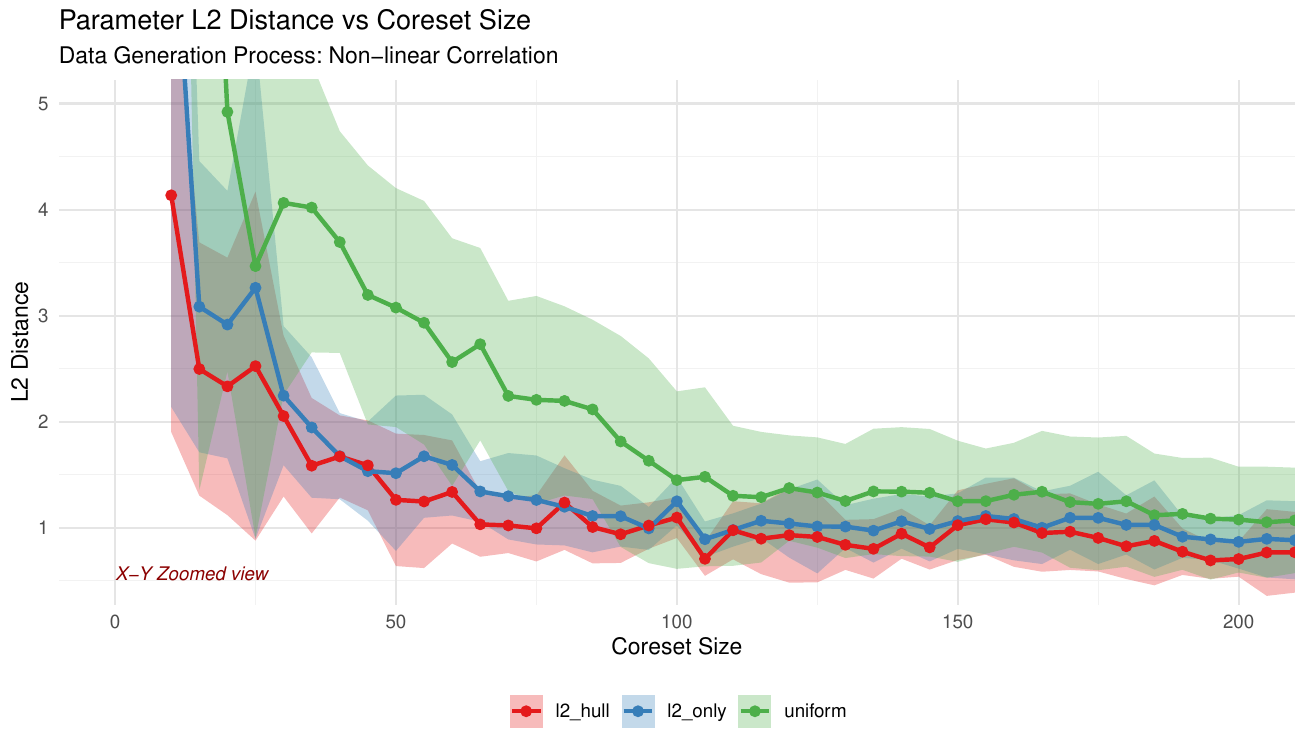}
        
    \end{subfigure}
    \hfill
    \begin{subfigure}[b]{0.3\textwidth}
        \includegraphics[width=\textwidth]{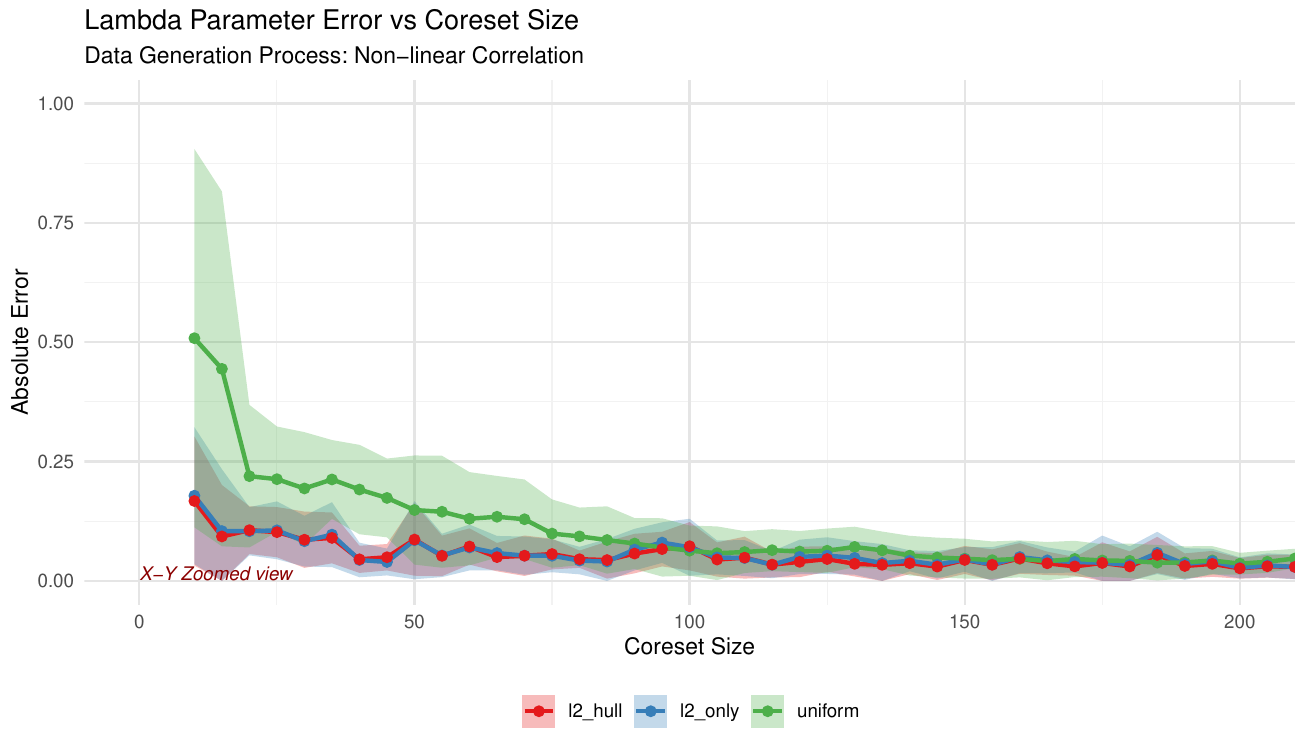}
        
    \end{subfigure}
    
    \vspace{0.5cm}
    \begin{subfigure}[b]{0.3\textwidth}
        \includegraphics[width=\textwidth]{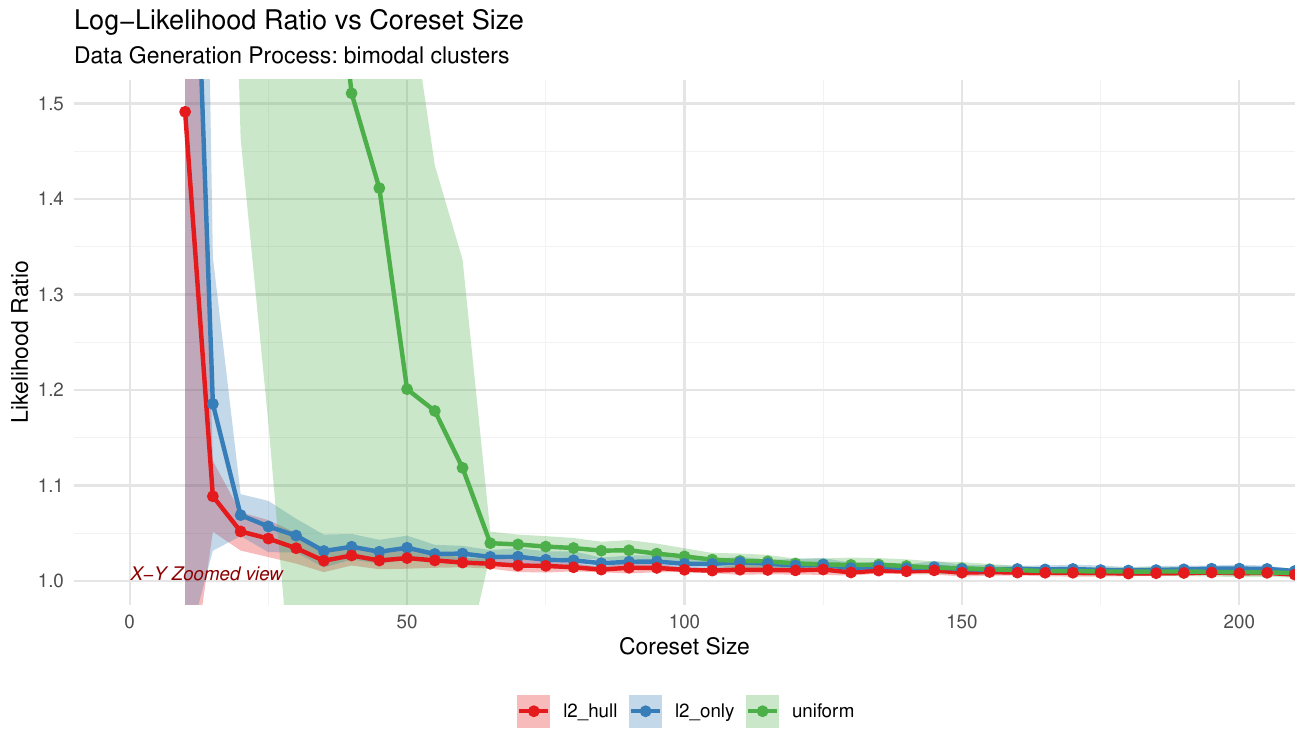}
    \end{subfigure}
    \hfill
    \begin{subfigure}[b]{0.3\textwidth}
        \includegraphics[width=\textwidth]{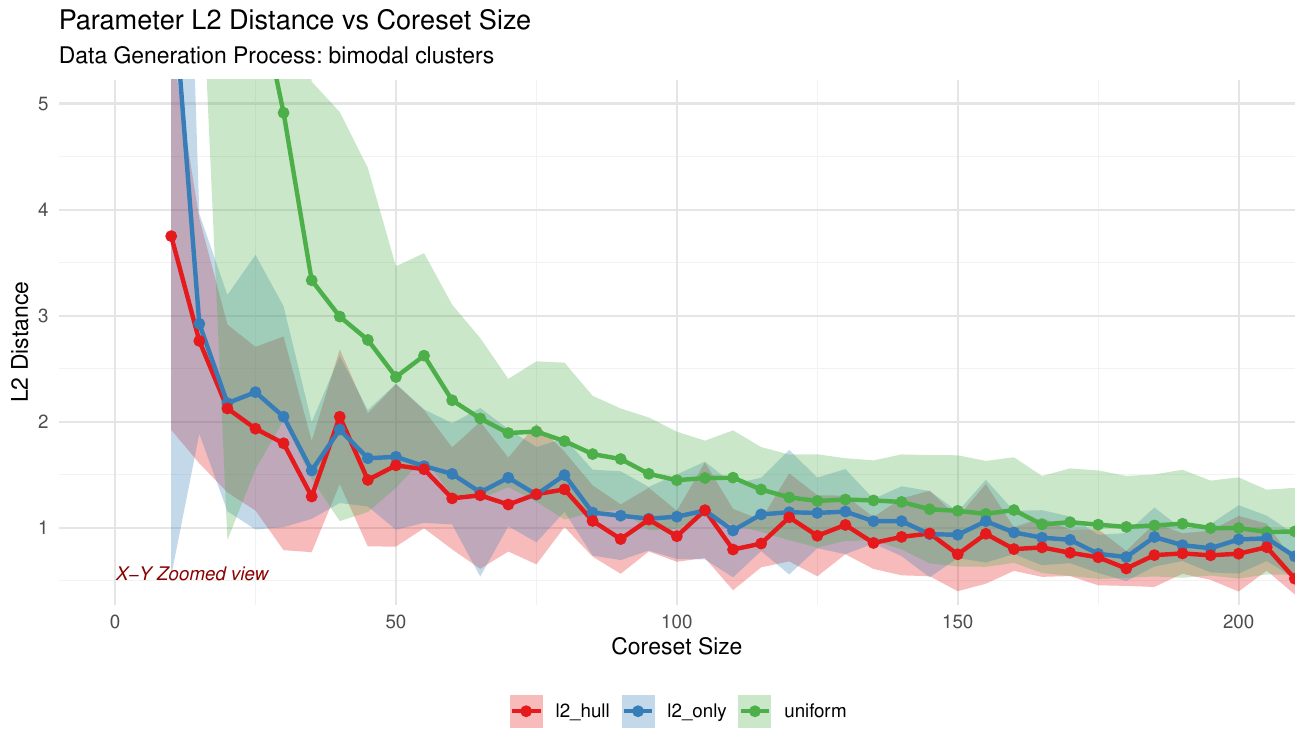}
        
    \end{subfigure}
    \hfill
    \begin{subfigure}[b]{0.3\textwidth}
        \includegraphics[width=\textwidth]{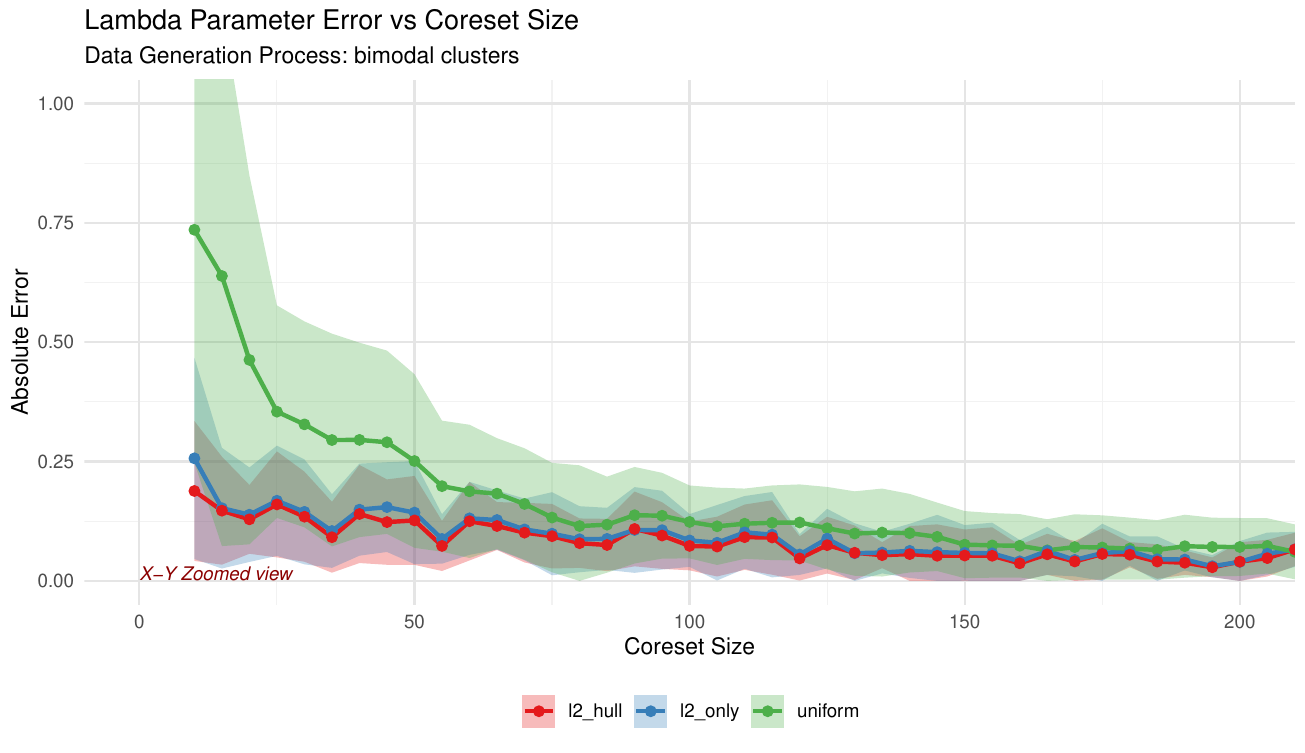}
        
    \end{subfigure}
    
    \caption{Convergence of the likelihood ratio, parameter error, and $\lambda$ error as coreset size increases. First row: Bivariate mixture data of two Gaussian distributions with different means and variances.
    Second row: Non-linear correlation. 
    Third row: Bimodal clusters with two distinct clusters with opposing correlation structure.}
    \label{fig:simulation_results1}
\end{figure}

\begin{figure}[htbp]
    \centering

    \begin{subfigure}[b]{0.3\textwidth}
        \includegraphics[width=\textwidth]{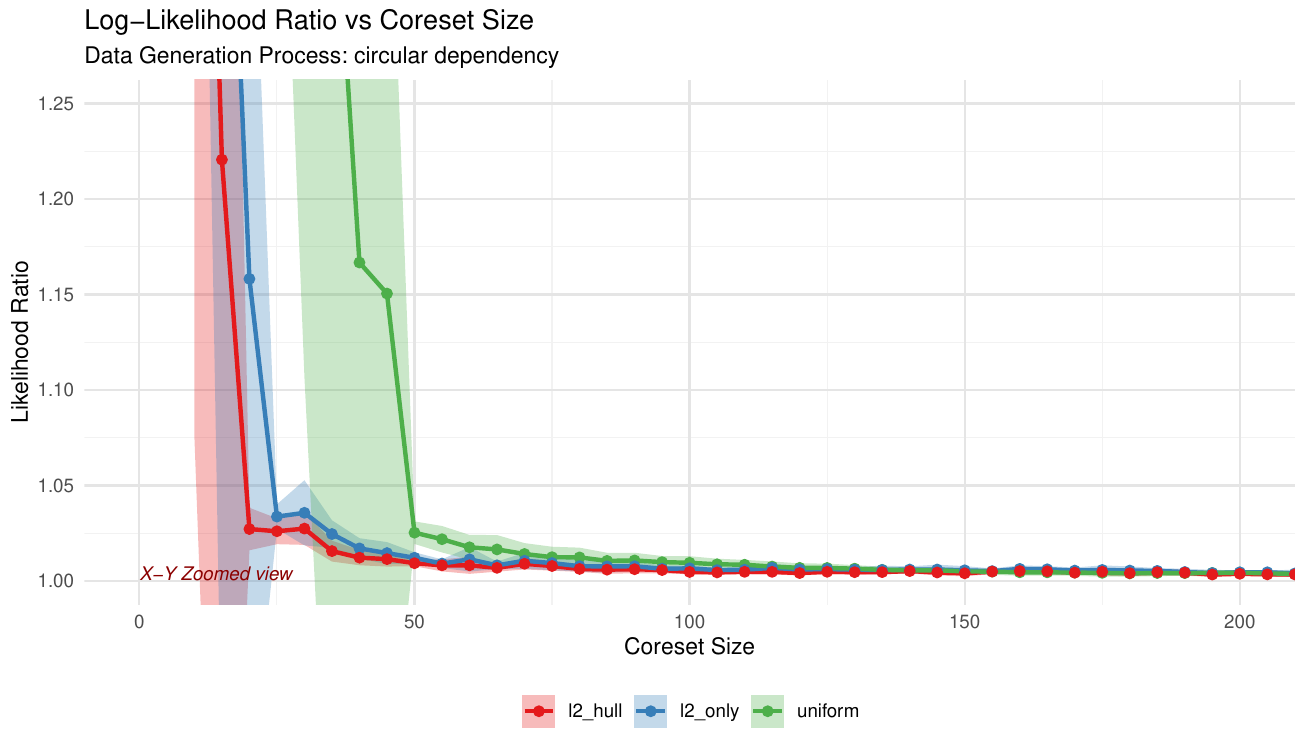}
    
    \end{subfigure}
    \hfill
    \begin{subfigure}[b]{0.3\textwidth}
        \includegraphics[width=\textwidth]{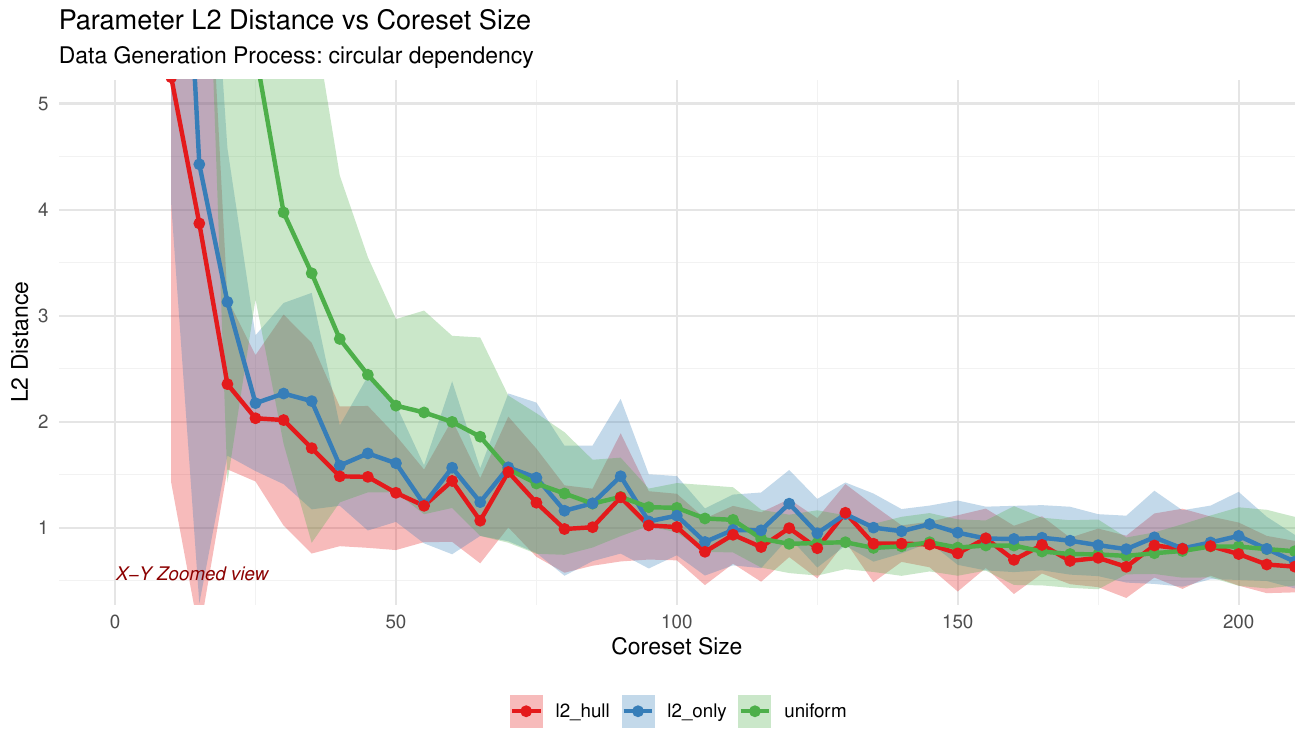}
    
    \end{subfigure}
    \hfill
    \begin{subfigure}[b]{0.3\textwidth}
        \includegraphics[width=\textwidth]{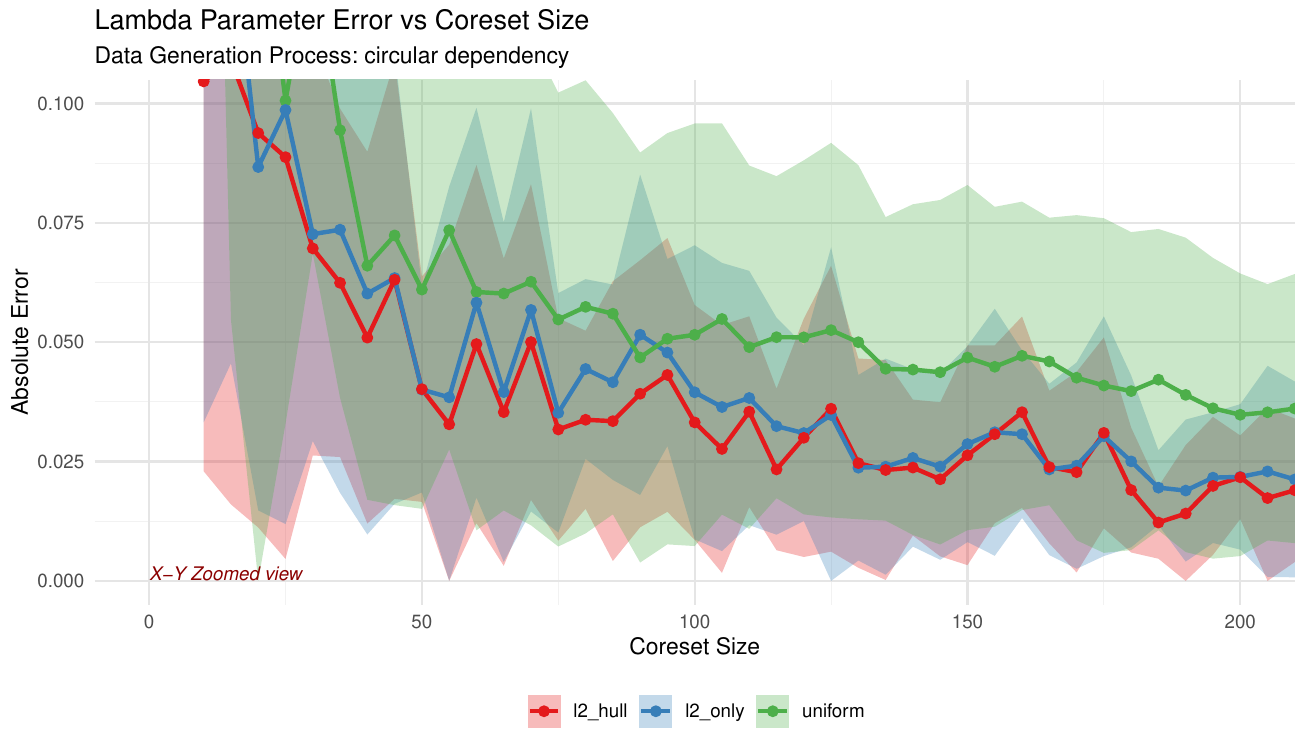}
    
    \end{subfigure}

    \vspace{0.5cm}
    \begin{subfigure}[b]{0.3\textwidth}
        \includegraphics[width=\textwidth]{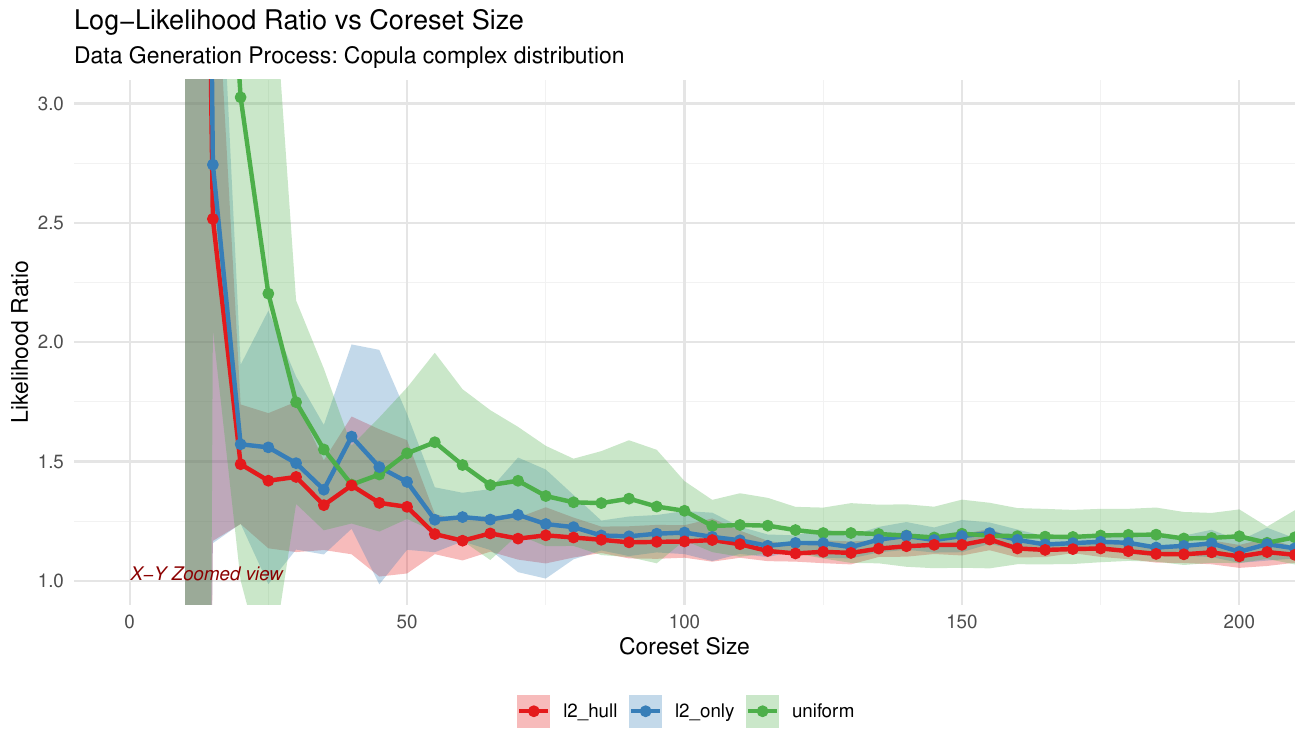}
    
    \end{subfigure}
    \hfill
    \begin{subfigure}[b]{0.3\textwidth}
        \includegraphics[width=\textwidth]{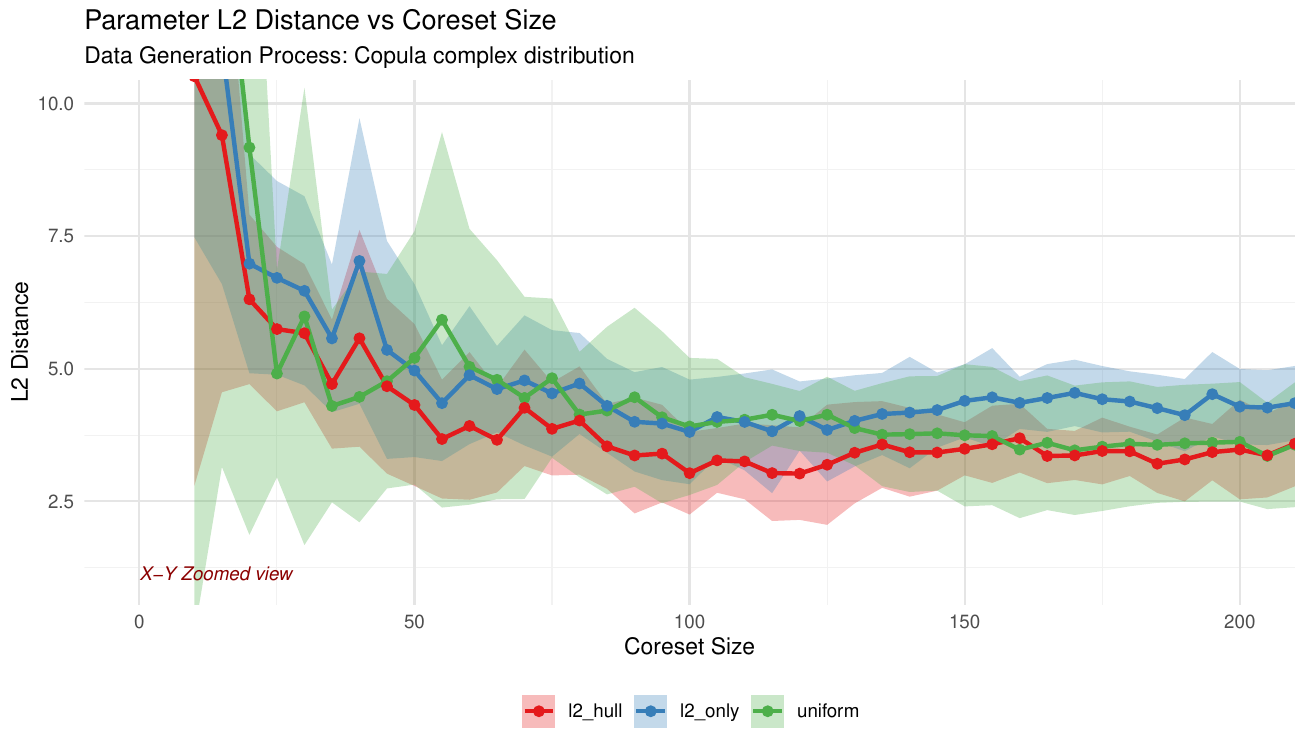}
    
    \end{subfigure}
    \hfill
    \begin{subfigure}[b]{0.3\textwidth}
        \includegraphics[width=\textwidth]{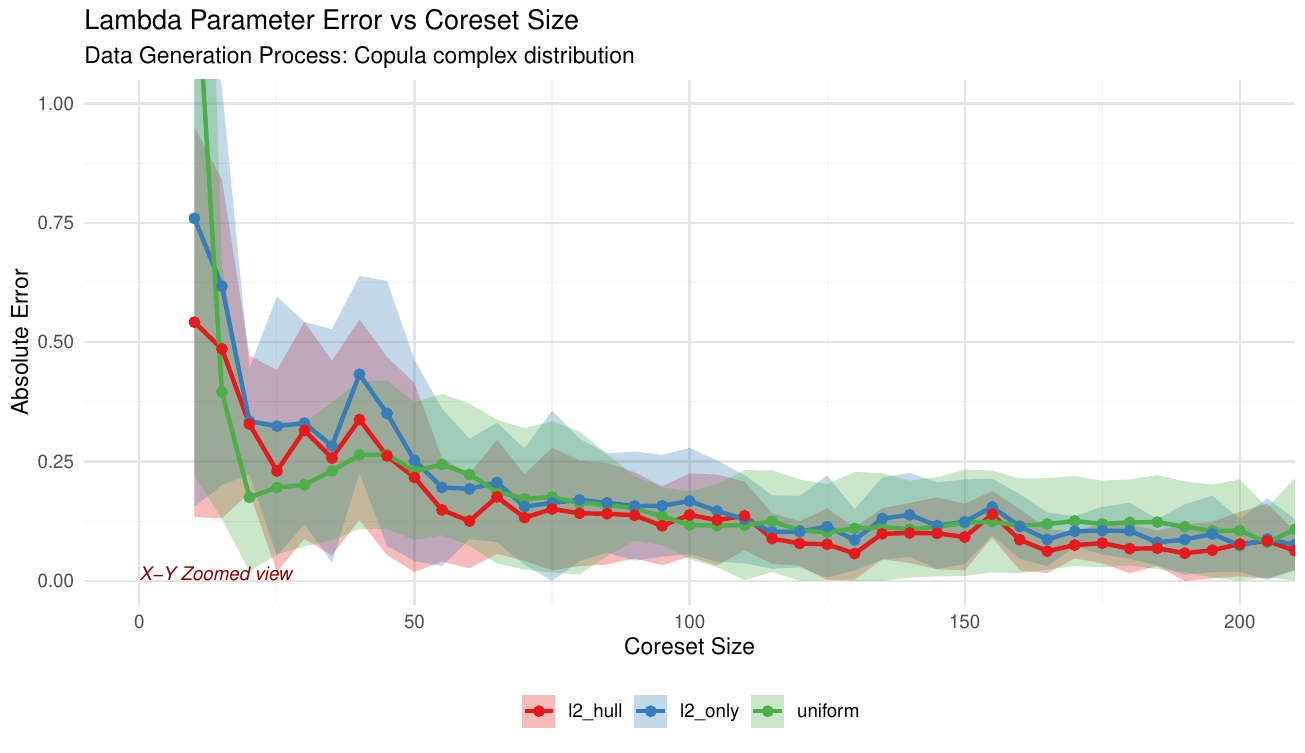}
        
    \end{subfigure}

    \vspace{0.5cm}
    \begin{subfigure}[b]{0.3\textwidth}
        \includegraphics[width=\textwidth]{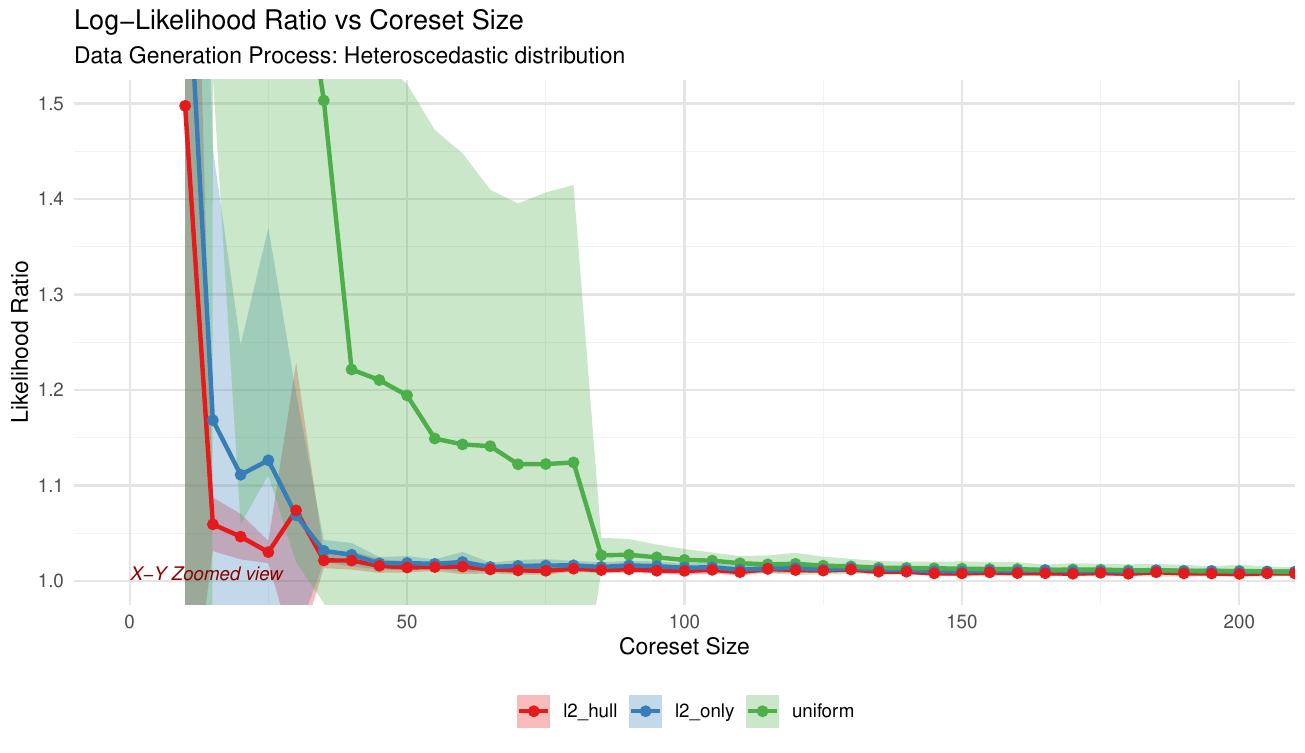}
        
    \end{subfigure}
    \hfill
    \begin{subfigure}[b]{0.3\textwidth}
        \includegraphics[width=\textwidth]{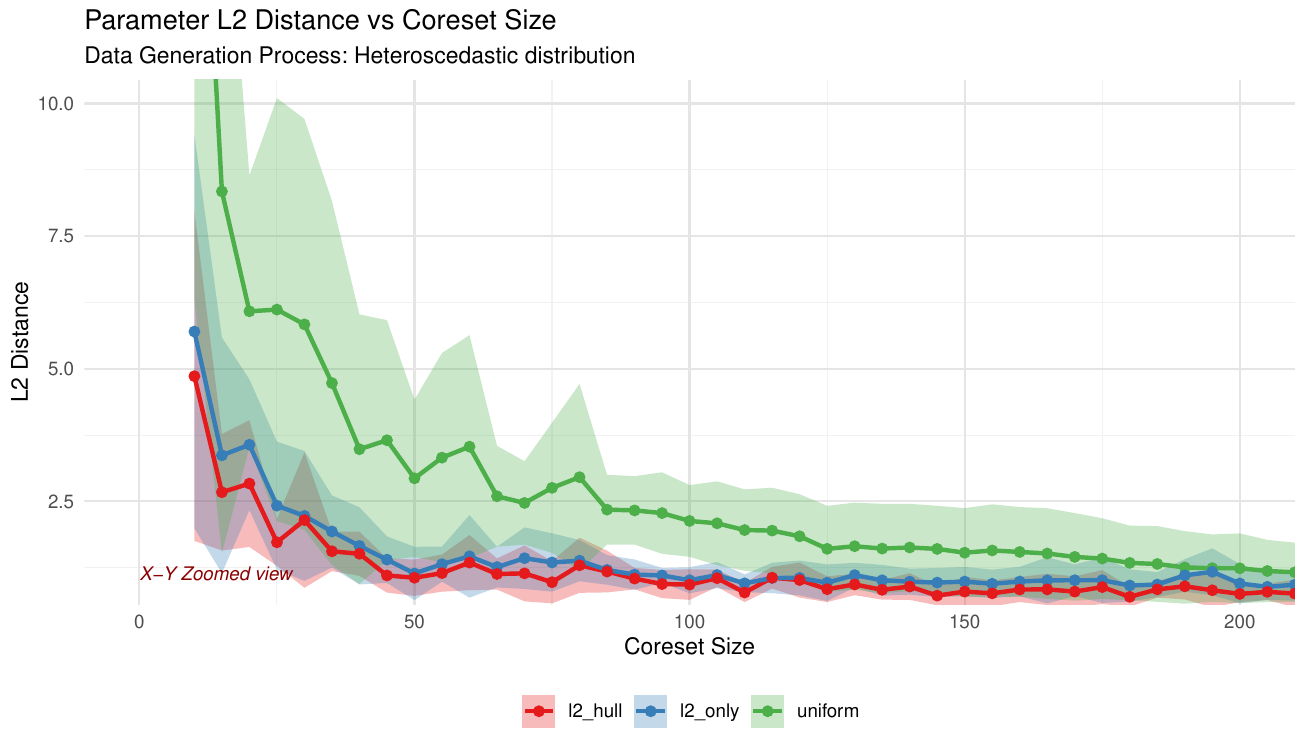}
        
    \end{subfigure}
    \hfill
    \begin{subfigure}[b]{0.3\textwidth}
        \includegraphics[width=\textwidth]{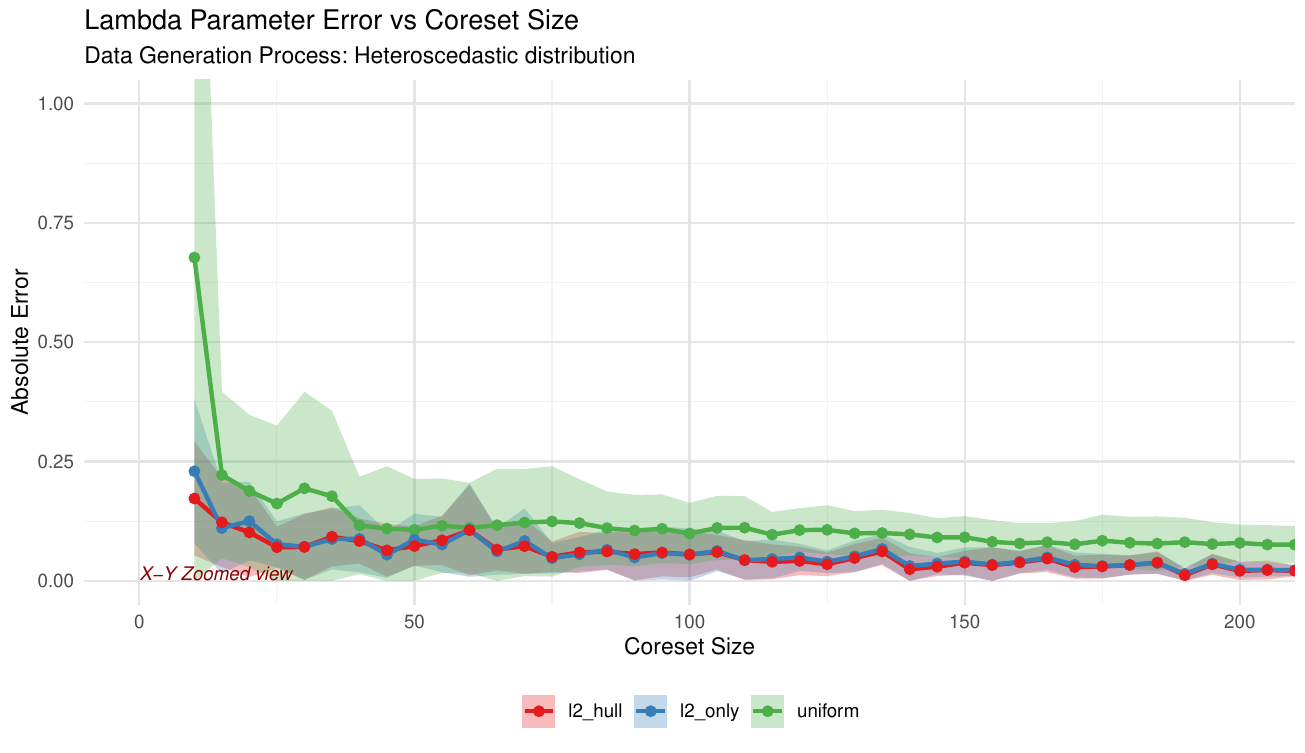}
       
    \end{subfigure}
\caption{Convergence of the likelihood ratio, parameter error, and $\lambda$ error as coreset size increases. 
First row: circular-dependency. Data points are distributed along circular trajectories, demonstrating a circular dependency structure between variables.  Second row: copula complex. A copula construct with tail dependence and non-linear correlation is used. Third row: heteroscedastic distribution. The conditional variance of the data varies with the level of the independent variable, reflecting heteroscedasticity.
}
    \label{fig:simulation_results2}
\end{figure}

\paragraph{Summary} With such multi-scenario simulations, we can systematically examine the proposed MCTM coreset scheme under different distributions. Detailed comparative graphs and numerical analyses showing the combined advantages of the new method in terms of likelihood approximation accuracy, parameter estimation bias, and time efficiency are presented in the experimental results below.

\begin{table}[htbp]
\centering
\caption{Performance comparison of different coreset methods for various data generation processes (coreset Size = 30)}
\label{tab:performance_comparison_30}
\resizebox{\textwidth}{!}{
\begin{tabular}{llccccc}
\toprule

\textbf{Data Gen. Process} & \textbf{Method} & \textbf{Param. $\ell_2$ dist.} & \textbf{$\lambda$ error} & \textbf{Likelihood ratio} & \textbf{Rel. Impr.(\%)} & \textbf{Total time(s)} \\
\midrule
\multirow{3}{*}{Bivariate normal} 
& $\ell_2$-hull & $2.56 \pm 0.76$ & $0.44 \pm 0.16$ & $1.54 \pm 0.29$ & 12.8 & $0.21 \pm 0.03$ \\
& $\ell_2$-only & $\mathbf{2.54 \pm 0.62}$ & $0.51 \pm 0.13$ & $1.65 \pm 0.33$ & 1.6 & $0.20 \pm 0.02$ \\
& uniform & $4.91 \pm 4.74$ & $\mathbf{0.29 \pm 0.26}$ & $1.94 \pm 1.13$ & baseline & $\mathbf{0.13 \pm 0.03}$ \\

\midrule

\multirow{3}{*}{Non-linear correlation}
& $\ell_2$-hull & $\mathbf{1.76 \pm 0.88}$ & $\mathbf{0.09 \pm 0.09}$ & $\mathbf{1.03 \pm 0.01}$ & \textbf{49.8} & $0.19 \pm 0.02$ \\
& $\ell_2$-only & $2.18 \pm 0.82$ & $0.10 \pm 0.10$ & $1.05 \pm 0.02$ & 38.8 & $0.19 \pm 0.01$ \\
& uniform & $3.08 \pm 0.91$ & $0.11 \pm 0.07$ & $1.21 \pm 0.46$ & baseline & $\mathbf{0.12 \pm 0.01}$ \\

\midrule

\multirow{3}{*}{Bivariate normal mixture}
& $\ell_2$-hull & $\mathbf{2.60 \pm 1.00}$ & $\mathbf{0.14 \pm 0.07}$ & $\mathbf{1.07 \pm 0.03}$ & \textbf{49.6} & $0.20 \pm 0.03$ \\
& $\ell_2$-only & $2.76 \pm 0.94$ & $0.16 \pm 0.10$ & $1.08 \pm 0.05$ & 43.7 & $0.21 \pm 0.04$ \\
& uniform & $4.14 \pm 1.80$ & $0.20 \pm 0.15$ & $1.42 \pm 0.34$ & baseline & $\mathbf{0.15 \pm 0.04}$ \\

\midrule

\multirow{3}{*}{Geometric Mixed Distribution}
& $\ell_2$-hull & $\mathbf{2.51 \pm 2.61}$ & $\mathbf{0.06 \pm 0.04}$ & $\mathbf{1.33 \pm 0.53}$ & \textbf{41.4} & $0.20 \pm 0.04$ \\
& $\ell_2$-only & $4.09 \pm 4.35$ & $0.07 \pm 0.03$ & $1.58 \pm 0.58$ & 9.0 & $0.20 \pm 0.02$ \\
& uniform & $4.11 \pm 2.49$ & $0.14 \pm 0.09$ & $1.47 \pm 0.53$ & baseline & $0.13 \pm 0.03$ \\

\midrule

\multirow{3}{*}{Skew-t distribution}
& $\ell_2$-hull & $\mathbf{3.55 \pm 1.46}$ & $0.61 \pm 0.27$ & $\mathbf{1.90 \pm 0.70}$ & 0 & $0.23 \pm 0.03$ \\
& $\ell_2$-only & $3.72 \pm 1.36$ & $0.70 \pm 0.26$ & $2.14 \pm 0.91$ & 0 & $0.23 \pm 0.09$ \\
& uniform & $4.14 \pm 2.04$ & $\mathbf{0.36 \pm 0.31}$ & $2.17 \pm 0.96$ & baseline & $\mathbf{0.15 \pm 0.07}$ \\

\midrule

\multirow{3}{*}{Heteroscedastic distribution}
& $\ell_2$-hull & $\mathbf{2.07 \pm 0.69}$ & $\mathbf{0.08 \pm 0.05}$ & $\mathbf{1.07 \pm 0.12}$ & \textbf{64.8} & $0.20 \pm 0.07$ \\
& $\ell_2$-only & $2.65 \pm 0.88$ & $\mathbf{0.08 \pm 0.05}$ & $1.09 \pm 0.13$ & 58.4 & $0.20 \pm 0.02$ \\
& uniform & $4.47 \pm 3.38$ & $0.16 \pm 0.14$ & $1.63 \pm 1.19$ & baseline & $\mathbf{0.13 \pm 0.02}$ \\

\midrule

\multirow{3}{*}{Copula complex distribution}
& $\ell_2$-hull & $\mathbf{4.56 \pm 1.57}$ & $\mathbf{0.18 \pm 0.13}$ & $\mathbf{1.23 \pm 0.20}$ & \textbf{58.0} & $0.26 \pm 0.04$ \\
& $\ell_2$-only & $5.30 \pm 1.48$ & $0.20 \pm 0.15$ & $1.36 \pm 0.29$ & 51.1 & $0.28 \pm 0.08$ \\
& uniform & $11.05 \pm 7.80$ & $0.27 \pm 0.22$ & $2.35 \pm 0.56$ & baseline & $\mathbf{0.18 \pm 0.06}$ \\

\midrule

\multirow{3}{*}{Spiral dependency}
& $\ell_2$-hull & $\mathbf{1.86 \pm 0.55}$ & $\mathbf{0.04 \pm 0.03}$ & $\mathbf{1.04 \pm 0.01}$ & \textbf{79.5} & $0.21 \pm 0.03$ \\
& $\ell_2$-only & $2.38 \pm 0.80$ & $0.06 \pm 0.05$ & $1.07 \pm 0.02$ & 71.9 & $0.22 \pm 0.07$ \\
& uniform & $6.16 \pm 3.50$ & $0.15 \pm 0.09$ & $2.05 \pm 0.60$ & baseline & $\mathbf{0.21 \pm 0.22}$ \\

\midrule

\multirow{3}{*}{Circular dependency}
& $\ell_2$-hull & $\mathbf{1.51 \pm 0.39}$ & $\mathbf{0.06 \pm 0.04}$ & $\mathbf{1.02 \pm 0.01}$ & \textbf{59.3} & $0.20 \pm 0.02$ \\
& $\ell_2$-only & $1.95 \pm 0.62$ & $\mathbf{0.06 \pm 0.04}$ & $1.12 \pm 0.29$ & 46.3 & $0.21 \pm 0.03$ \\
& uniform & $4.11 \pm 3.24$ & $0.08 \pm 0.08$ & $1.33 \pm 0.73$ & baseline & $\mathbf{0.14 \pm 0.03}$ \\

\midrule

\multirow{3}{*}{t Copula}
& $\ell_2$-hull & $\mathbf{5.77 \pm 1.61}$ & $0.51 \pm 0.30$ & $\mathbf{1.87 \pm 0.62}$ & 0 & $0.26 \pm 0.03$ \\
& $\ell_2$-only & $7.18 \pm 1.63$ & $0.60 \pm 0.29$ & $1.98 \pm 0.63$ & 0 & $0.24 \pm 0.03$ \\
& uniform & $6.58 \pm 3.35$ & $\mathbf{0.23 \pm 0.22}$ & $2.56 \pm 0.92$ & baseline & $\mathbf{0.23 \pm 0.25}$ \\

\midrule

\multirow{3}{*}{Piecewise dependency}
& $\ell_2$-hull & $\mathbf{2.20 \pm 0.99}$ & $\mathbf{0.30 \pm 0.13}$ & $\mathbf{1.11 \pm 0.19}$ & \textbf{58.5} & $0.21 \pm 0.04$ \\
& $\ell_2$-only & $2.38 \pm 0.83$ & $0.35 \pm 0.15$ & $\mathbf{1.11 \pm 0.15}$ & 53.4 & $0.19 \pm 0.02$ \\
& uniform & $5.48 \pm 4.11$ & $0.46 \pm 0.41$ & $1.53 \pm 0.74$ & baseline & $\mathbf{0.13 \pm 0.03}$ \\

\midrule

\multirow{3}{*}{Hourglass dependency}
& $\ell_2$-hull & $1.99 \pm 1.19$ & $\mathbf{0.14 \pm 0.07}$ & $\mathbf{1.04 \pm 0.02}$ & \textbf{55.8} & $0.19 \pm 0.03$ \\
& $\ell_2$-only & $\mathbf{1.85 \pm 1.04}$ & $0.15 \pm 0.08$ & $1.06 \pm 0.03$ & 51.7 & $0.18 \pm 0.01$ \\
& uniform & $5.00 \pm 3.21$ & $0.16 \pm 0.15$ & $1.65 \pm 0.69$ & baseline & $\mathbf{0.12 \pm 0.01}$ \\

\midrule

\multirow{3}{*}{Bimodal clusters}
& $\ell_2$-hull & $\mathbf{2.14 \pm 0.47}$ & $\mathbf{0.15 \pm 0.09}$ & $\mathbf{1.04 \pm 0.01}$ & \textbf{58.5} & $0.19 \pm 0.02$ \\
& $\ell_2$-only & $2.61 \pm 0.66$ & $0.17 \pm 0.09$ & $1.06 \pm 0.02$ & 51.7 & $0.18 \pm 0.02$ \\
& uniform & $4.43 \pm 3.05$ & $0.22 \pm 0.16$ & $1.61 \pm 0.74$ & baseline & $\mathbf{0.14 \pm 0.04}$ \\

\midrule

\multirow{3}{*}{Sinusoidal dependency}
& $\ell_2$-hull & $\mathbf{1.42 \pm 0.48}$ & $\mathbf{0.05 \pm 0.05}$ & $\mathbf{1.02 \pm 0.01}$ & \textbf{38.6} & $0.18 \pm 0.03$ \\
& $\ell_2$-only & $1.66 \pm 0.50$ & $0.09 \pm 0.08$ & $1.03 \pm 0.01$ & 16.8 & $0.19 \pm 0.02$ \\
& uniform & $1.91 \pm 0.74$ & $0.10 \pm 0.06$ & $1.04 \pm 0.01$ & baseline & $\mathbf{0.12 \pm 0.01}$ \\

\bottomrule
\end{tabular}
}
\begin{tablenotes}
\small
\item Note: The results in the table are the mean $\pm 1$ standard deviation of independent simulations. The Relative Improvement is calculated as the average percentage improvement across all metrics, where improvement for parameter $\ell_2$ distance and $\lambda$ error is $(\text{baseline} - \text{method})/\text{baseline} \times 100\%$, and for likelihood ratio it is $(|\text{baseline}-1| - |\text{method}-1|)/|\text{baseline}-1| \times 100\%$. The parameters $\ell_2$ distance and $\lambda$ error are better when smaller, and the log-likelihood ratio is better when closer to 1. Negative relative improvements are shown as 0. The best performance for each metric in each data generation process is highlighted in \textbf{bold}. 
\end{tablenotes}
\end{table}

\begin{table}[htbp]
\centering
\caption{Performance comparison of different coreset methods for various data generation processes (coreset Size = 100)}
\label{tab:performance_comparison_100}
\resizebox{\textwidth}{!}{
\begin{tabular}{llccccc}
\toprule
\textbf{Data Gen. Process} & \textbf{Method} & \textbf{Param. $\ell_2$ dist.} & \textbf{$\lambda$ error} & \textbf{Likelihood ratio} & \textbf{Rel. Impr.(\%)} & \textbf{Total time(s)} \\
\midrule
\multirow{3}{*}{Bivariate normal} 
& $\ell_2$-hull & $\mathbf{1.80 \pm 0.42}$ & $0.30 \pm 0.07$ & $1.32 \pm 0.11$ & 0 & $0.23 \pm 0.02$ \\
& $\ell_2$-only & $2.01 \pm 0.36$ & $0.39 \pm 0.09$ & $1.48 \pm 0.17$ & 0 & $0.24 \pm 0.04$ \\
& uniform & $1.98 \pm 0.66$ & $\mathbf{0.11 \pm 0.06}$ & $\mathbf{1.19 \pm 0.08}$ & baseline & $\mathbf{0.13 \pm 0.02}$ \\

\midrule

\multirow{3}{*}{Non-linear correlation}
& $\ell_2$-hull & $\mathbf{1.05 \pm 0.42}$ & $\mathbf{0.03 \pm 0.03}$ & $\mathbf{1.01 \pm 0.00}$ & \textbf{52.6} & $0.20 \pm 0.02$ \\
& $\ell_2$-only & $1.20 \pm 0.53$ & $\mathbf{0.03 \pm 0.03}$ & $\mathbf{1.01 \pm 0.01}$ & 38.8 & $0.22 \pm 0.02$ \\
& uniform & $1.79 \pm 0.79$ & $0.08 \pm 0.05$ & $1.02 \pm 0.01$ & baseline & $\mathbf{0.11 \pm 0.02}$ \\

\midrule

\multirow{3}{*}{Bivariate normal mixture}
& $\ell_2$-hull & $1.54 \pm 0.42$ & $\mathbf{0.04 \pm 0.04}$ & $\mathbf{1.06 \pm 0.01}$ & \textbf{34.6} & $0.30 \pm 0.22$ \\
& $\ell_2$-only & $\mathbf{1.53 \pm 0.45}$ & $0.07 \pm 0.04$ & $\mathbf{1.06 \pm 0.02}$ & 26.4 & $0.21 \pm 0.02$ \\
& uniform & $1.99 \pm 0.58$ & $0.09 \pm 0.09$ & $1.09 \pm 0.04$ & baseline & $\mathbf{0.13 \pm 0.02}$ \\

\midrule

\multirow{3}{*}{Geometric Mixed Distribution}
& $\ell_2$-hull & $\mathbf{1.19 \pm 0.51}$ & $\mathbf{0.02 \pm 0.02}$ & $\mathbf{1.01 \pm 0.00}$ & \textbf{49.4} & $0.20 \pm 0.02$ \\
& $\ell_2$-only & $1.27 \pm 0.60$ & $0.03 \pm 0.02$ & $1.01 \pm 0.00$ & 42.1 & $0.21 \pm 0.03$ \\
& uniform & $1.75 \pm 0.66$ & $0.06 \pm 0.04$ & $1.02 \pm 0.01$ & baseline & $0.13 \pm 0.03$ \\

\midrule

\multirow{3}{*}{Skew-t distribution}
& $\ell_2$-hull & $\mathbf{2.04 \pm 0.42}$ & $0.26 \pm 0.14$ & $\mathbf{1.24 \pm 0.10}$ & 8.1 & $0.22 \pm 0.01$ \\
& $\ell_2$-only & $\mathbf{1.99 \pm 0.33}$ & $0.33 \pm 0.16$ & $1.28 \pm 0.12$ & 0 & $0.22 \pm 0.03$ \\
& uniform & $2.67 \pm 1.11$ & $\mathbf{0.20 \pm 0.15}$ & $1.34 \pm 0.25$ & baseline & $\mathbf{0.12 \pm 0.02}$ \\

\midrule

\multirow{3}{*}{Heteroscedastic distribution}
& $\ell_2$-hull & $\mathbf{1.06 \pm 0.27}$ & $\mathbf{0.04 \pm 0.03}$ & $\mathbf{1.01 \pm 0.00}$ & \textbf{27.6} & $0.24 \pm 0.08$ \\
& $\ell_2$-only & $1.34 \pm 0.40$ & $\mathbf{0.04 \pm 0.02}$ & $1.02 \pm 0.00$ & 11.3 & $0.21 \pm 0.02$ \\
& uniform & $1.46 \pm 0.84$ & $0.09 \pm 0.05$ & $1.01 \pm 0.01$ & baseline & $\mathbf{0.13 \pm 0.02}$ \\

\midrule

\multirow{3}{*}{Copula complex distribution}
& $\ell_2$-hull & $\mathbf{3.53 \pm 0.70}$ & $\mathbf{0.12 \pm 0.09}$ & $\mathbf{1.15 \pm 0.10}$ & \textbf{32.0} & $0.26 \pm 0.04$ \\
& $\ell_2$-only & $4.05 \pm 0.78$ & $0.13 \pm 0.08$ & $1.16 \pm 0.08$ & 24.1 & $0.26 \pm 0.03$ \\
& uniform & $3.75 \pm 1.42$ & $0.18 \pm 0.15$ & $1.34 \pm 0.27$ & baseline & $\mathbf{0.18 \pm 0.04}$ \\

\midrule

\multirow{3}{*}{Spiral dependency}
& $\ell_2$-hull & $\mathbf{1.37 \pm 0.43}$ & $\mathbf{0.02 \pm 0.02}$ & $\mathbf{1.01 \pm 0.01}$ & \textbf{34.4} & $0.25 \pm 0.08$ \\
& $\ell_2$-only & $1.52 \pm 0.37$ & $0.03 \pm 0.02$ & $1.03 \pm 0.01$ & 0 & $0.27 \pm 0.19$ \\
& uniform & $1.88 \pm 0.57$ & $0.04 \pm 0.03$ & $1.02 \pm 0.01$ & baseline & $\mathbf{0.13 \pm 0.01}$ \\

\midrule

\multirow{3}{*}{Circular dependency}
& $\ell_2$-hull & $\mathbf{0.94 \pm 0.36}$ & $\mathbf{0.02 \pm 0.01}$ & $\mathbf{1.01 \pm 0.00}$ & \textbf{54.2} & $0.21 \pm 0.01$ \\
& $\ell_2$-only & $1.05 \pm 0.20$ & $0.03 \pm 0.02$ & $\mathbf{1.01 \pm 0.00}$ & 43.9 & $0.21 \pm 0.01$ \\
& uniform & $1.83 \pm 0.66$ & $0.05 \pm 0.04$ & $1.01 \pm 0.01$ & baseline & $\mathbf{0.12 \pm 0.01}$ \\

\midrule

\multirow{3}{*}{t Copula}
& $\ell_2$-hull & $3.86 \pm 1.03$ & $0.28 \pm 0.25$ & $1.39 \pm 0.38$ & 0 & $0.27 \pm 0.04$ \\
& $\ell_2$-only & $4.71 \pm 1.26$ & $0.32 \pm 0.27$ & $1.49 \pm 0.50$ & 0 & $0.28 \pm 0.15$ \\
& uniform & $\mathbf{2.43 \pm 1.04}$ & $\mathbf{0.13 \pm 0.14}$ & $\mathbf{1.16 \pm 0.16}$ & baseline & $\mathbf{0.19 \pm 0.11}$ \\

\midrule

\multirow{3}{*}{Piecewise dependency}
& $\ell_2$-hull & $\mathbf{1.06 \pm 0.36}$ & $\mathbf{0.09 \pm 0.08}$ & $\mathbf{1.01 \pm 0.01}$ & \textbf{49.8} & $0.22 \pm 0.03$ \\
& $\ell_2$-only & $1.29 \pm 0.27$ & $0.12 \pm 0.11$ & $\mathbf{1.01 \pm 0.01}$ & 29.8 & $0.21 \pm 0.02$ \\
& uniform & $1.62 \pm 0.90$ & $0.15 \pm 0.16$ & $1.03 \pm 0.04$ & baseline & $\mathbf{0.12 \pm 0.02}$ \\

\midrule

\multirow{3}{*}{Hourglass dependency}
& $\ell_2$-hull & $\mathbf{0.96 \pm 0.33}$ & $\mathbf{0.08 \pm 0.09}$ & $\mathbf{1.01 \pm 0.00}$ & \textbf{46.0} & $0.21 \pm 0.02$ \\
& $\ell_2$-only & $0.97 \pm 0.40$ & $0.09 \pm 0.09$ & $1.02 \pm 0.01$ & 38.6 & $0.20 \pm 0.01$ \\
& uniform & $1.72 \pm 0.54$ & $0.12 \pm 0.09$ & $1.04 \pm 0.02$ & baseline & $\mathbf{0.11 \pm 0.01}$ \\

\midrule

\multirow{3}{*}{Bimodal clusters}
& $\ell_2$-hull & $\mathbf{0.95 \pm 0.24}$ & $\mathbf{0.05 \pm 0.05}$ & $\mathbf{1.01 \pm 0.00}$ & \textbf{46.5} & $0.20 \pm 0.01$ \\
& $\ell_2$-only & $1.03 \pm 0.32$ & $0.06 \pm 0.04$ & $1.02 \pm 0.00$ & 34.9 & $0.20 \pm 0.01$ \\
& uniform & $1.63 \pm 0.53$ & $0.11 \pm 0.10$ & $1.02 \pm 0.01$ & baseline & $\mathbf{0.12 \pm 0.01}$ \\

\midrule

\multirow{3}{*}{Sinusoidal dependency}
& $\ell_2$-hull & $\mathbf{1.24 \pm 0.39}$ & $\mathbf{0.03 \pm 0.03}$ & $\mathbf{1.01 \pm 0.00}$ & \textbf{42.0} & $0.20 \pm 0.01$ \\
& $\ell_2$-only & $1.28 \pm 0.44$ & $0.07 \pm 0.04$ & $\mathbf{1.01 \pm 0.01}$ & 15.6 & $0.21 \pm 0.01$ \\
& uniform & $1.36 \pm 0.40$ & $0.07 \pm 0.04$ & $1.02 \pm 0.01$ & baseline & $\mathbf{0.11 \pm 0.01}$ \\

\bottomrule
\end{tabular}
}
\begin{tablenotes}
\small
\item Note: The results in the table are the mean $\pm 1$ standard deviation of independent simulations. The Relative Improvement is calculated as the average percentage improvement across all metrics, where improvement for parameter $\ell_2$ distance and $\lambda$ error is $(\text{baseline} - \text{method})/\text{baseline} \times 100\%$, and for likelihood ratio it is $(|\text{baseline}-1| - |\text{method}-1|)/|\text{baseline}-1| \times 100\%$. The parameters $\ell_2$ distance and $\lambda$ error are better when smaller, and the log-likelihood ratio is better when closer to 1. Negative relative improvements are shown as 0. The best performance for each metric in each data generation process is highlighted in \textbf{bold}.
\end{tablenotes}
\end{table}

Figure~\ref{fig:simulation_results1} visualizes the performance specifically for three distributions: bivariate normal mixture, non-linear correlation, and bimodal cluster distribution. In Figure~\ref{fig:simulation_results1}, the red line corresponds to our proposed method combining $\ell_2$ subsampling with convex hull approximation, the blue line denotes the $\ell_2$ subsampling method alone, and the green line indicates the traditional uniform subsampling method. All simulations are based on an original dataset size of $10\,000$ points.

\begin{figure}[htbp]
  \centering
  \includegraphics[width=0.3\textwidth]{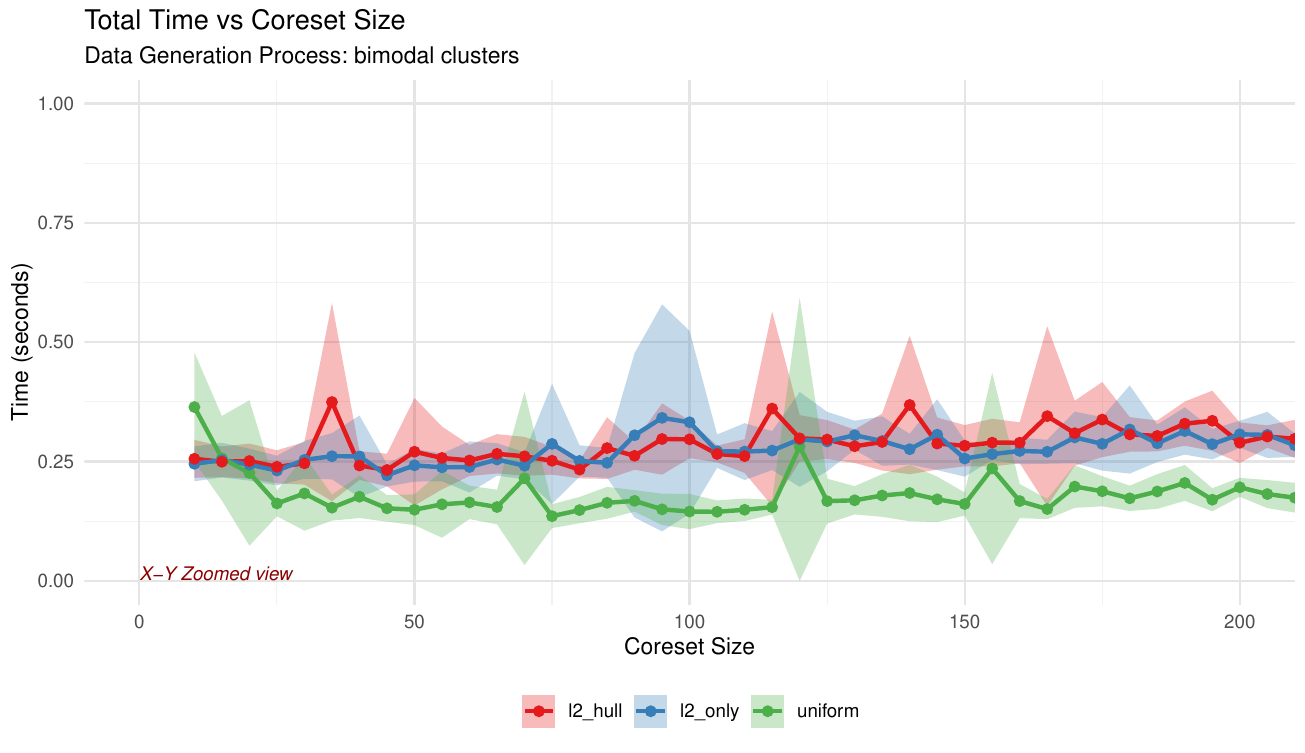}\hfill
  \includegraphics[width=0.3\textwidth]{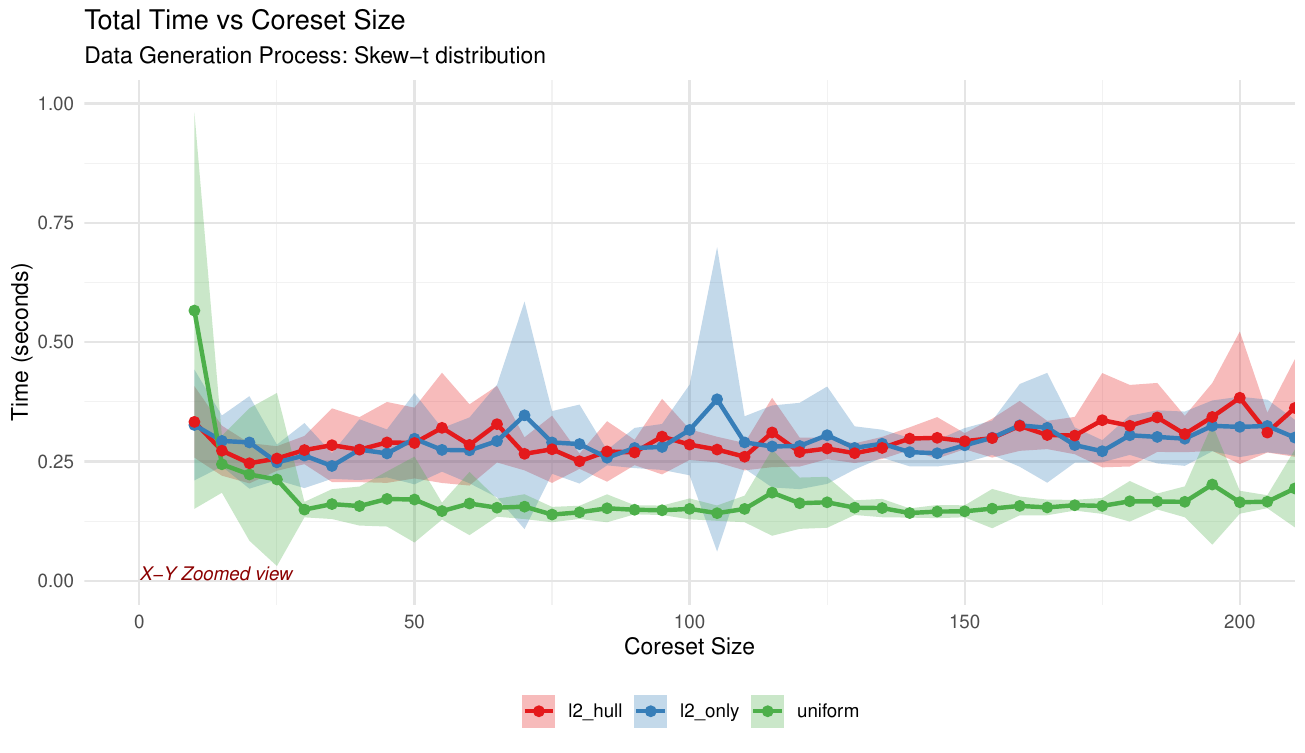}\hfill
  \includegraphics[width=0.3\textwidth]{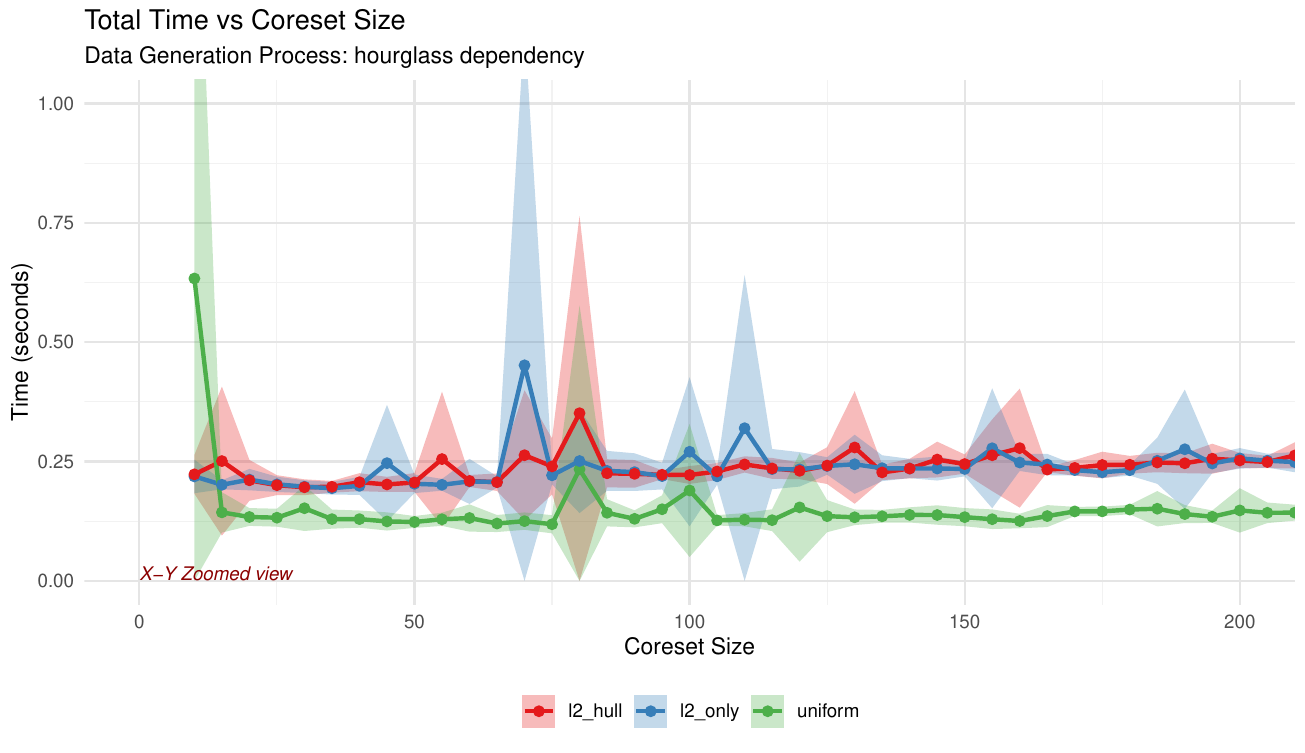}

  \vspace{0.5em}
  \includegraphics[width=0.3\textwidth]{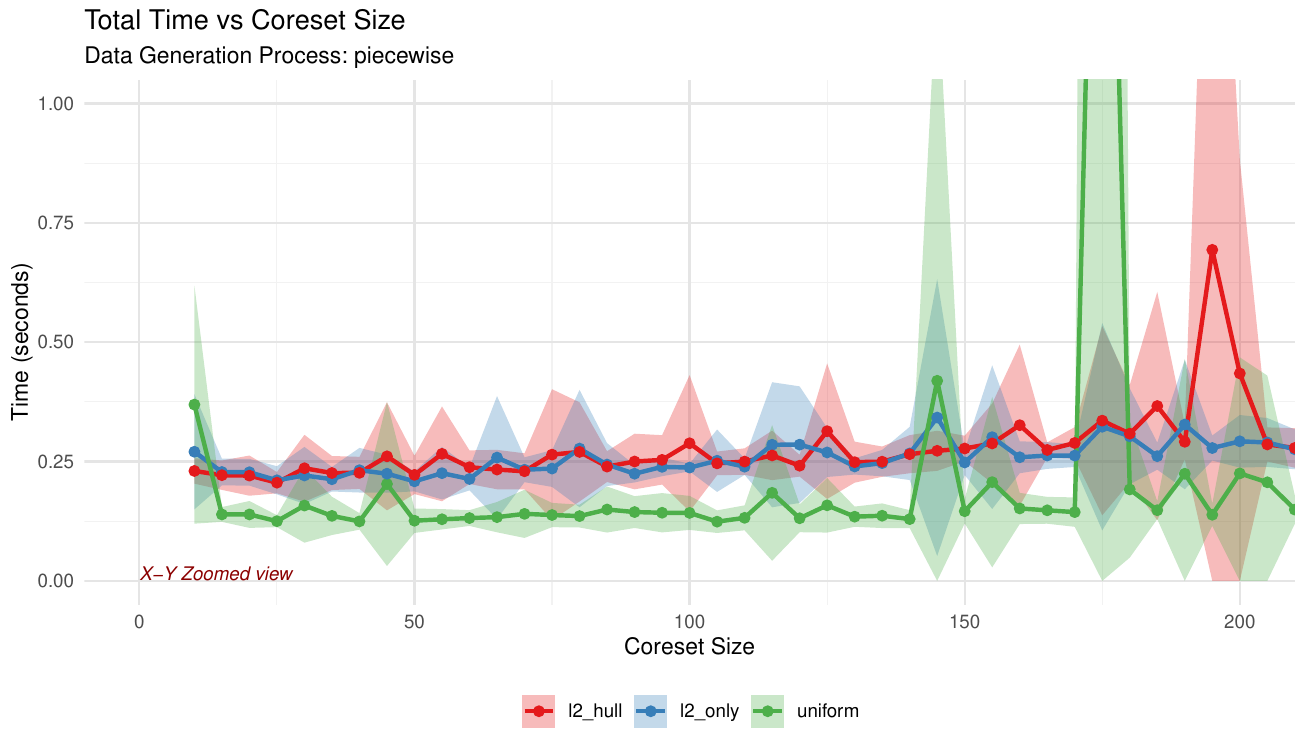}\hfill
  \includegraphics[width=0.3\textwidth]{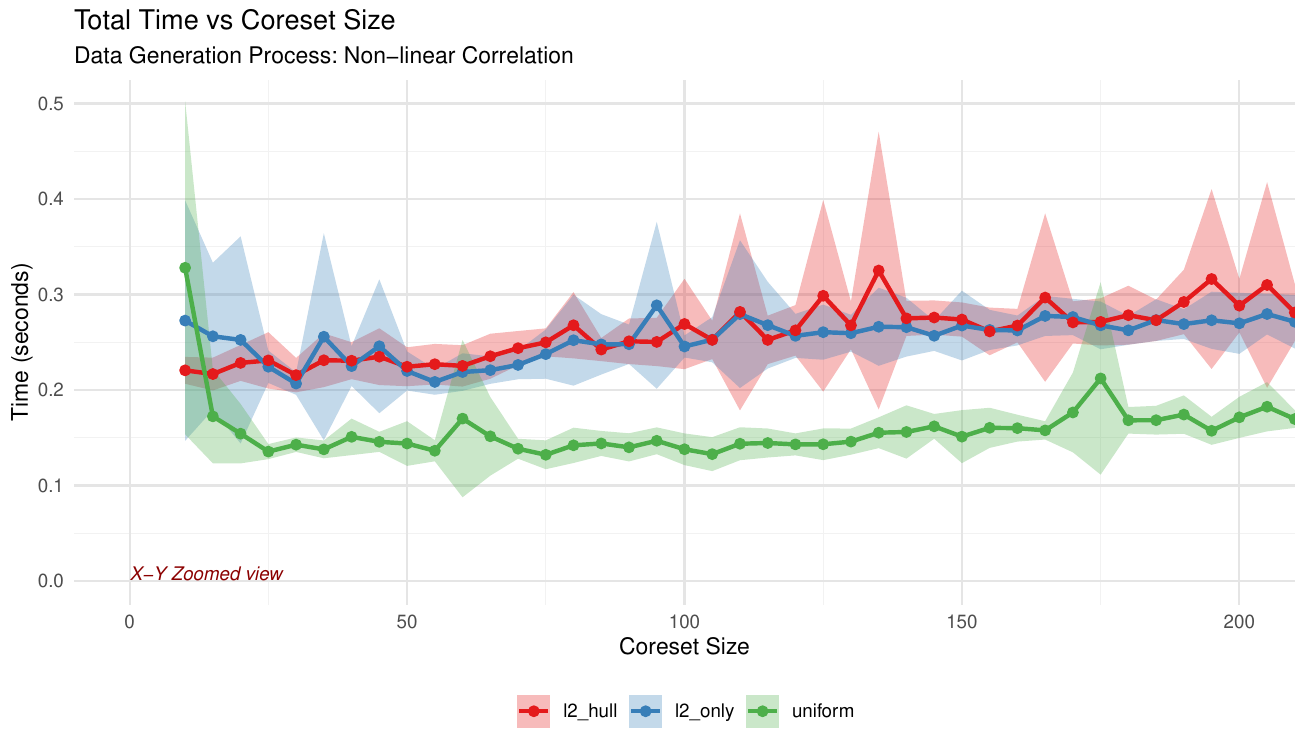}\hfill
  \includegraphics[width=0.3\textwidth]{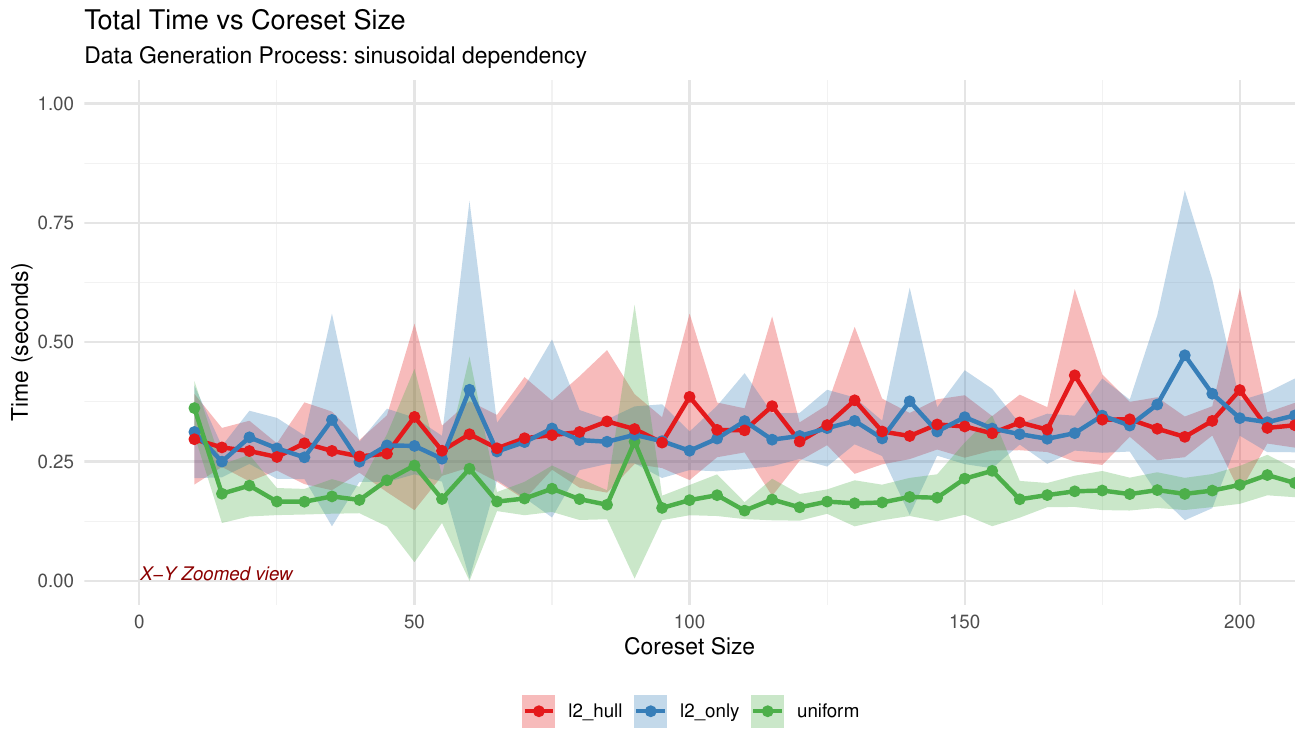}

  \vspace{0.5em}
  \includegraphics[width=0.3\textwidth]{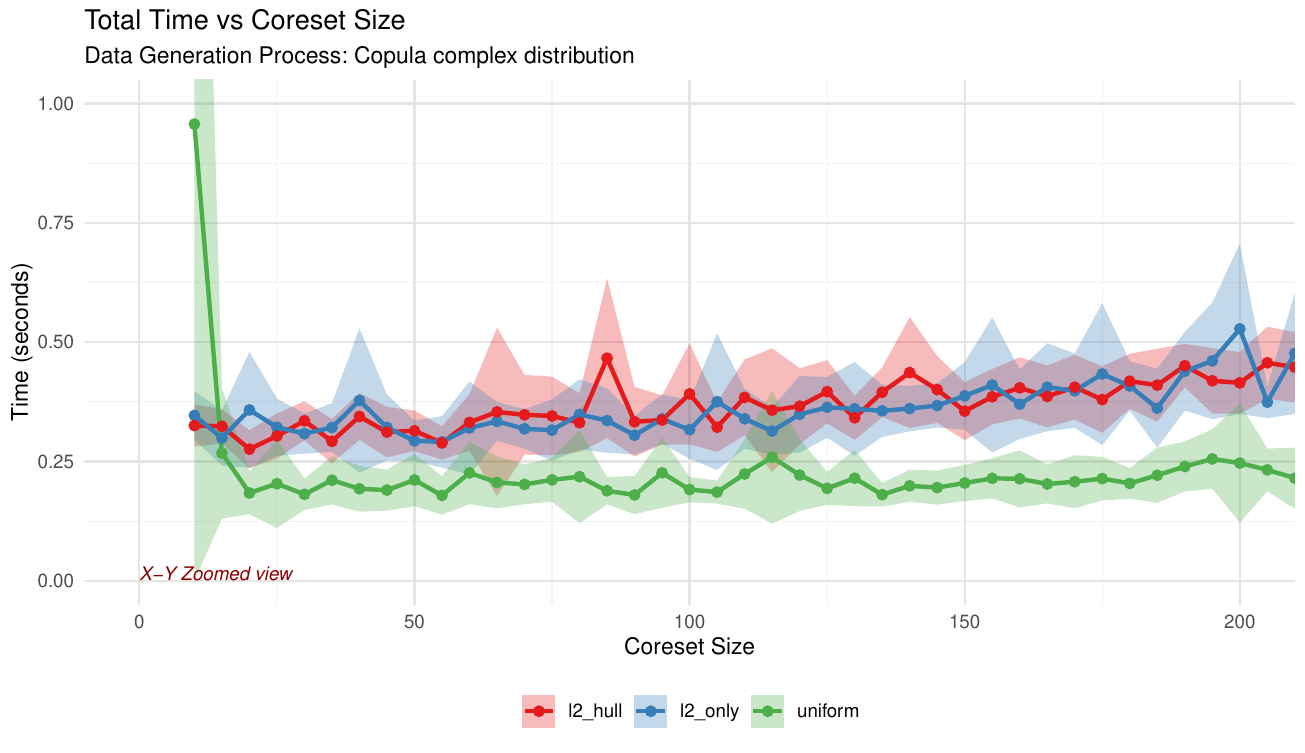}\hfill
  \includegraphics[width=0.3\textwidth]{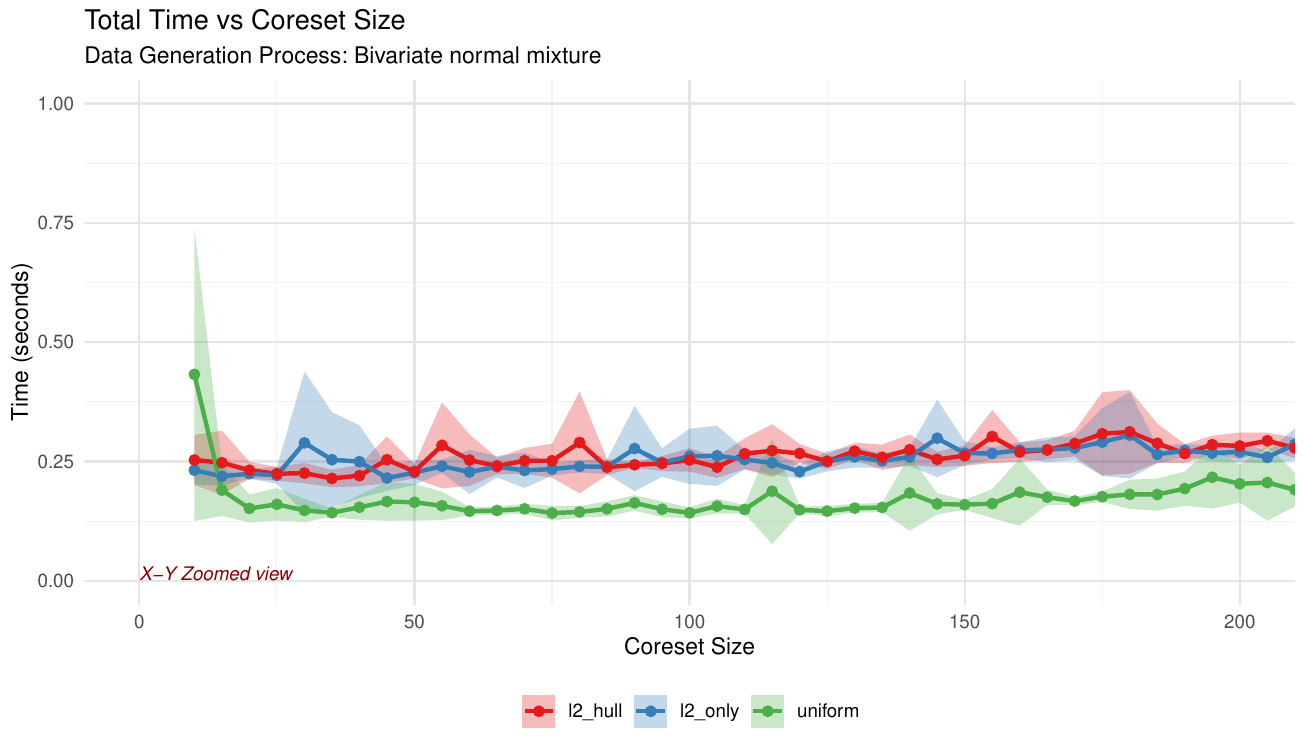}\hfill
  \includegraphics[width=0.3\textwidth]{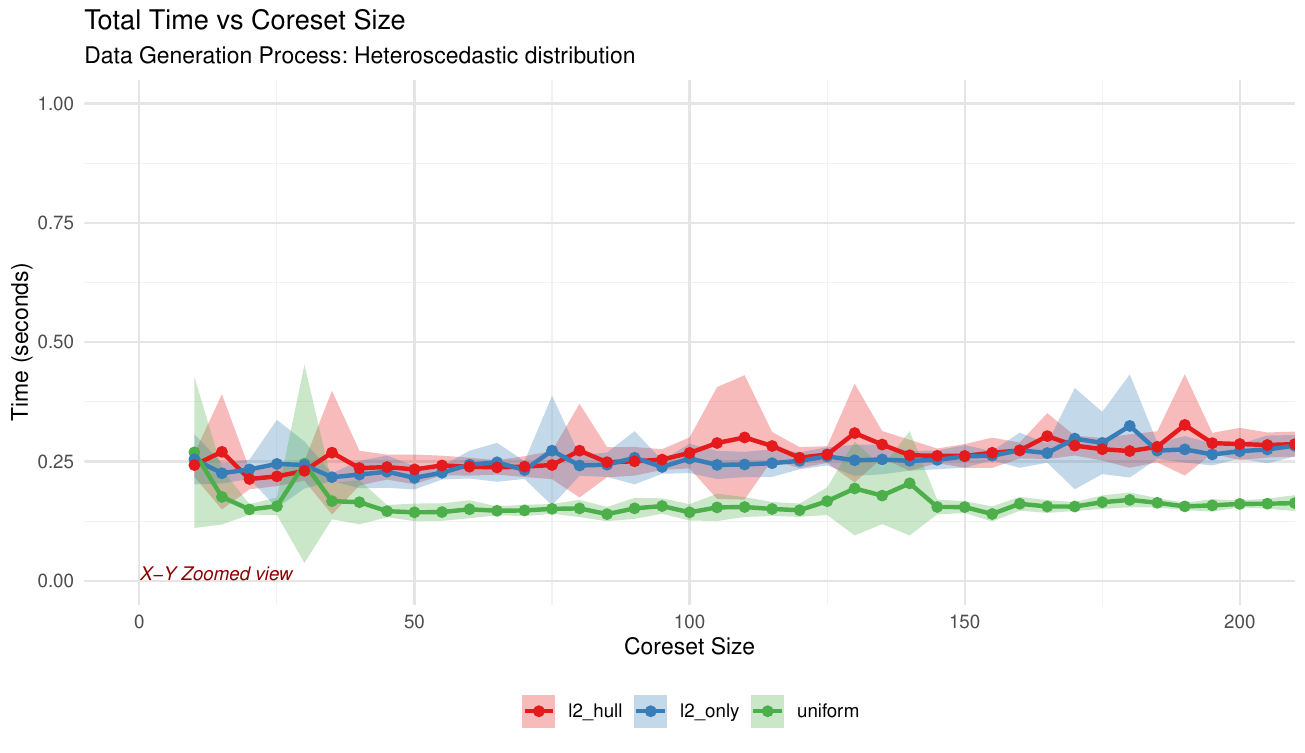}

  \caption{Computation time for 9 simulation distributions}
  \label{fig:sim_times}
\end{figure}

To further illustrate the advantages of our coreset approach, we show the effect of reconstructing the univariate marginal densities after constructing the coreset with three different sampling strategies on simulated data from a binary normal distribution, with realistic marginal distribution predictions for $X$ and $Y$ in Figure~\ref{fig:pred_marg_x} and Figure~\ref{fig:pred_marg_y}, respectively. Each column corresponds to a different coreset size $k=50,100,500$, while the three rows show uniform sampling, $\ell_2$ sampling, and $\ell_2$-hull sampling, respectively. To reflect the stability of the methods, each subplot of the figure plots a single estimate from 10 replicate trials (thin light-gray colored line) and shows their average result as a solid black line, while the red line gives the true theoretical density. It can be seen that when the coreset size is small (e.g., $k=50$), uniform sampling is too random and often fails to cover both ends of the distribution, leading to significant deviations in the predicted curves; using only $\ell_2$ leverage sampling, with a relatively small coreset size, there is still a relatively high risk of failure.; and introduction of the convex hull can effectively ensure fits the theoretical curves better at the minimum size. As $k$ increases, the average predictions of all three methods gradually converge to the red true density - but at any scale, the $\ell_2$ + convex hull scheme reproduces the main shape of the distribution earliest and most consistently, fully reflecting its ability to efficiently capture the distribution with small samples.

\begin{figure}[!ht]
    \centering
    \begin{minipage}{.32\textwidth}
        \centering
        \includegraphics[width=\linewidth]{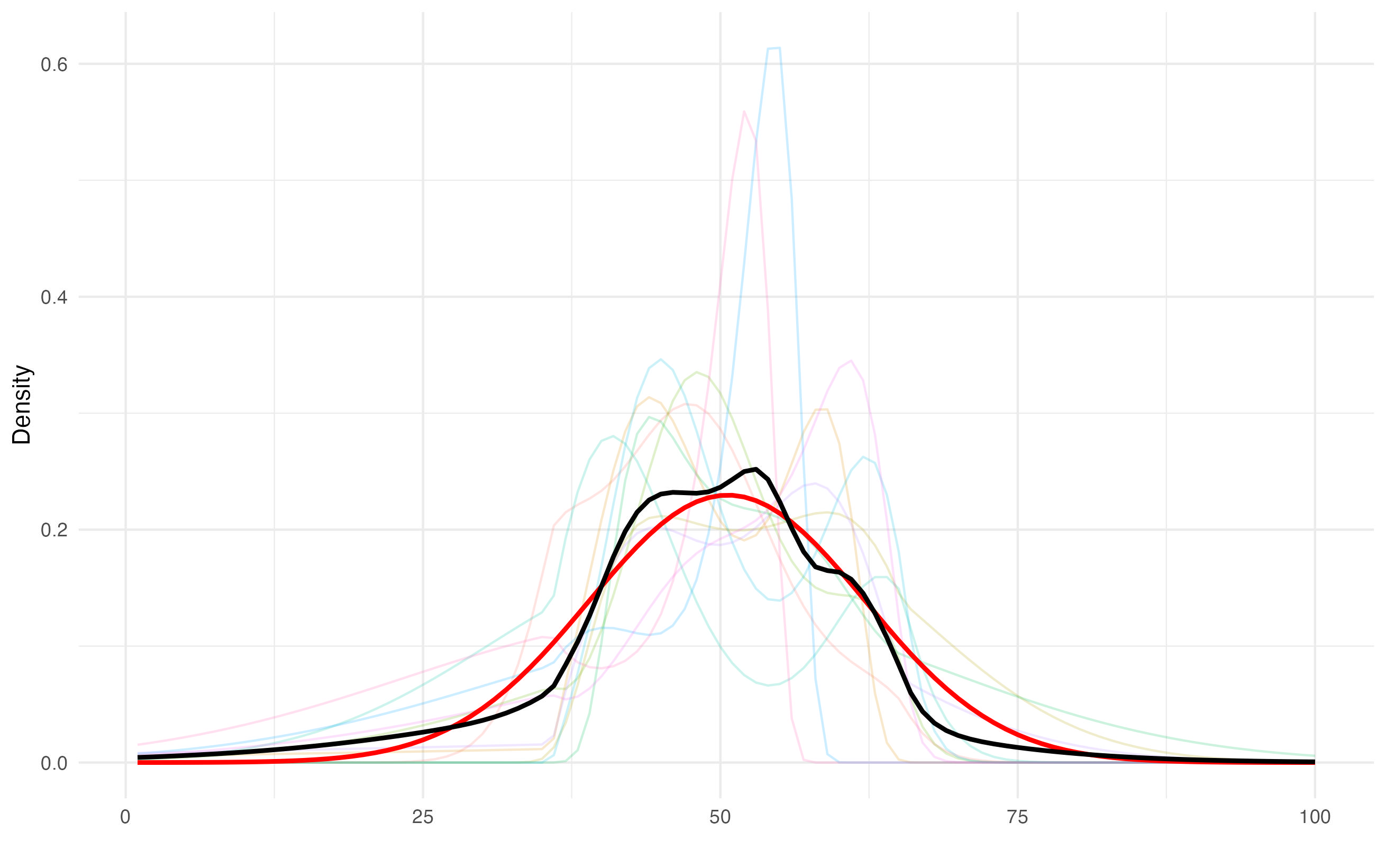} 
    \end{minipage}\hfill
    \begin{minipage}{.32\textwidth}
        \centering
        \includegraphics[width=\linewidth]{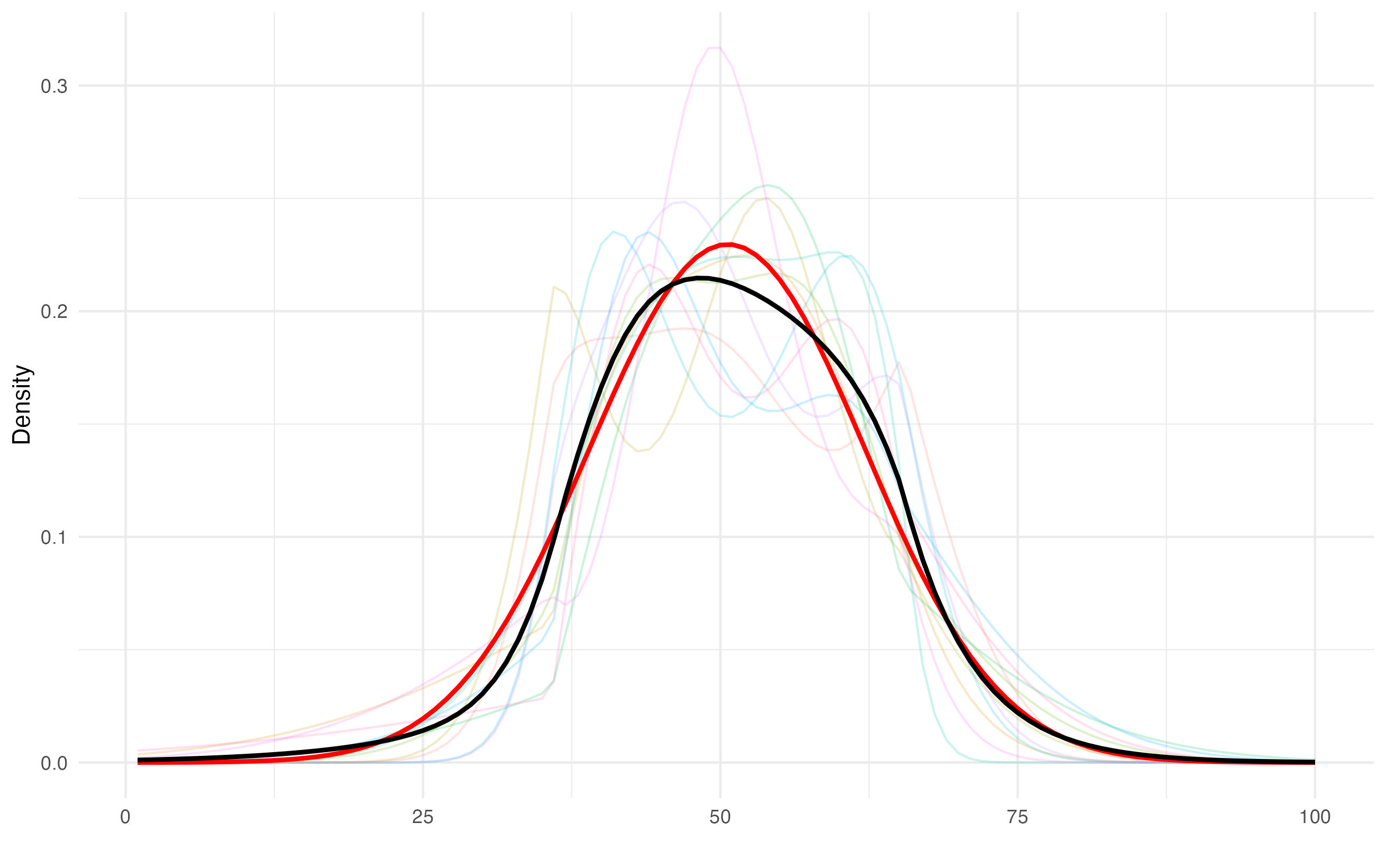} 
    \end{minipage}\hfill
    \begin{minipage}{.32\textwidth}
        \centering
        \includegraphics[width=\linewidth]{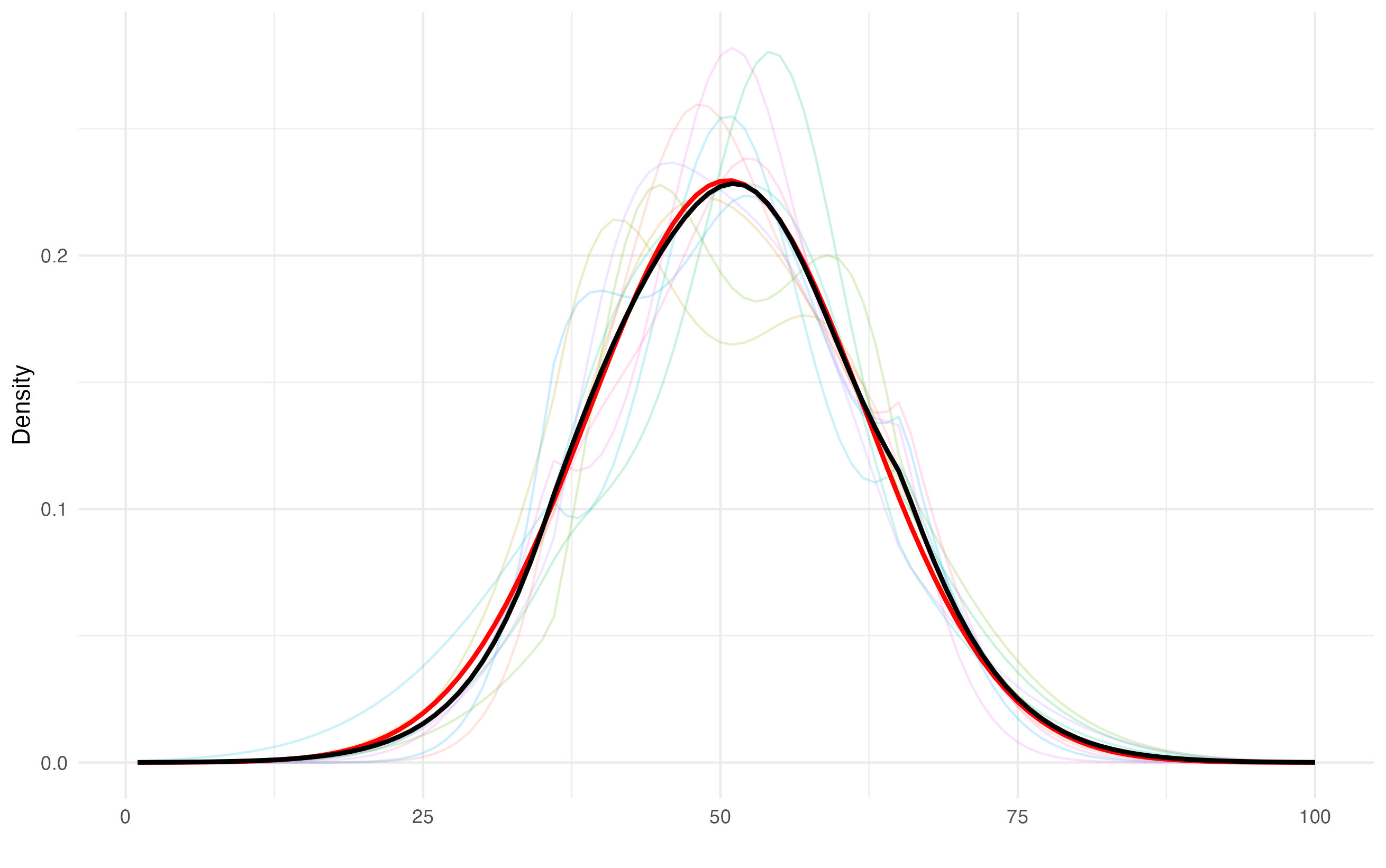} 
    \end{minipage}
    \begin{minipage}{.32\textwidth}
        \centering
        \includegraphics[width=\linewidth]{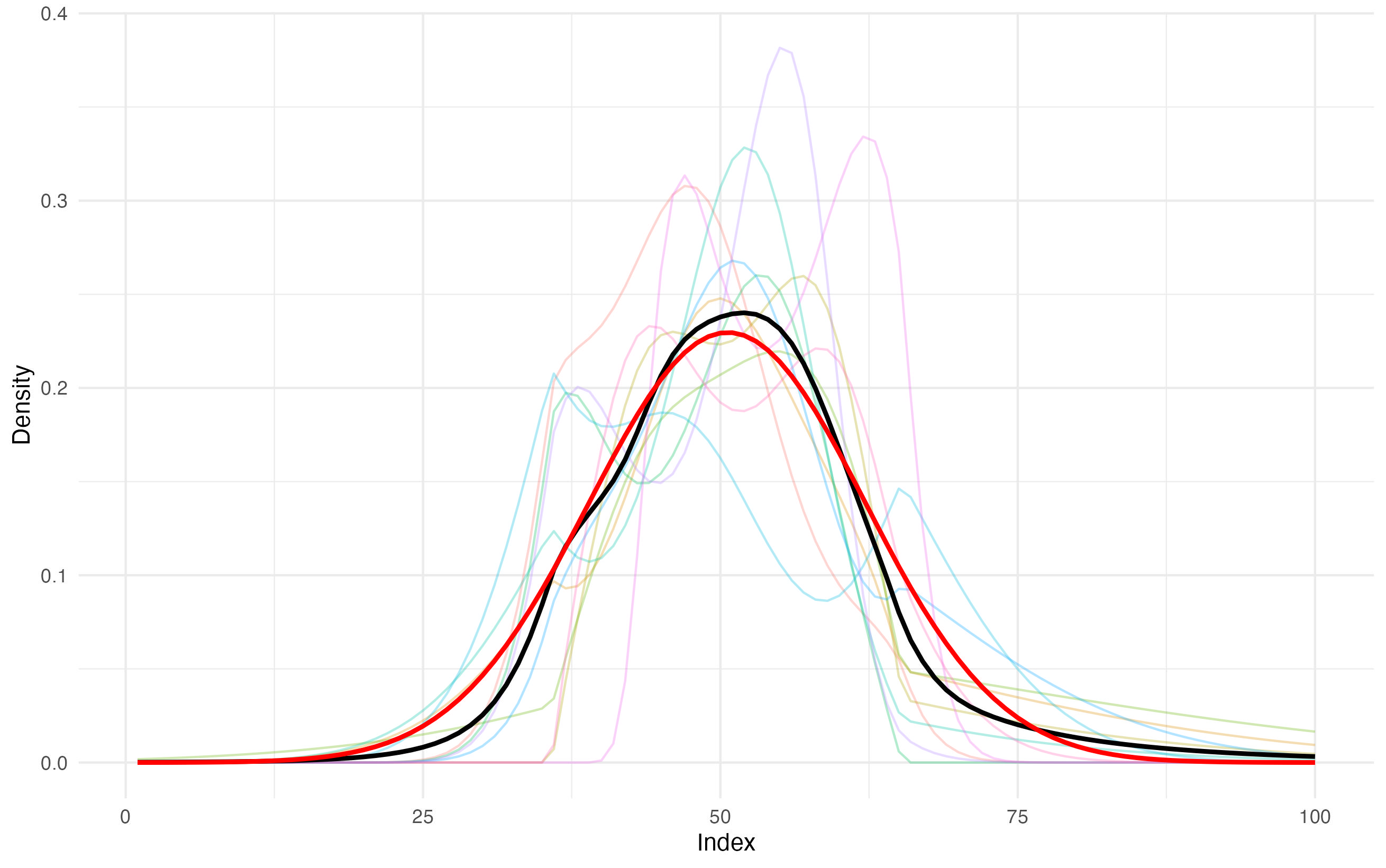} 
    \end{minipage}\hfill
    \begin{minipage}{.32\textwidth}
        \centering
        \includegraphics[width=\linewidth]{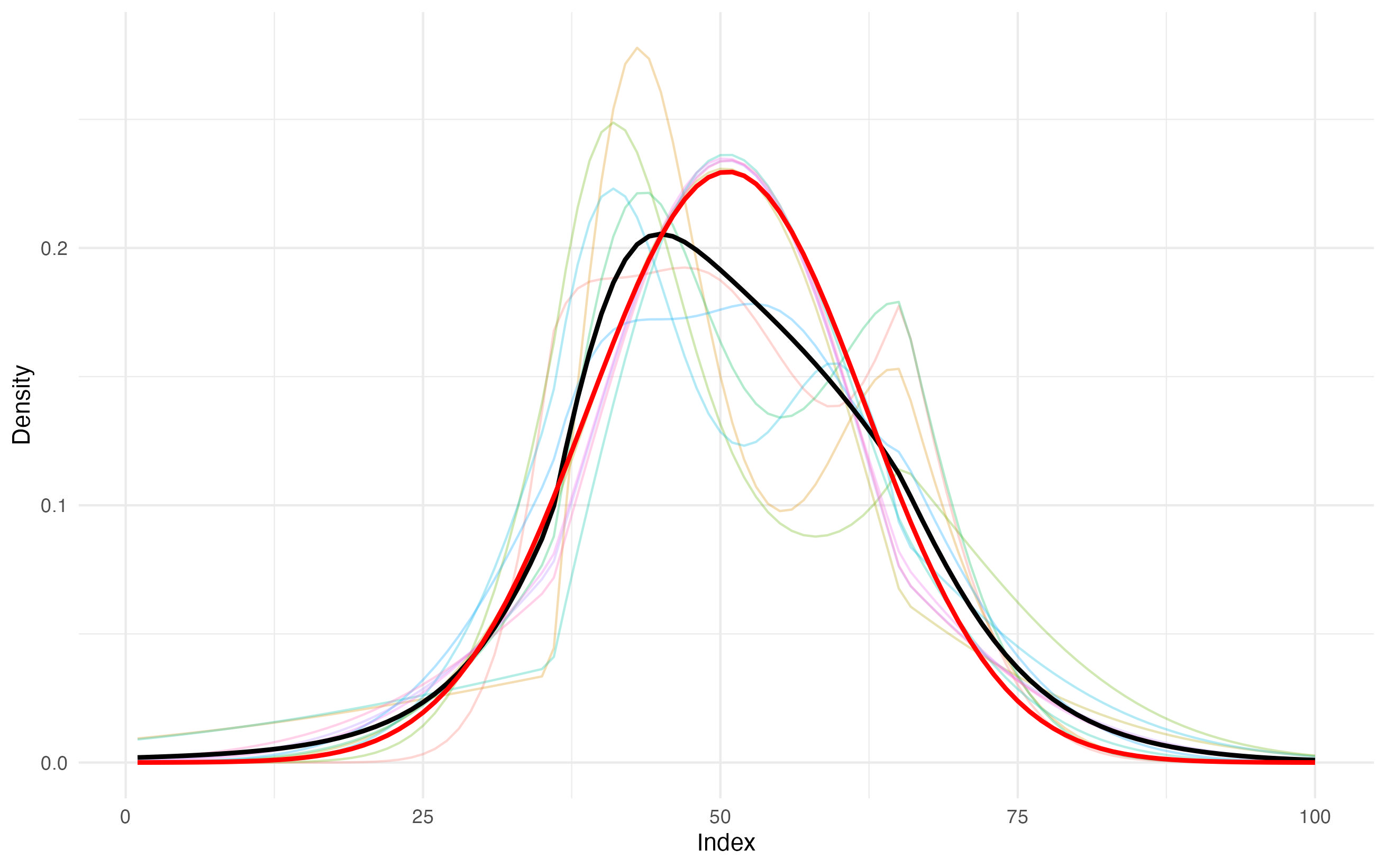} 
    \end{minipage}\hfill
    \begin{minipage}{.32\textwidth}
        \centering
        \includegraphics[width=\linewidth]{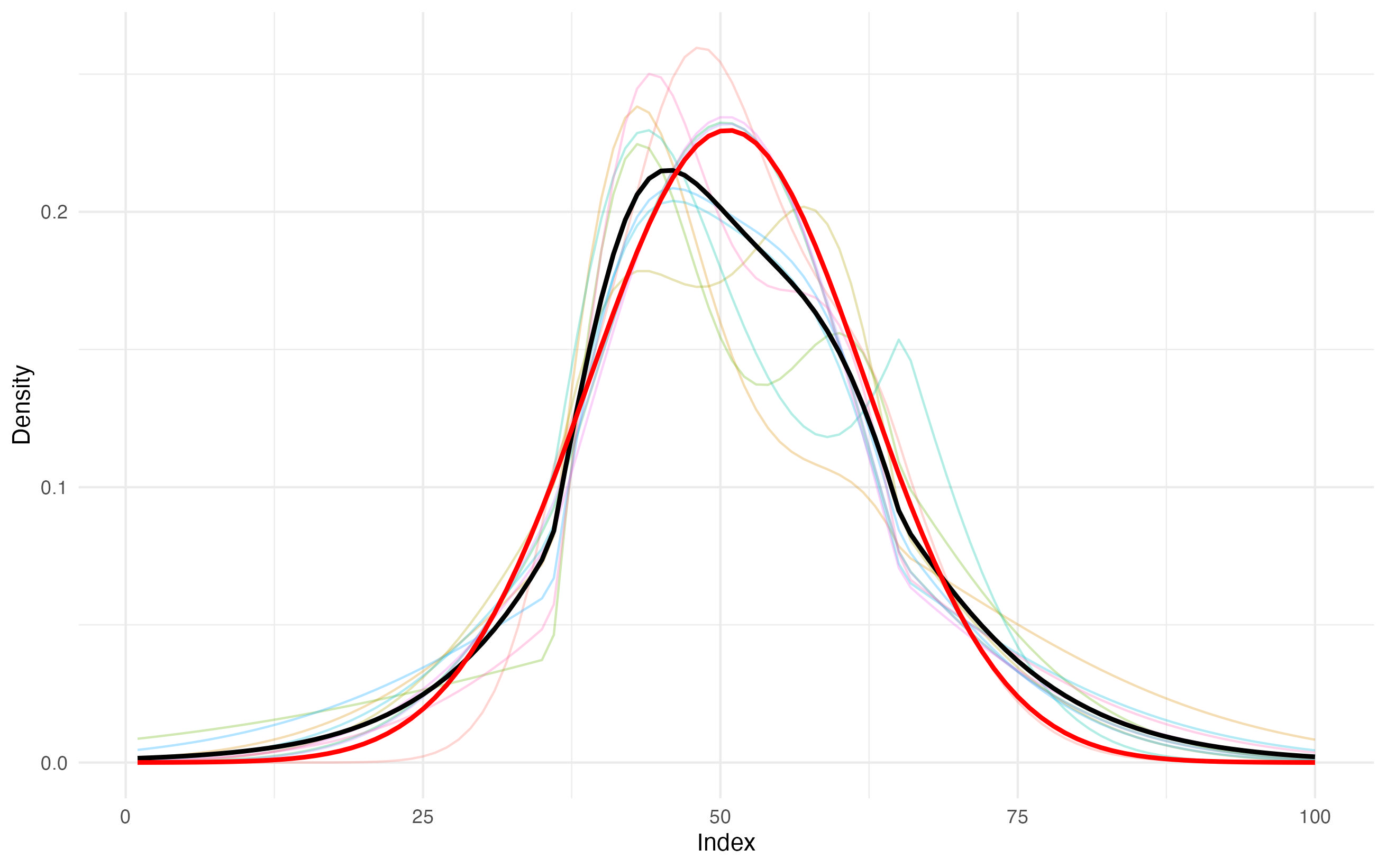} 
    \end{minipage}
    \begin{minipage}{.32\textwidth}
        \centering
        \includegraphics[width=\linewidth]{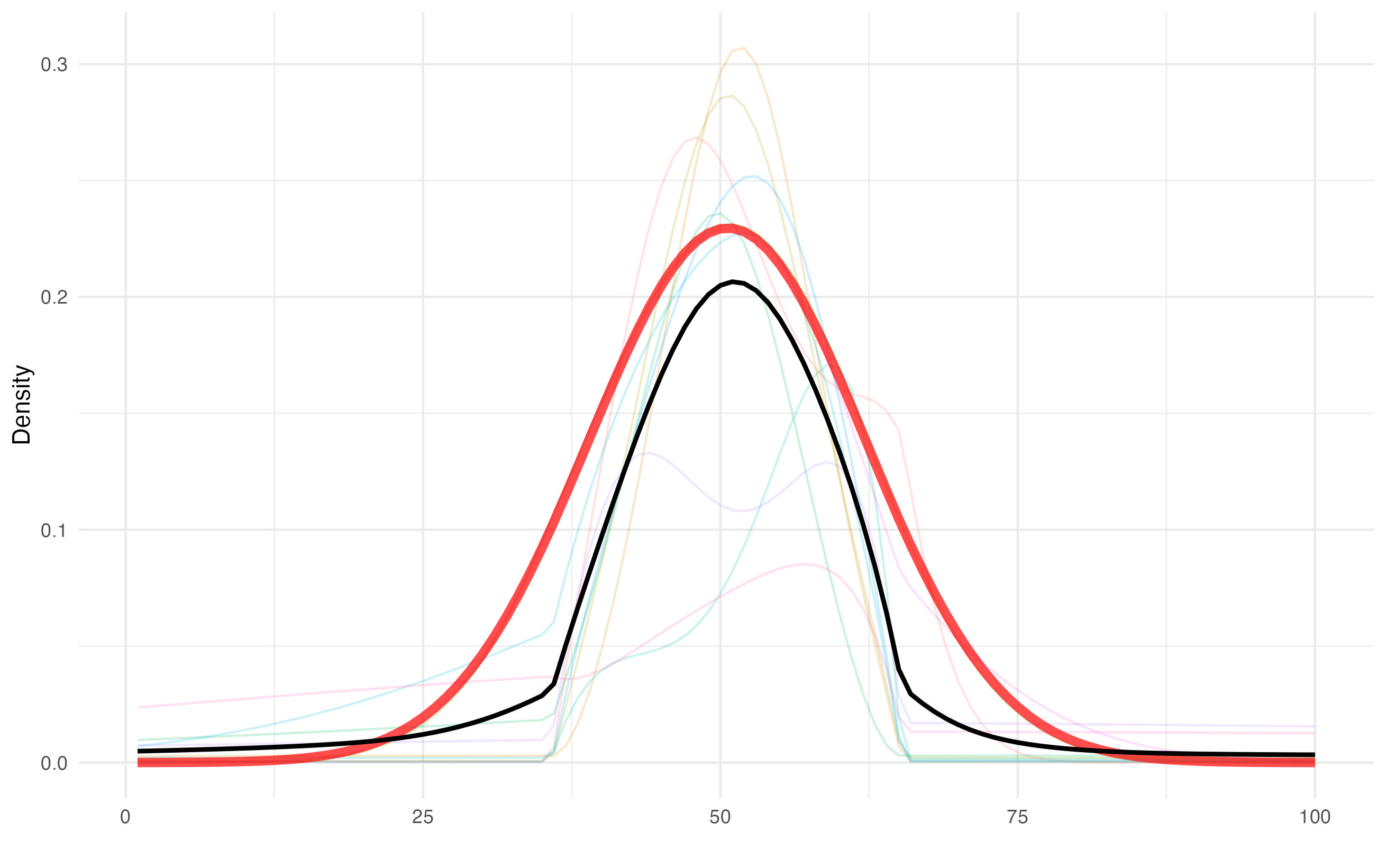} 
    \end{minipage}\hfill
    \begin{minipage}{.32\textwidth}
        \centering
        \includegraphics[width=\linewidth]{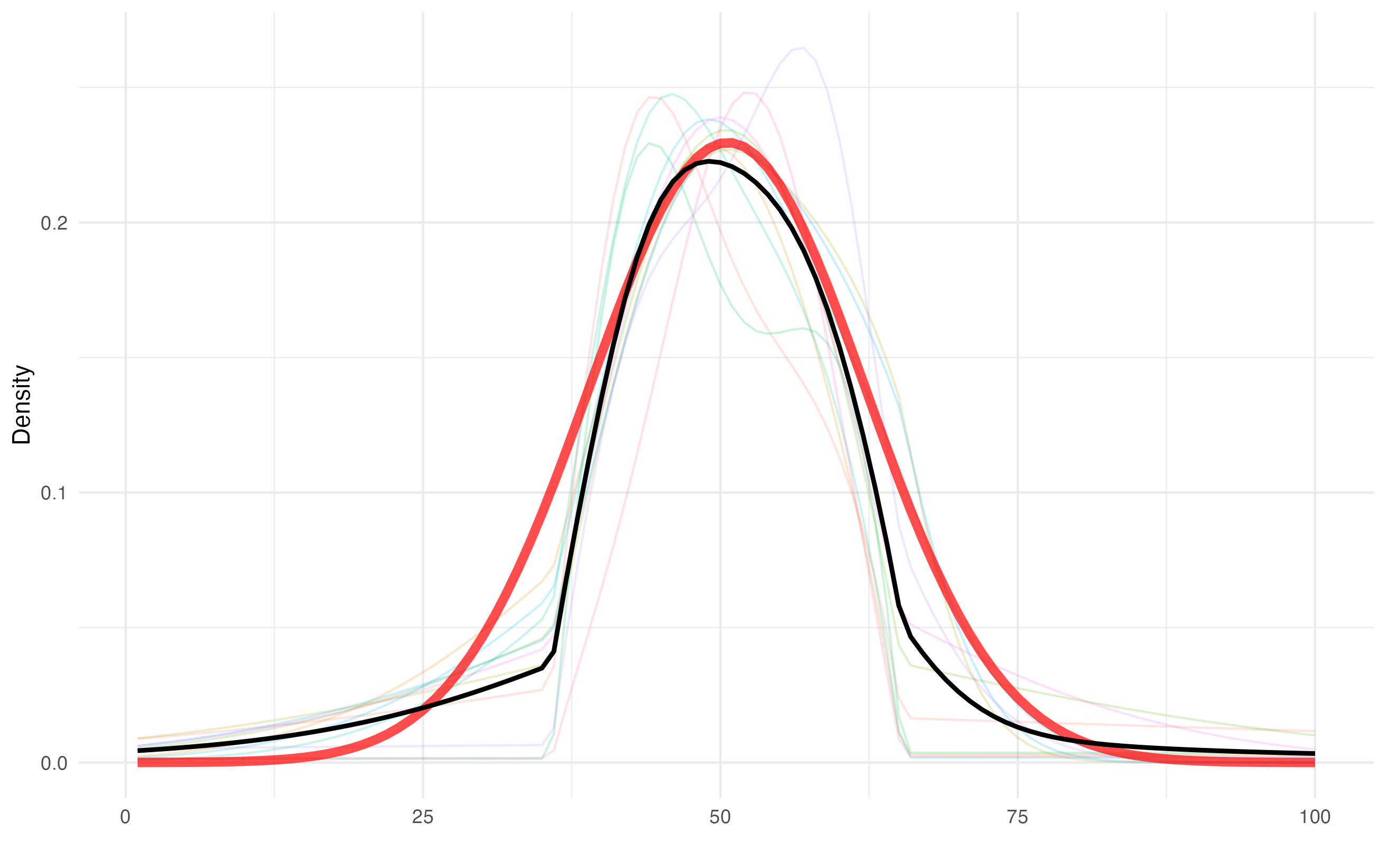} 
    \end{minipage}\hfill
    \begin{minipage}{.32\textwidth}
        \centering
        \includegraphics[width=\linewidth]{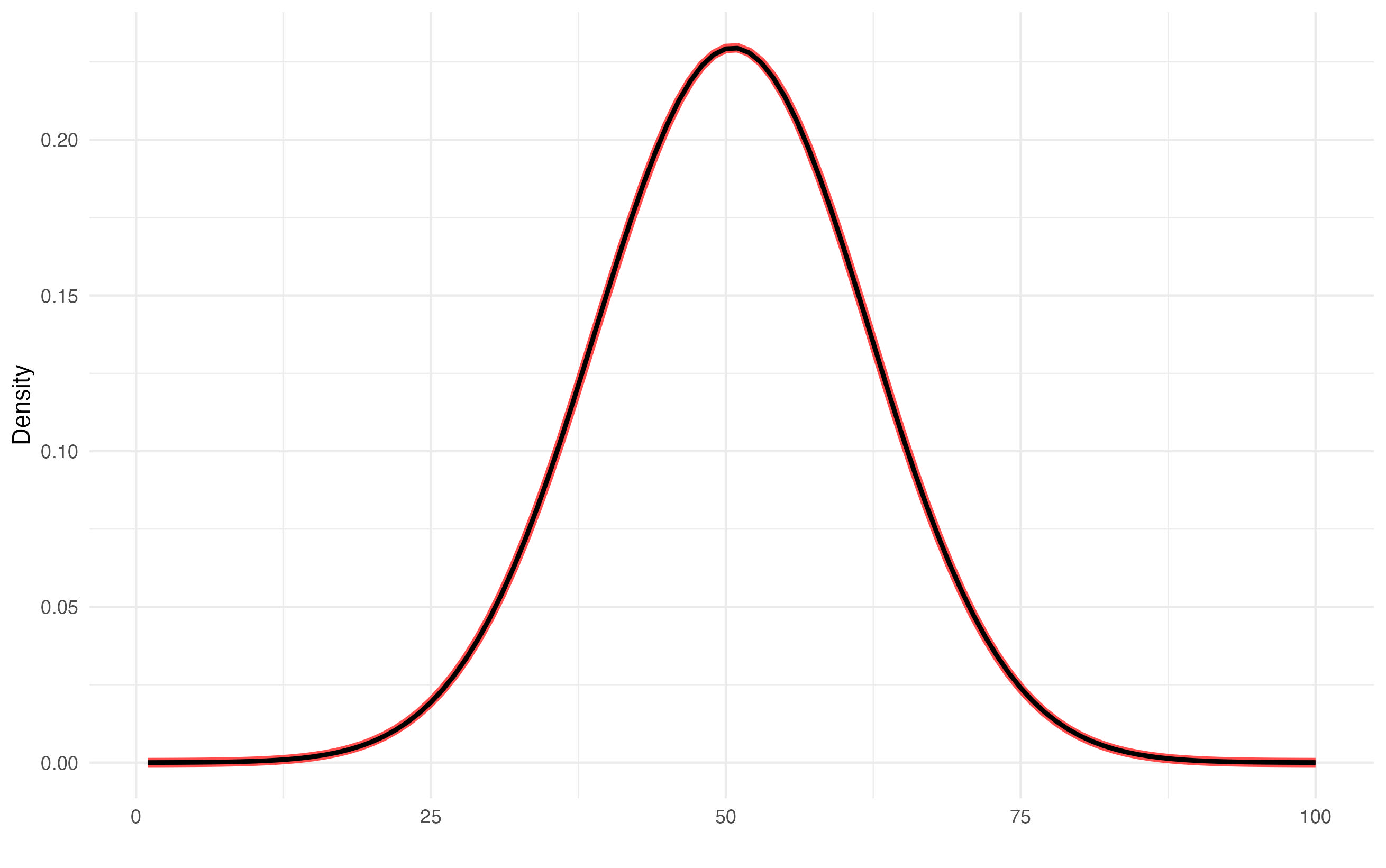} 
    \end{minipage}
    \caption{Predicted marginal density of $x$ for the normal distribution of different coreset size, $k=50, 100, 500$. First Row: uniform subsampling. Second row: Only $\ell_2$ sampling. Third row: $\ell_2$ with $\varepsilon$-kernel convex hull}
    \label{fig:pred_marg_x}
\end{figure}

\begin{figure}[!ht]
    \centering
    \begin{minipage}{.32\textwidth}
        \centering
        \includegraphics[width=\linewidth]{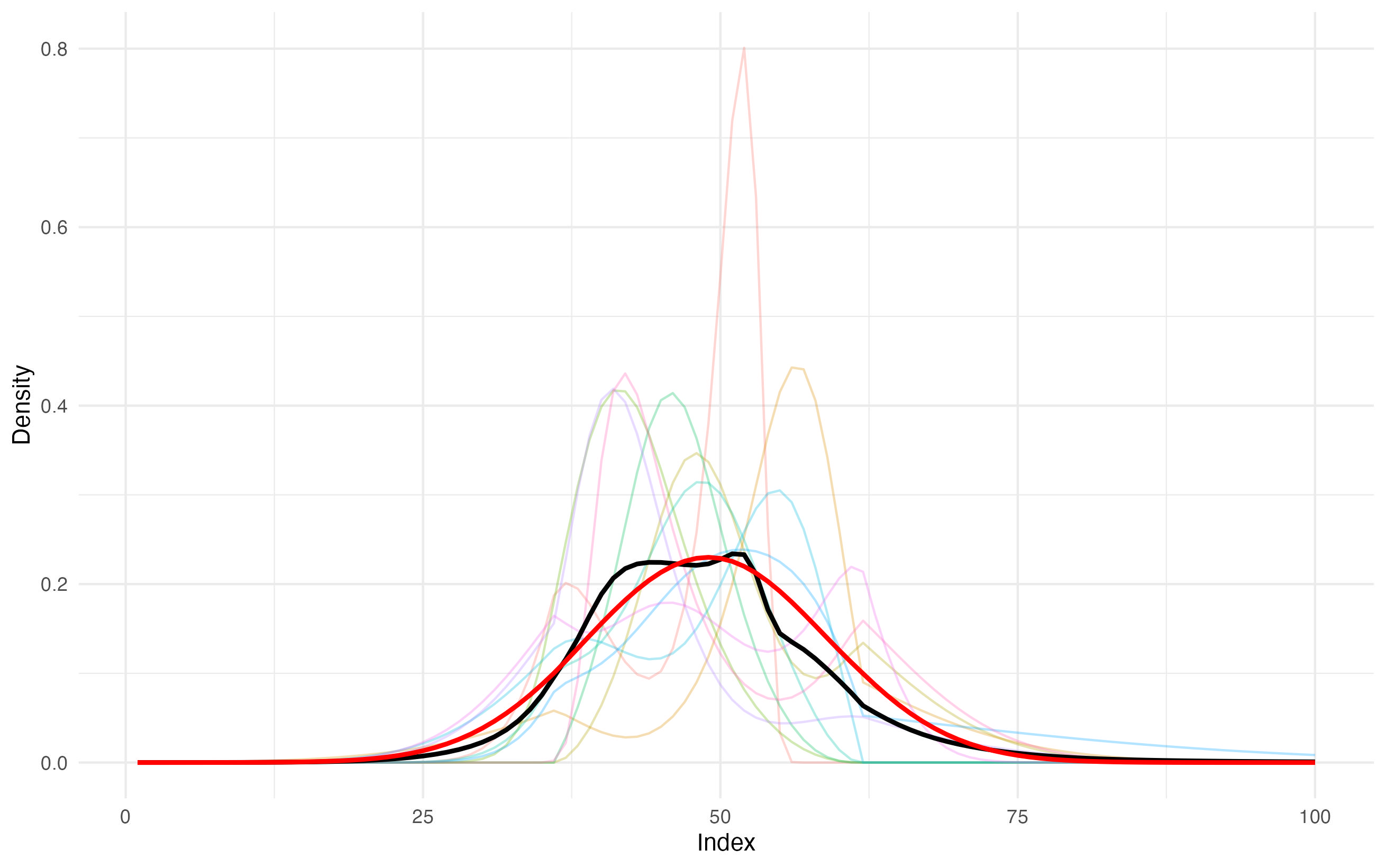} 
    \end{minipage}\hfill
    \begin{minipage}{.32\textwidth}
        \centering
        \includegraphics[width=\linewidth]{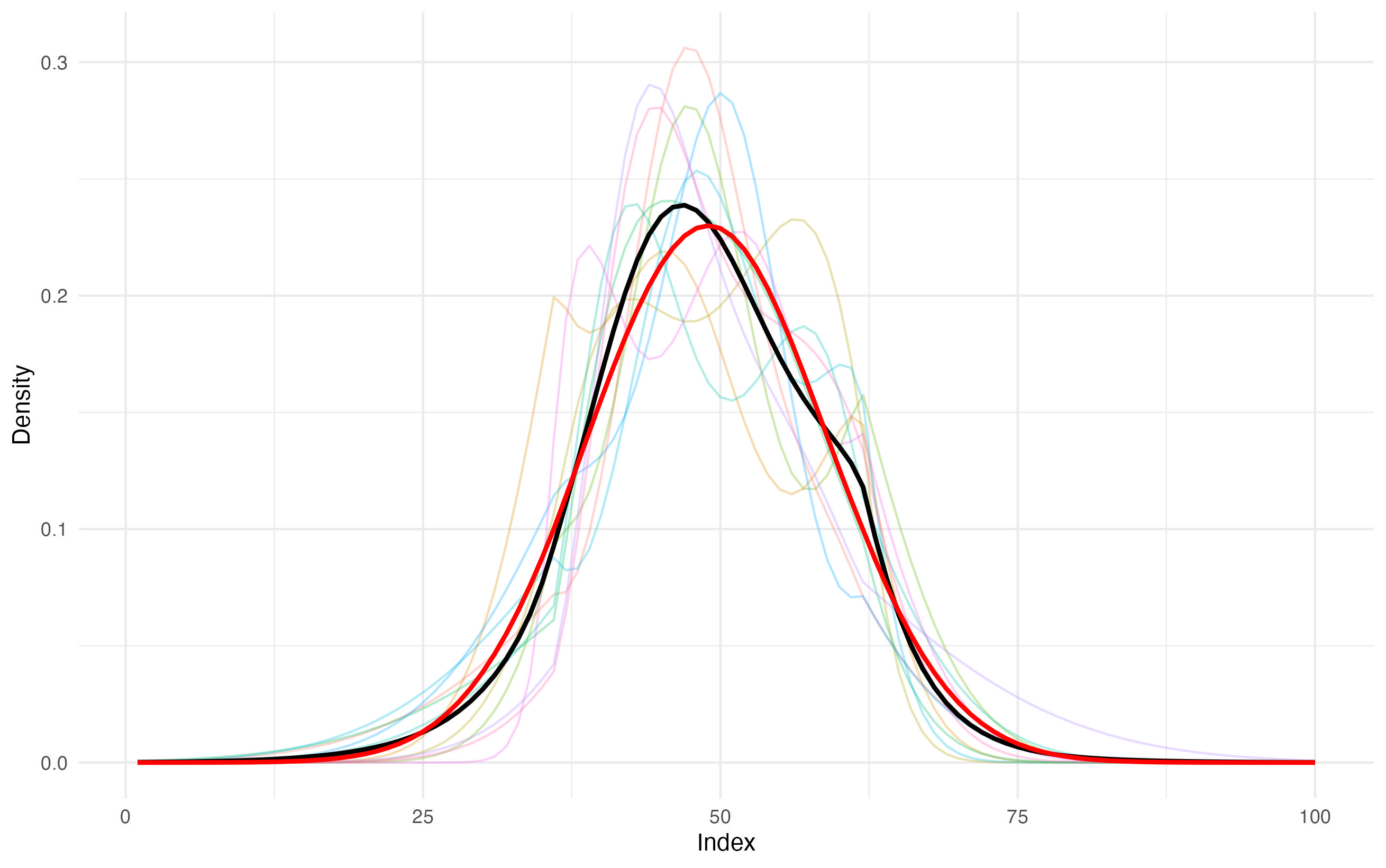} 
    \end{minipage}\hfill
    \begin{minipage}{.32\textwidth}
        \centering
        \includegraphics[width=\linewidth]{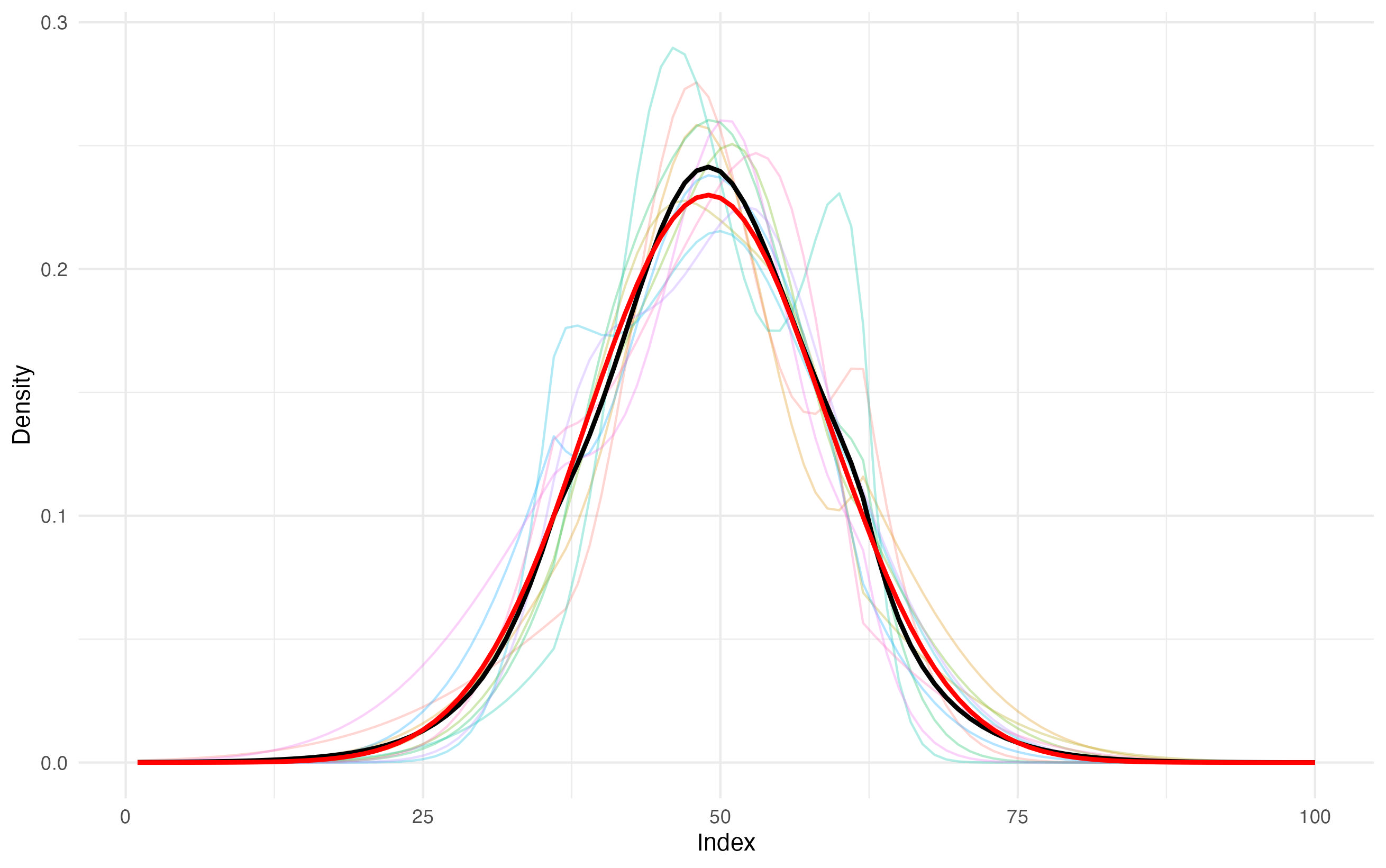} 
    \end{minipage}
    \begin{minipage}{.32\textwidth}
        \centering
        \includegraphics[width=\linewidth]{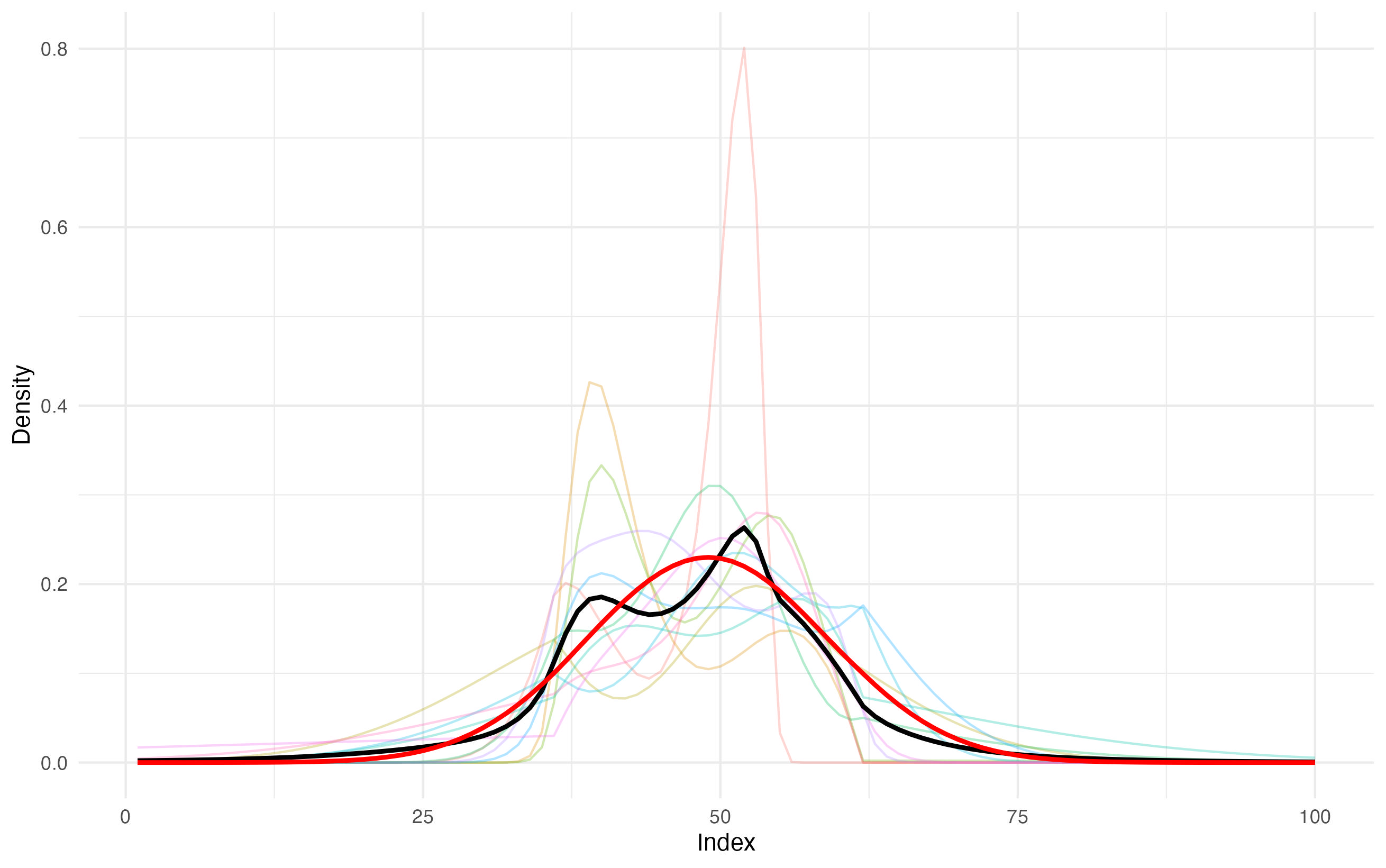} 
    \end{minipage}\hfill
    \begin{minipage}{.32\textwidth}
        \centering
        \includegraphics[width=\linewidth]{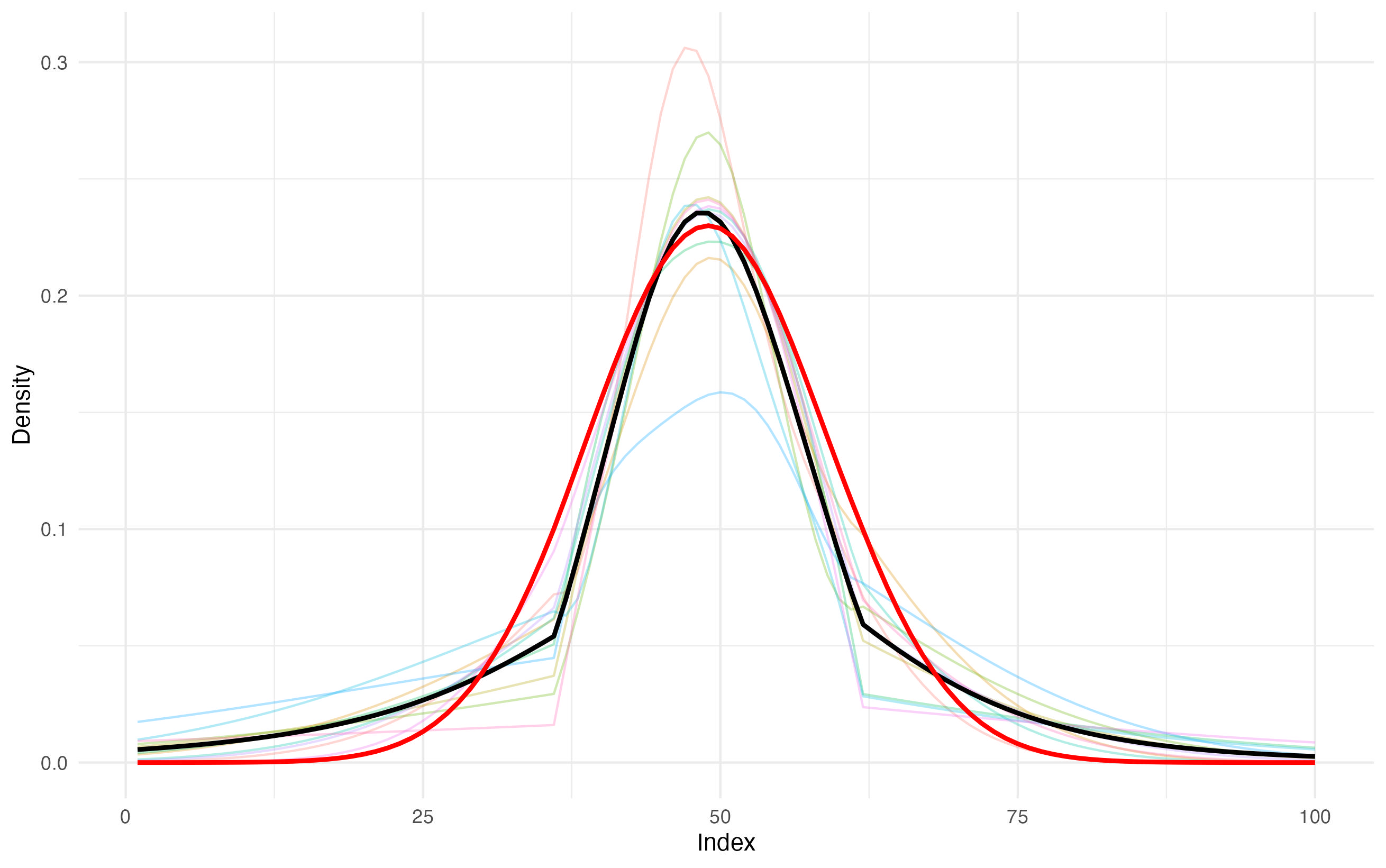} 
    \end{minipage}\hfill
    \begin{minipage}{.32\textwidth}
        \centering
        \includegraphics[width=\linewidth]{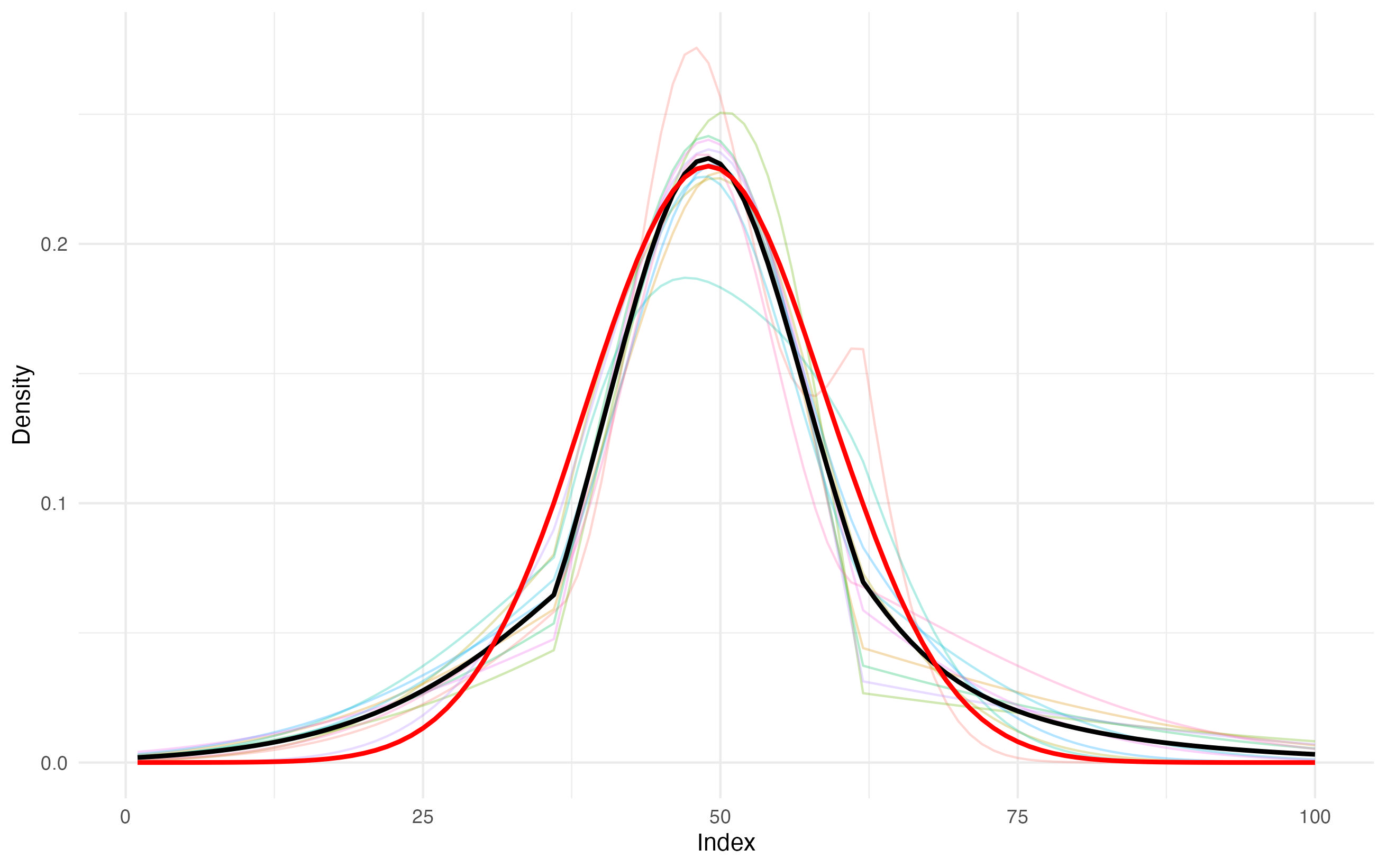} 
    \end{minipage}
    \begin{minipage}{.32\textwidth}
        \centering
        \includegraphics[width=\linewidth]{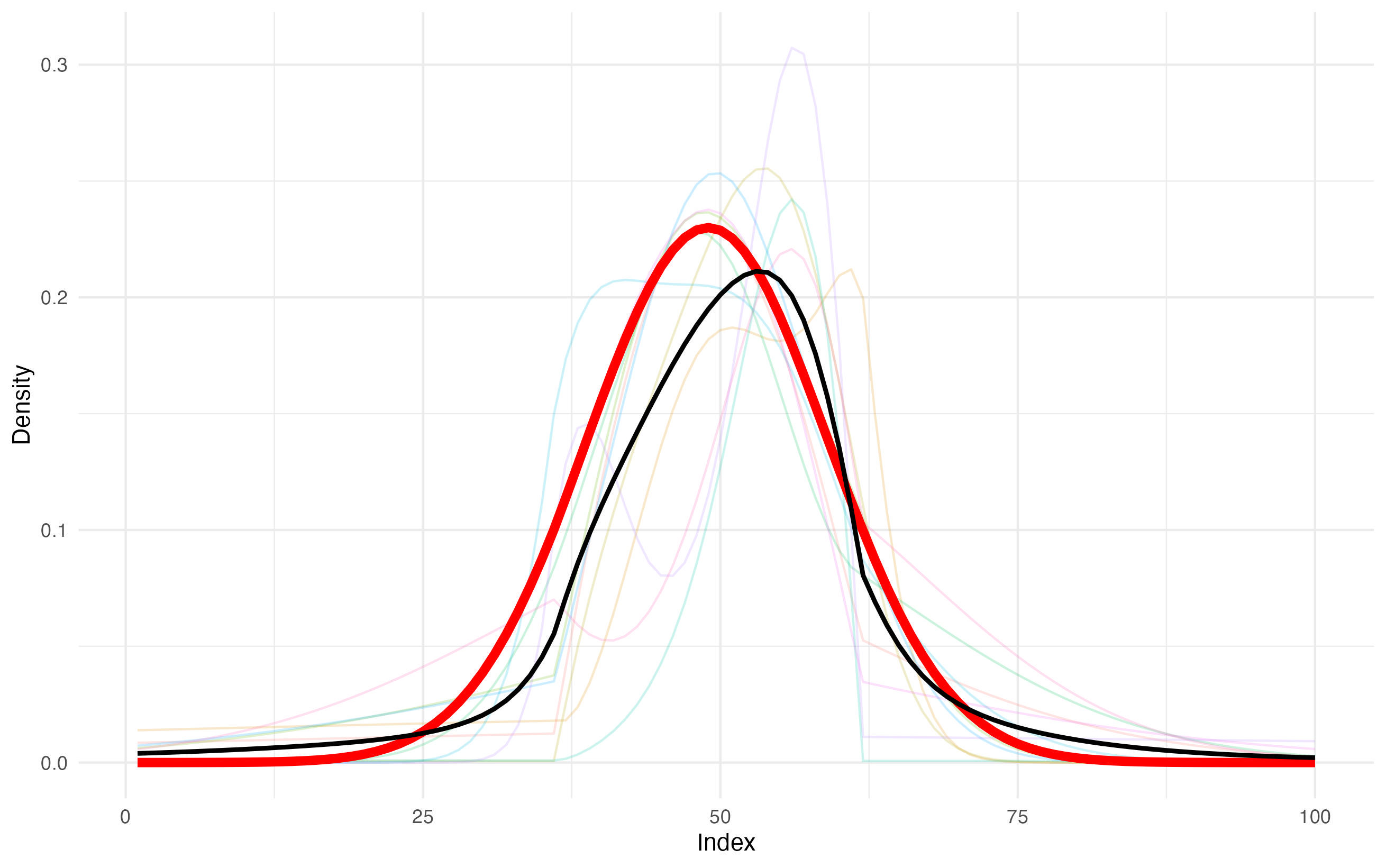} 
    \end{minipage}\hfill
    \begin{minipage}{.32\textwidth}
        \centering
        \includegraphics[width=\linewidth]{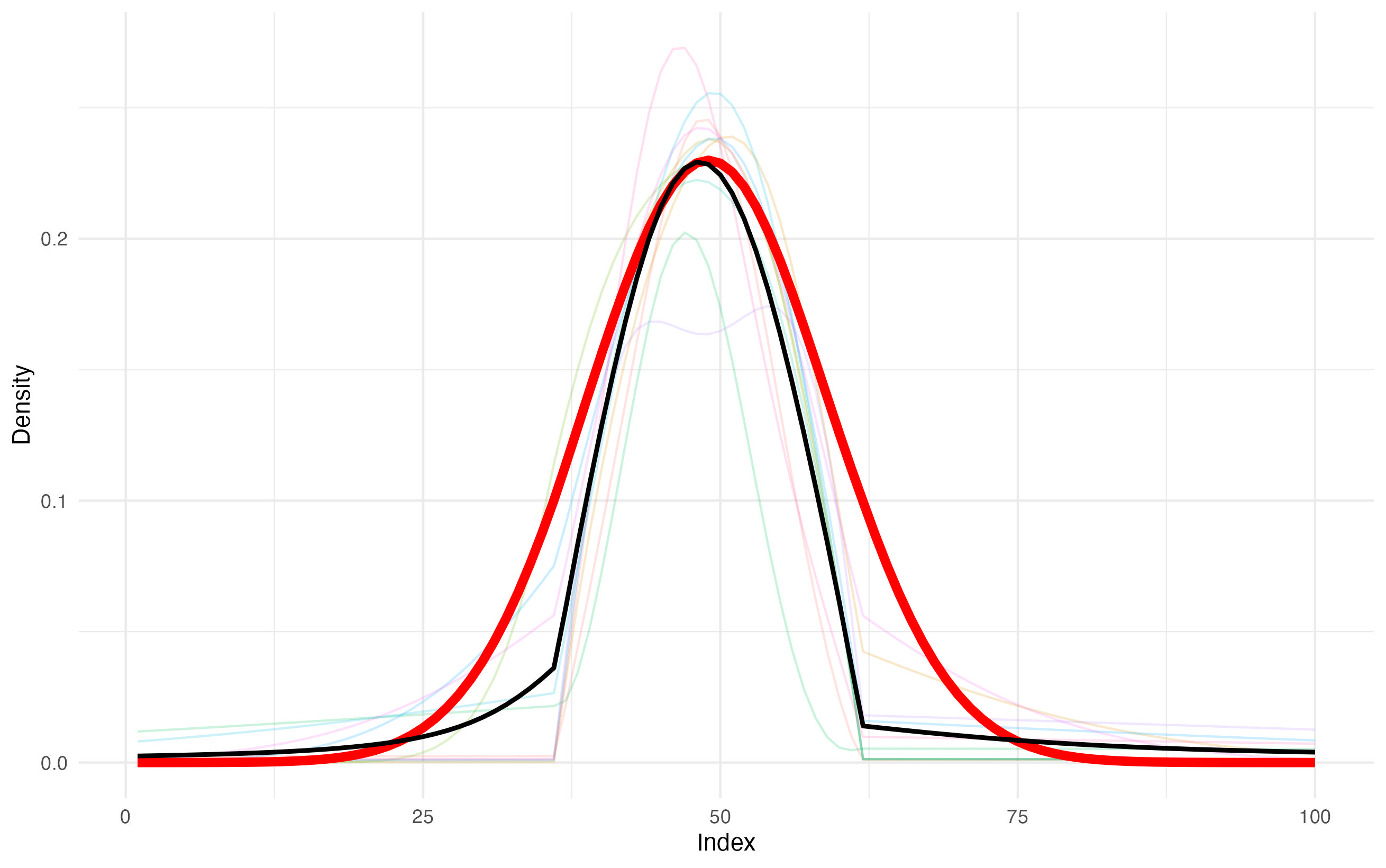} 
    \end{minipage}\hfill
    \begin{minipage}{.32\textwidth}
        \centering
        \includegraphics[width=\linewidth]{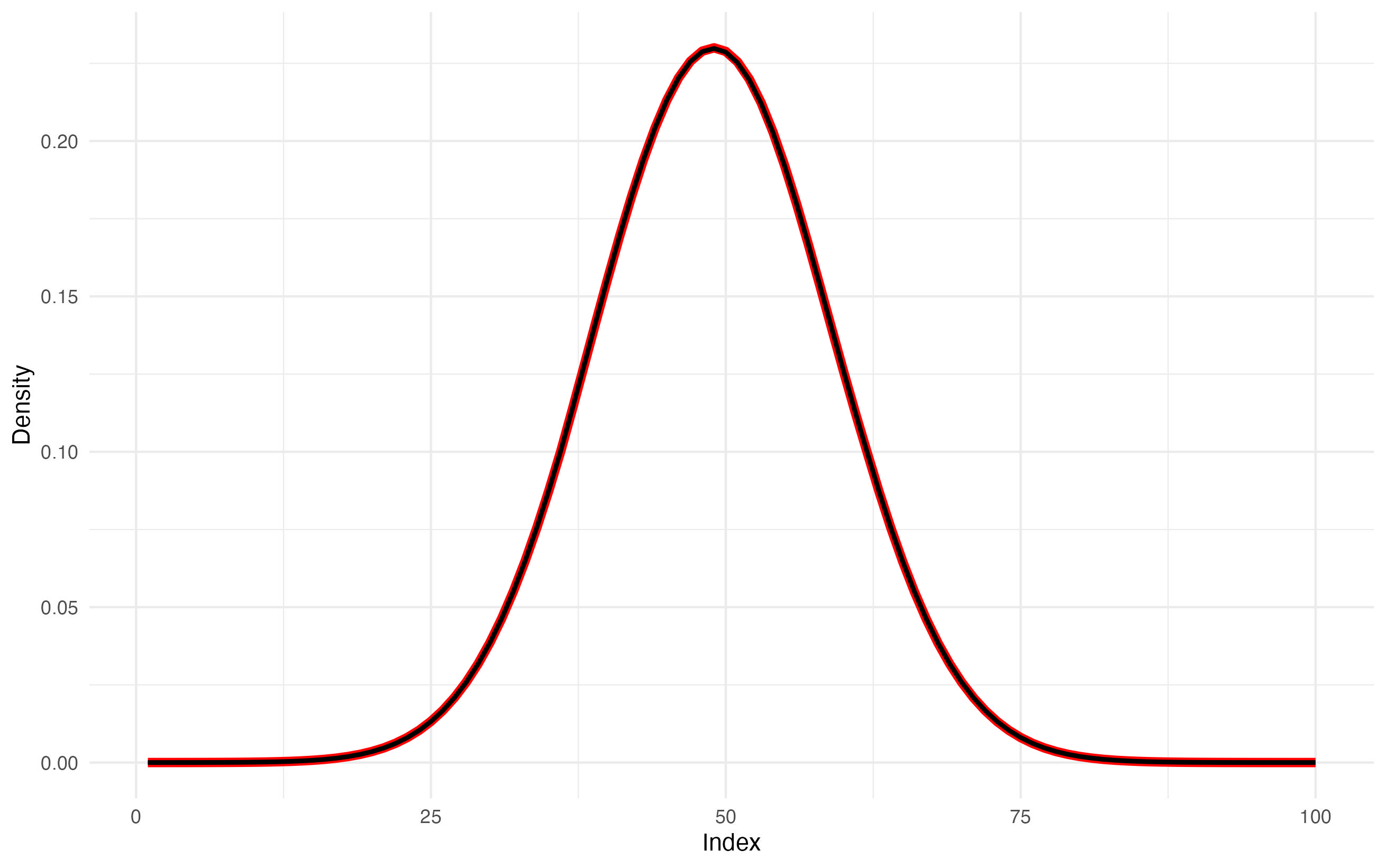} 
    \end{minipage}
    \caption{Predicted marginal density of $y$ of different coreset size, $k=50, 100, 500$. First Row: uniform subsampling. Second row: Only $\ell_2$ sampling. Third row: $\ell_2$ with $\varepsilon$-kernel convex hull}
    \label{fig:pred_marg_y}
\end{figure}

\subsection{Real-world Experiments Results}
\label{app:real_world}

\subsubsection{Covertype Dataset}

We demonstrate our method on the widely used UCI Covertype dataset \cite{covertype_31}, originally intended for forest cover type classification. The full dataset comprises $n=581\,012$ observations and $54$ variables, among which $10$ continuous variables represent various terrain features, such as elevation, slope, aspect, distances to hydrological features, and hillshade indices. 

This dataset presents several practical challenges motivating the use of MCTM. Specifically, the continuous variables exhibit complex joint dependency structures characterized by multimodality, heavy skewness, and highly non-linear pairwise interactions. Standard Gaussian or parametric copula approaches typically fail to capture these features adequately, thus justifying the adoption of MCTM, which flexibly models the multivariate distribution through non-linear transformations.

However, fitting a full MCTM on large-scale datasets is computationally intensive and sometimes fails. For example, in our case, on the Covertype dataset, when we attempted to fit the 10-dimensional MCTM model to the original $n=581\,012$ size dataset, hardware limitations caused this fit to simply crash. Therefore, we selected a subsample of $n=300\,000$ as the original model, and directly fitting an MCTM with only $10$ continuous variables on $n=300\,000$ observations already requires several hours of computation time. The computational challenge rapidly intensifies with higher dimensions or larger sample sizes, thus providing a natural and compelling scenario for demonstrating the practical benefits of our coreset approach.

We use the Covertype dataset to empirically assess whether coresets can efficiently approximate the full-data MCTM while substantially reducing computational costs, thereby enabling scalable probabilistic modeling of large multivariate datasets.

\begin{figure}[htbp]
    \centering
    \begin{subfigure}[t]{0.48\textwidth}
        \includegraphics[width=\linewidth]{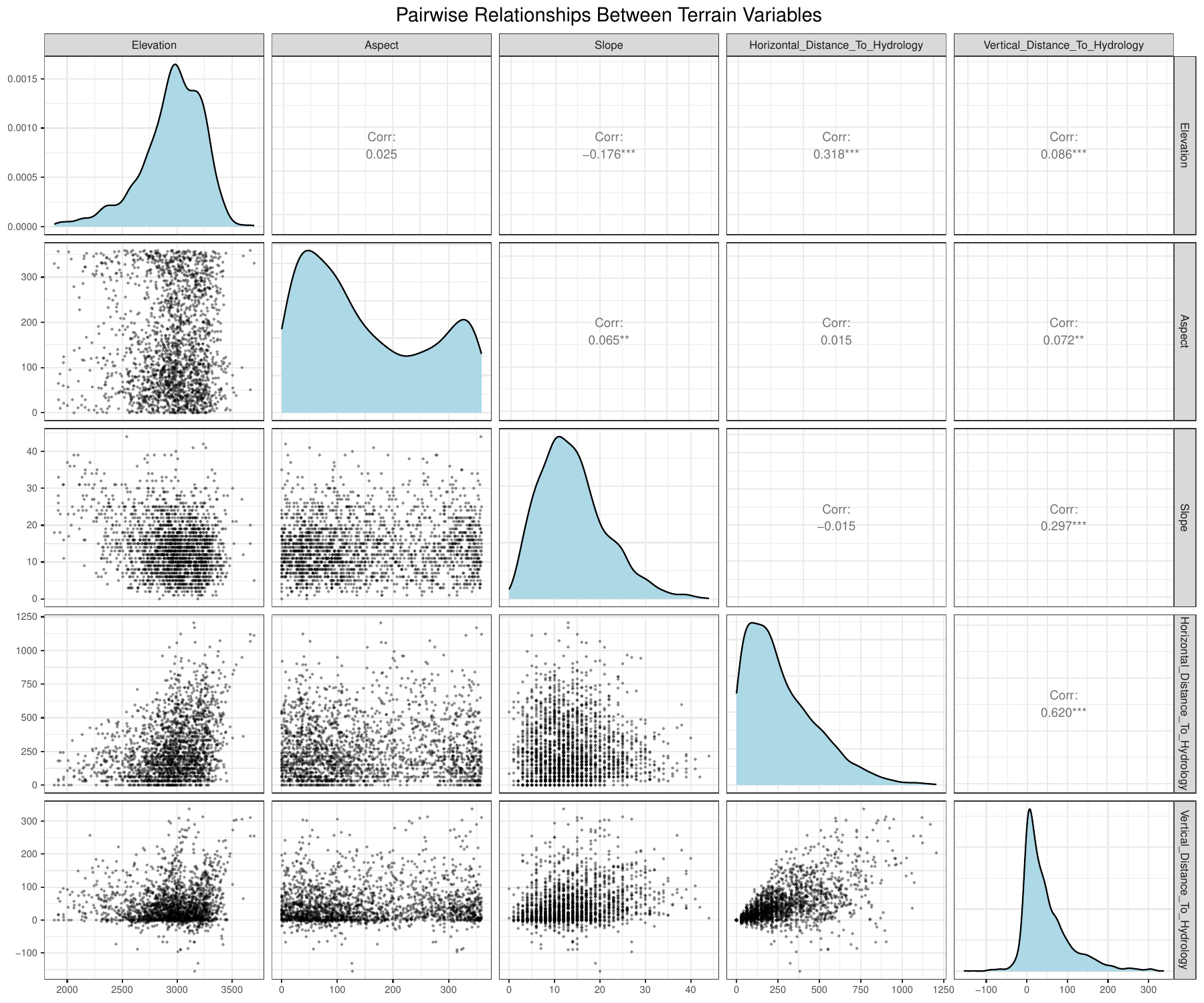}
        \caption{Elevation, Aspect, Slope, Horizontal and Vertical Distance to Hydrology.}
        \label{fig:ggpairs1}
    \end{subfigure}
    \hfill
    \begin{subfigure}[t]{0.48\textwidth}
        \includegraphics[width=\linewidth]{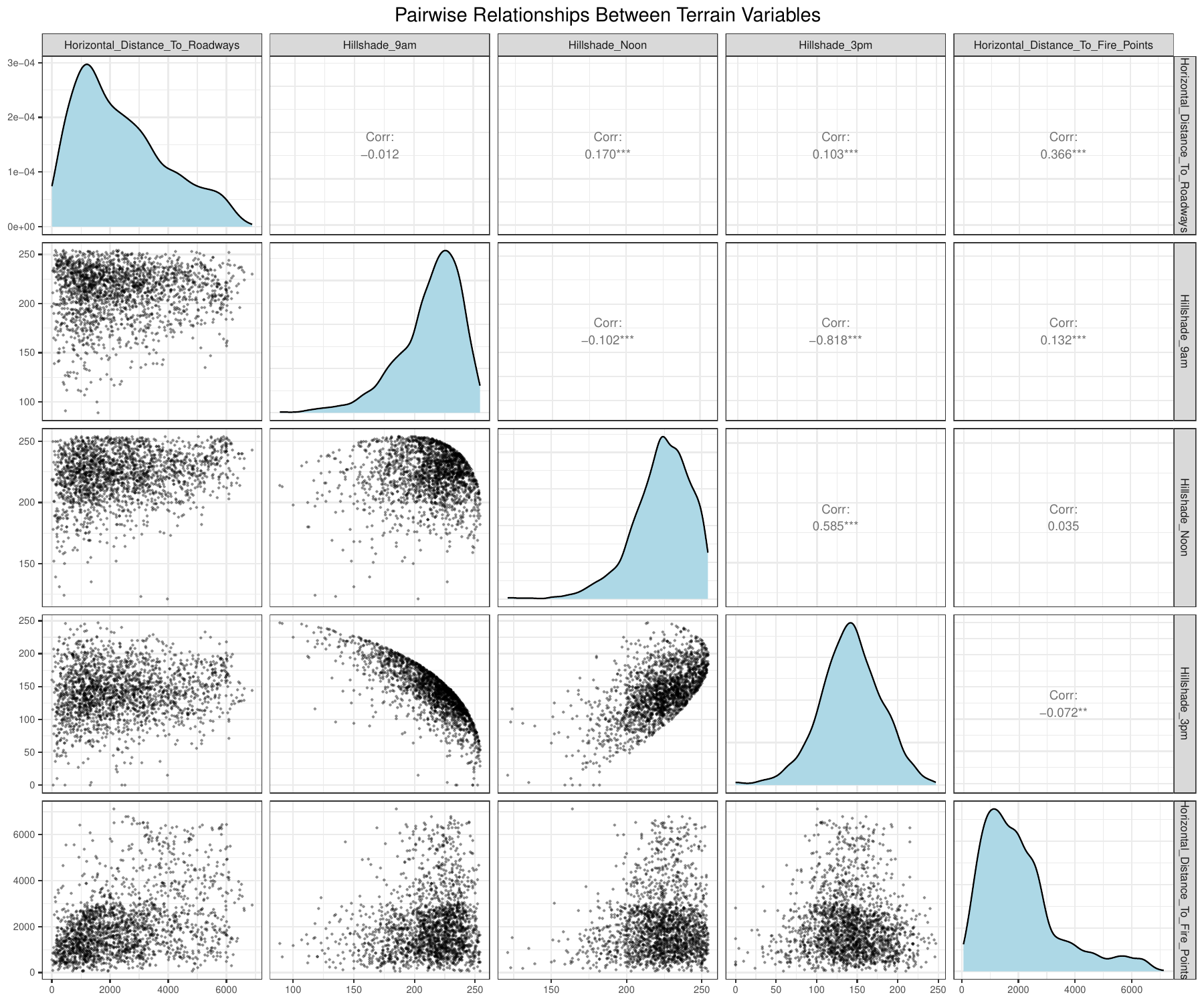}
        \caption{Roadways, Hillshade at different times, and Fire Points.}
        \label{fig:ggpairs2}
    \end{subfigure}
    \caption{Pairwise relationships between different sets of terrain variables from the Covertype dataset.}
    \label{fig:ggpairs_combined}
\end{figure}

\begin{figure}[htbp]
  \centering
  \begin{subfigure}[b]{0.48\textwidth}
    \includegraphics[width=\textwidth]{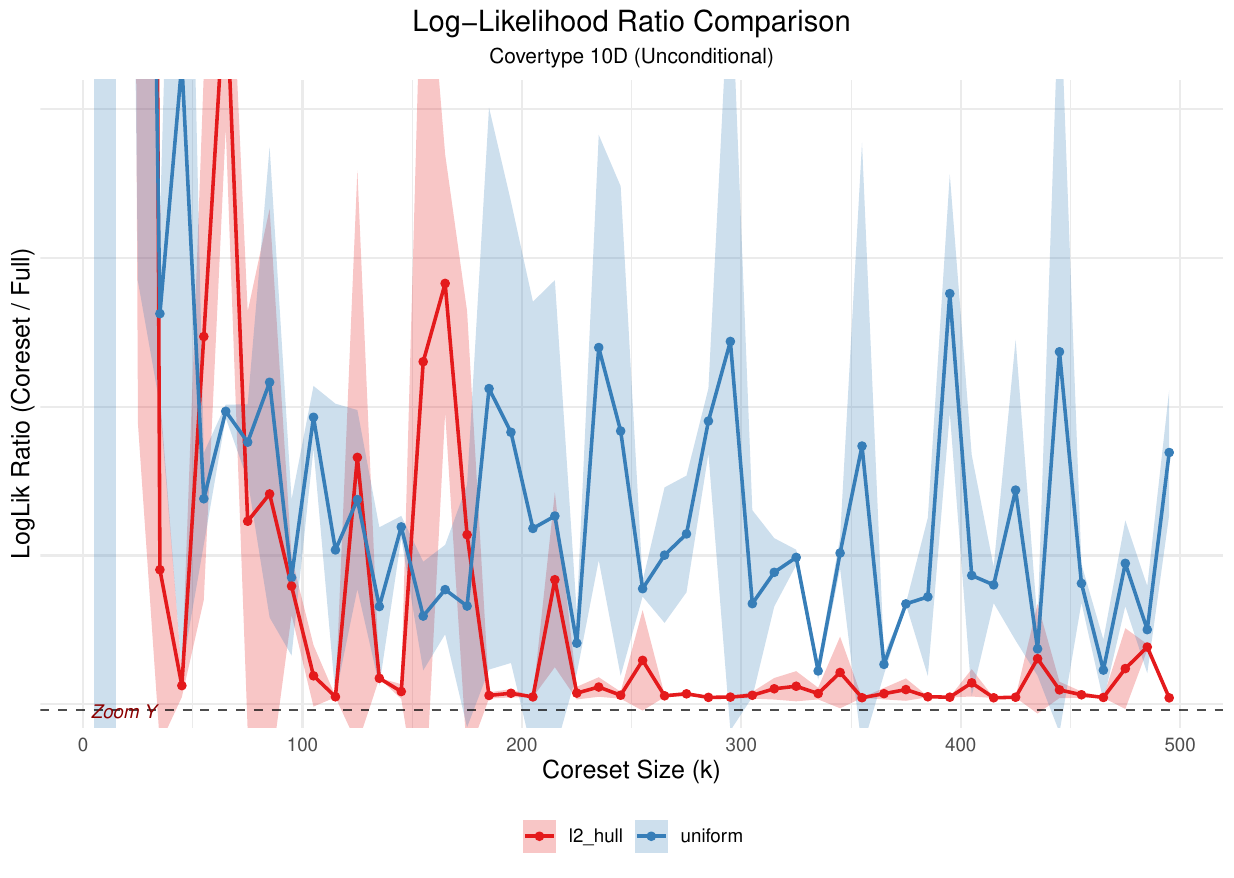}
    \caption{Log-Likelihood Ratio Comparison}
    \label{fig:covertype_llk}
  \end{subfigure}
  \hfill
  \begin{subfigure}[b]{0.48\textwidth}
    \includegraphics[width=\textwidth]{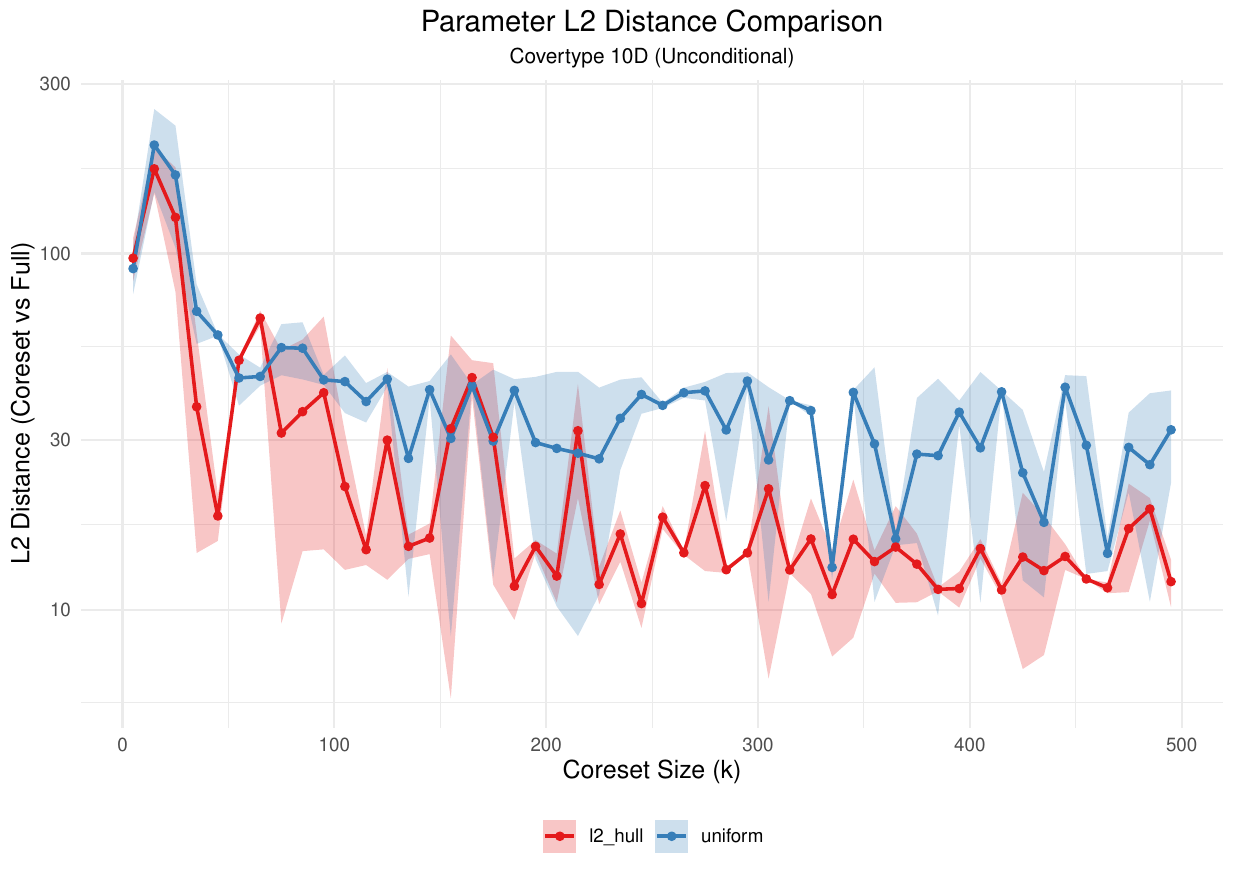}
    \caption{Parameter $\ell_2$ Distance Comparison}
    \label{fig:covertype_param_l2}
  \end{subfigure}

  \vspace{1em}

  \begin{subfigure}[b]{0.48\textwidth}
    \includegraphics[width=\textwidth]{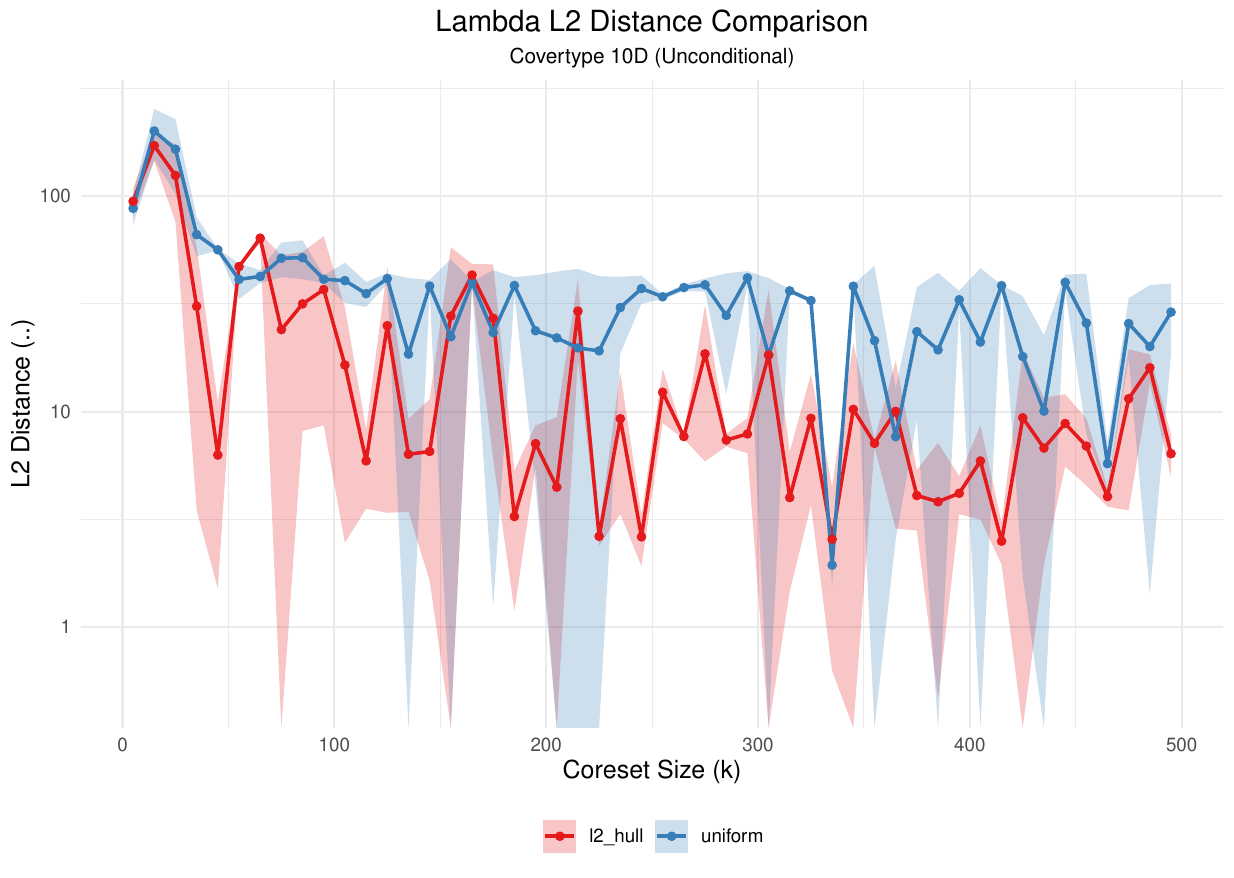}
    \caption{Lambda $\ell_2$ Distance Comparison}
    \label{fig:covertype_lambda_l2}
  \end{subfigure}
  \hfill
  \begin{subfigure}[b]{0.48\textwidth}
    \includegraphics[width=\textwidth]{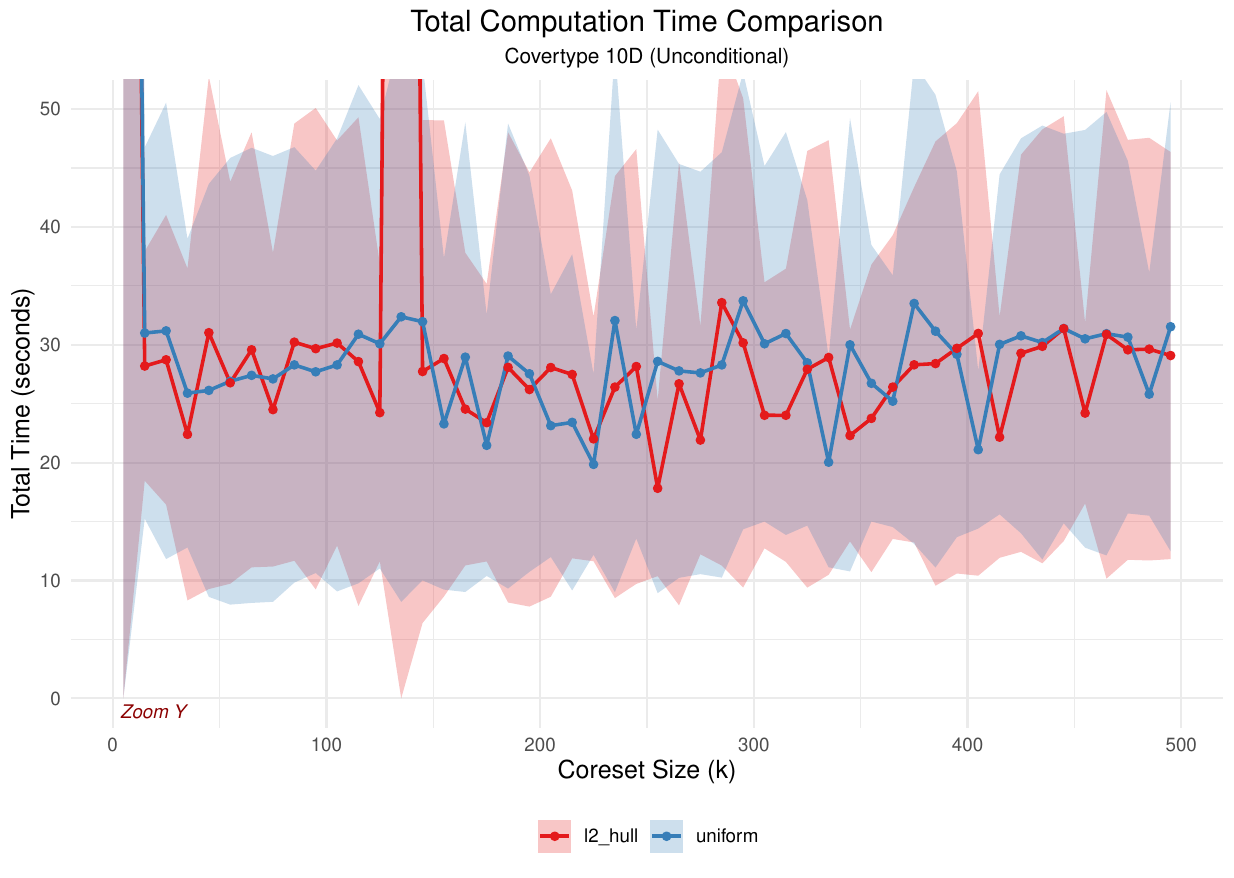}
    \caption{Total Computation Time Comparison}
    \label{fig:covertype_time}
  \end{subfigure}
  \caption{Covertype dataset (10-dimensional, unconditional model) on \texttt{$\ell_2$-hull} versus uniform sampling (\texttt{uniform}) compared with respect to four evaluation metrics:
    (a) log-likelihood ratio, (b) parameter-space $\ell_2$ distances, (c) $\ell_2$ distances dependent on parameter $\lambda$,
    (d) Total computation time.}

  \label{fig:covertype_results}
\end{figure}

As can be seen from the Figure~\ref{fig:covertype_results}, our proposed \texttt{$\ell_2$-hull} method significantly outperforms the uniform sampling (\texttt{uniform}) in all the four evaluation metrics. Specifically, \texttt{$\ell_2$-hull} is closer to 1 in the log-likelihood ratio, which indicates a more accurate approximation of the original model; it maintains a smaller error in the $\ell_2$ distance between the $\vartheta$ parameters and $\lambda$ as well, which proves that it is able to better preserve the distributional structure of the original data in the parameter space whose dimension is roughly squared compared to the plain data dimension; and at the same time, its overall running time is comparable to that of uniform sampling, which does not introduce any additional significant overheads. It can be seen that for multivariate large-scale datasets, the \texttt{$\ell_2$-hull} method, which combines $\ell_2$ subsampling with convex hull techniques, outperforms simple uniform sampling in terms of both performance and efficiency.

\subsubsection{Equity Return Dataset}

In finance, returns on multiple stocks often have complex non-linear and time-varying dependence structures. To capture these dependencies, copula models are widely used in risk management and asset allocation. However, traditional copula mostly need to specify the marginal distributions and assume fixed dependence parameters, which makes it difficult to simultaneously model the dynamic changes of the marginals and conditional correlations. MCTM can more comprehensively reflect the characteristics of the joint distribution of financial asset returns by jointly estimating the marginal distributions and the dependence structures through a flexible transformation function, and thus has a good prospect of being applied in the modeling of stock returns.  Therefore, we select 10 and 20 representative stock returns data, see Table~\ref{tab:stocklist10} and Table~\ref{tab:stocklist20}, respectively, construct their joint distributions based on the multivariate conditional transformation model MCTM, and combine the coreset method proposed in this paper to improve the computational efficiency of model fitting.

\begin{table}[htbp]
\centering
\caption{Performance comparison on 10 stock return series for different coreset sizes (1985-2025)}
\label{tab:stock_return_performance}
\resizebox{\textwidth}{!}{
\begin{tabular}{llccccc}
\toprule
\textbf{Coreset Size} & \textbf{Method} & \textbf{Param. $\ell_2$ dist.} & \textbf{$\lambda$ error} & \textbf{Log‐likelihood ratio} & \textbf{Rel. Impr. L2 / $\lambda$ / LL (\%)} & \textbf{Total time (s)} \\
\midrule

\multirow{3}{*}{$k=50$}
& $\ell_2$‐hull & $\mathbf{45.518 \pm 1.643}$       & $2.992 \pm 0.530$               & $1.683 \pm 0.234$              & 12.0 / 0.0 / 57.5 & $8.21 \pm 2.30$ \\
& $\ell_2$‐only & $45.784 \pm 0.554$                & $2.599 \pm 0.629$               & $\mathbf{1.351 \pm 0.082}$              & 11.5 / 0.0 / 78.2 & $8.45 \pm 2.53$ \\
& uniform      & $51.745 \pm 2.474$                & $\mathbf{1.515 \pm 0.303}$      & $2.610 \pm 0.101$     & baseline         & $\mathbf{7.98 \pm 1.88}$ \\

\midrule

\multirow{3}{*}{$k=100$}
& $\ell_2$‐hull & $\mathbf{39.888 \pm 0.643}$       & $1.613 \pm 0.106$               & $1.399 \pm 0.147$              & 22.1 / 0.0 / 89.3 & $9.93 \pm 2.58$ \\
& $\ell_2$‐only & $41.183 \pm 1.503$                & $1.404 \pm 0.114$               & $\mathbf{1.153 \pm 0.044}$              & 19.6 / 0.0 / 96.1 & $9.80 \pm 2.91$ \\
& uniform      & $51.195 \pm 2.616$                & $\mathbf{0.913 \pm 0.140}$      & $4.926 \pm 0.390$     & baseline         & $\mathbf{8.94 \pm 2.10}$ \\

\midrule

\multirow{3}{*}{$k=200$}
& $\ell_2$‐hull & $\mathbf{33.722 \pm 1.242}$       & $1.050 \pm 0.133$               & $1.236 \pm 0.145$              & 29.6 / 0.0 / 80.9 & $11.37 \pm 2.89$ \\
& $\ell_2$‐only & $38.031 \pm 0.282$                & $1.147 \pm 0.140$               & $\mathbf{1.085 \pm 0.004}$              & 20.6 / 0.0 / 93.1 & $11.40 \pm 3.51$ \\
& uniform      & $47.881 \pm 1.217$                & $\mathbf{0.621 \pm 0.108}$      & $2.242 \pm 0.149$     & baseline         & $\mathbf{10.13 \pm 2.69}$ \\

\midrule

\multirow{3}{*}{$k=300$}
& $\ell_2$‐hull & $\mathbf{29.814 \pm 1.363}$       & $0.905 \pm 0.100$               & $1.127 \pm 0.081$              & 33.6 / 0.0 / 92.1 & $13.95 \pm 4.48$ \\
& $\ell_2$‐only & $35.107 \pm 0.916$                & $0.851 \pm 0.091$      & $\mathbf{1.070 \pm 0.007}$     & 21.8 / 0.0 / 93.5 & $\mathbf{12.63 \pm 3.13}$ \\
& uniform      & $44.915 \pm 3.172$                & $\mathbf{0.488 \pm 0.050}$      & $2.079 \pm 0.171$     & baseline         & $13.01 \pm 3.91$ \\

\bottomrule
\end{tabular}
}
\begin{tablenotes}
\small
\item Results are mean $\pm 1$ standard deviation over 5 valid trials. Bold indicates the best (lowest for error metrics, likelihood ratio closer for 1 for better) per coreset size and metric. Relative improvement (\%) over the uniform baseline shown per metric (Param L2 / Lambda / LL)
\end{tablenotes}
\end{table}

\begin{table}[htbp]
\centering
\caption{Performance comparison on 20 stock return series for different coreset sizes (1985-2025)}
\label{tab:stock_return_performance_2}
\resizebox{\textwidth}{!}{
\begin{tabular}{llccccc}
\toprule
\textbf{Coreset Size} & \textbf{Method} & \textbf{Param. $\ell_2$ dist.} & \textbf{$\lambda$ error} & \textbf{Log‐likelihood ratio} & \textbf{Rel. Impr. L2 / $\lambda$ / LL (\%)} & \textbf{Total time (s)} \\
\midrule

\multirow{3}{*}{$k=50$}
& $\ell_2$‐hull & $54.254 \pm 1.367$                & $3.944 \pm 0.563$                 & $5.208 \pm 0.210$                & 3.6 / 0.0 / 0.0 & $30.95 \pm 5.06$ \\
& $\ell_2$‐only & $\mathbf{53.922 \pm 1.280}$       & $3.920 \pm 0.344$                 & $30.303 \pm 0.426$       & 4.2 / 0.0 / 0.0 & $31.10 \pm 6.00$ \\
& uniform      & $56.263 \pm 1.042$                & $\mathbf{3.490 \pm 0.260}$        & $\mathbf{4.504 \pm 0.200}$                & baseline         & $\mathbf{29.00 \pm 7.55}$ \\

\midrule

\multirow{3}{*}{$k=100$}
& $\ell_2$‐hull & $\mathbf{46.280 \pm 1.887}$       & $2.270 \pm 0.138$                 & $\mathbf{1.675 \pm 0.154}$       & 13.4 / 0.0 / 89.6 & $35.56 \pm 7.94$ \\
& $\ell_2$‐only & $47.385 \pm 0.788$                & $2.304 \pm 0.370$                 & $1.718 \pm 0.141$                & 11.4 / 0.0 / 88.9 & $\mathbf{34.95 \pm 6.53}$ \\
& uniform      & $53.469 \pm 2.072$                & $\mathbf{2.006 \pm 0.151}$        & $7.518 \pm 0.305$                & baseline         & $35.22 \pm 6.65$ \\

\midrule

\multirow{3}{*}{$k=200$}
& $\ell_2$‐hull & $40.225 \pm 3.690$                & $1.511 \pm 0.138$                 & $1.574 \pm 0.147$                & 20.4 / 0.0 / 62.3 & $49.73 \pm 7.74$ \\
& $\ell_2$‐only & $\mathbf{38.852 \pm 0.925}$       & $1.458 \pm 0.079$        & $\mathbf{1.131 \pm 0.044}$       & 23.1 / 0.0 / 91.4 & $52.99 \pm 12.33$ \\
& uniform      & $50.531 \pm 1.513$                & $\mathbf{1.224 \pm 0.158}$                 & $2.525 \pm 0.067$                & baseline         & $\mathbf{45.91 \pm 7.39}$ \\

\midrule

\multirow{3}{*}{$k=300$}
& $\ell_2$‐hull & $\mathbf{35.819 \pm 3.552}$       & $1.165 \pm 0.099$                 & $1.400 \pm 0.107$                & 23.6 / 0.0 / 75.2 & $\mathbf{60.99 \pm 6.86}$ \\
& $\ell_2$‐only & $35.036 \pm 0.286$                & $1.137 \pm 0.132$        & $\mathbf{1.109 \pm 0.033}$       & 25.3 / 0.0 / 93.2 & $61.41 \pm 10.14$ \\
& uniform      & $46.906 \pm 3.669$                & $\mathbf{1.028 \pm 0.044}$                 & $2.617 \pm 0.195$                & baseline         & $63.89 \pm 12.15$ \\

\bottomrule
\end{tabular}
}
\begin{tablenotes}
\small
\item Results are mean $\pm 1$ standard deviation over 5 valid trials. Bold indicates the best (lowest for error metrics, likelihood ratio closer for 1 for better) per coreset size and metric. Relative improvement (\%) over the uniform baseline shown per metric (Param L2 / Lambda / LL)
\end{tablenotes}
\end{table}

In Table~\ref{tab:stock_return_performance} (10 stocks) and Table~\ref{tab:stock_return_performance_2} (20 stocks) we can see that in most of the experimental scenarios, the $\ell_2$-hull method achieves excellent performance, especially in the two metrics of $\ell_2$ distance and log-likelihood ratio in the parameter space, which are comparable compared to the pure $\ell_2$ sampling scheme. This may be due to the fact that in these scenarios, there are not many extreme points and the convex hull approximation does not add a significant advantage. In addition to this, the $\ell_2$-hull is far superior to uniform subsampling in terms of parameter error and log-likelihood ratios in the vast majority of scenarios, and is no slouch in terms of estimation accuracy dependent on the structural parameter $\lambda$. In summary, the combination of sensitivity sampling and convex hull technique not only maintains the comparable effect with simple $\ell_2$ sampling, but also outperforms uniform sampling when facing sparse extremes and complex multivariate structures, and achieves a good balance between high accuracy and high efficiency.
\begin{table}[htbp]
\centering
\caption{List of 10 Selected Stocks}
\label{tab:stocklist10}
\begin{tabular}{ll}
\toprule
\textbf{Ticker} & \textbf{Company Name} \\
\midrule
JNJ  & Johnson \& Johnson \\
PG   & Procter \& Gamble Co. \\
KO   & The Coca-Cola Company \\
XOM  & Exxon Mobil Corporation \\
WMT  & Walmart Inc. \\
IBM  & International Business Machines Corporation \\
GE   & General Electric Company \\
MMM  & 3M Company \\
MCD  & McDonald's Corporation \\
PFE  & Pfizer Inc. \\
\bottomrule
\end{tabular}
\end{table}

\vspace{1cm}

\begin{table}[htbp]
\centering
\caption{List of 20 Selected Stocks}
\label{tab:stocklist20}
\begin{tabular}{ll}
\toprule
\textbf{Ticker} & \textbf{Company Name} \\
\midrule
JNJ   & Johnson \& Johnson \\
PG    & Procter \& Gamble Co. \\
KO    & The Coca-Cola Company \\
XOM   & Exxon Mobil Corporation \\
WMT   & Walmart Inc. \\
IBM   & International Business Machines Corporation \\
GE    & General Electric Company \\
MMM   & 3M Company \\
MCD   & McDonald's Corporation \\
PFE   & Pfizer Inc. \\
AAPL  & Apple Inc. \\
MSFT  & Microsoft Corporation \\
INTC  & Intel Corporation \\
CSCO  & Cisco Systems Inc. \\
AMGN  & Amgen Inc. \\
CMCSA & Comcast Corporation \\
COST  & Costco Wholesale Corporation \\
GILD  & Gilead Sciences Inc. \\
SBUX  & Starbucks Corporation \\
TOT   & TotalEnergies SE \\
\bottomrule
\end{tabular}
\end{table}

\end{document}